\documentclass{article}
\usepackage{ijcai25}

\usepackage{times}
\usepackage{soul}
\usepackage{url}
\usepackage[hidelinks]{hyperref}
\usepackage[utf8]{inputenc}
\usepackage[small]{caption}
\usepackage{graphicx}
\usepackage{amsmath}
\usepackage{amsthm}
\usepackage{booktabs}
\usepackage{algorithm}
\usepackage{algorithmic}
\usepackage[switch]{lineno}

\usepackage{amssymb}
\usepackage{newtxmath}
\usepackage{multirow}
\usepackage[normalem]{ulem}
\useunder{\uline}{\ul}{}
\usepackage[table,xcdraw]{xcolor}
\usepackage{makecell}
\usepackage{marvosym}
\usepackage{rotating}
\usepackage{hyperref}
\usepackage{amsmath}
\usepackage{amssymb}
\usepackage{float}
\usepackage {pifont}

\newtheorem{theorem}{Theorem}

\title{FreEformer: Frequency Enhanced Transformer for \\ Multivariate Time Series Forecasting}

\author{
Wenzhen Yue$^1$\and
Yong Liu$^2$\and
Xianghua Ying$^1$ \and
Bowei Xing$^1$\and
Ruohao Guo$^1$\And
Ji Shi$^1$
\\
\affiliations
$^1$National Key Laboratory of General Artificial Intelligence, School of Intelligence Science and Technology, Peking University \\
$^2$School of Software, BNRist, Tsinghua University \\
\emails
yuewenzhen@stu.pku.edu.cn, liuyong21@mails.tsinghua.edu.cn,
xhying@pku.edu.cn
}

\begin{document}

\maketitle

\begin{abstract}
This paper presents \textbf{FreEformer}, a simple yet effective model that leverages a \textbf{Fre}quency \textbf{E}nhanced Trans\textbf{former} for multivariate time series forecasting. Our work is based on the assumption that the frequency spectrum provides a global perspective on the composition of series across various frequencies and is highly suitable for robust representation learning. Specifically, we first convert time series into the complex frequency domain using the Discrete Fourier Transform (DFT). The Transformer architecture is then applied to the frequency spectra to capture cross-variate dependencies, with the real and imaginary parts processed independently. However, we observe that the vanilla attention matrix exhibits a low-rank characteristic, thus limiting representation diversity. This could be attributed to the inherent sparsity of the frequency domain and the strong-value-focused nature of Softmax in vanilla attention. To address this, we enhance the vanilla attention mechanism by introducing an additional learnable matrix to the original attention matrix, followed by row-wise L1 normalization. Theoretical analysis~demonstrates that this enhanced attention mechanism improves both feature diversity and gradient flow. Extensive experiments demonstrate that FreEformer consistently outperforms state-of-the-art models on eighteen real-world benchmarks covering electricity, traffic, weather, healthcare and finance. Notably, the enhanced attention mechanism also consistently improves the performance of state-of-the-art Transformer-based forecasters.

\end{abstract}

\section{Introduction}

Multivariate time series forecasting holds significant importance in real-world domains such as weather~\cite{nature_weather}, energy~\cite{informer2021}, transportation~\cite{aaai2022_catn} and finance~\cite{survey4}. In recent years, various deep learning models have been proposed, significantly pushing the performance boundaries. Among these models, Recurrent Neural Networks (RNN) \cite{rnn_2020}, Convolutional Neural Networks (CNN) \cite{tcn,timesnet},  LLM \cite{zhou2023one,time_llm}, Multi-Layer Perceptrons (MLP) \cite{linear,fits}, Transformers-based methods \cite{patchtst,itransformer,card} have demonstrated great potential due to their strong representation capabilities.

\nocite{rlinear,diff_refine}

\begin{figure}[t]
   \centering
   \includegraphics[width=1.0\linewidth]{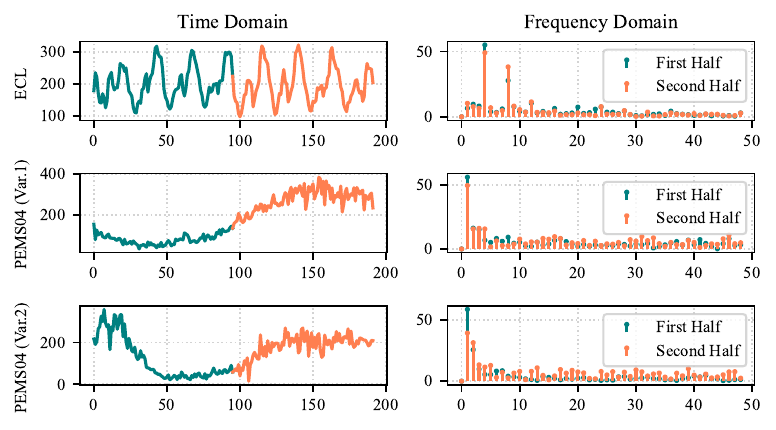}
   \caption{Time series and their corresponding frequency spectra. The series are normalized before applying the DFT, and the amplitudes of the frequency spectra are plotted. (1) The frequency spectra often exhibit strong consistency across adjacent temporal spans within the same time series, forming the basis for frequency-based forecasting. (2) Strong correlations between the two variables in PEMS04 (rows 2 and 3) are observed, suggesting that exploring such multivariate relationships could lead to more robust representations. (3) The frequency spectrum usually exhibits sparsity, with a few dominant frequencies.}
   \label{fig1}
\end{figure}

In recent years, frequency-domain-based models have been proposed and have achieved great performance~\cite{frets,fits}, benefiting from the robust frequency domain modeling. As shown in Figure \ref{fig1}, frequency spectra exhibit strong consistency across different spans of the same series, making them suitable for forecasting. Most existing frequency-domain-based works \cite{filternet} rely on linear layers to learn frequency-domain representations, resulting in a performance gap. Frequency-domain Transformer-based models remain under-explored. Recently, Fredformer \cite{fredformer} applies the vanilla Transformer to patched frequency tokens to address the frequency bias issue. However, the patching technique introduces additional hyper-parameters and undermines the inherent global perspective \cite{frets} of frequency-domain modeling.

In this paper, we adopt a simple yet effective approach by applying the Transformer to frequency-domain variate tokens for representation learning. Specifically, we embed the entire frequency spectrum as variate tokens and capture cross-variate dependencies among them. This architecture offers three main advantages: 1) As shown in Section \ref{preliminaries}, simple frequency-domain operations can correspond to complex temporal operations \cite{frets}; 2) Inter-variate correlations typically exists (e.g., Figure \ref{fig1}) \cite{itransformer}, and and learning these correlations could be beneficial for more robust frequency representation; 3) The permutation-invariant nature of the attention mechanism naturally aligns with the order-insensitivity of variates.

Furthermore, we observe that for the frequency-domain representation, the attention matrix of vanilla attention often exhibits a low-rank characteristic, which reduces the diversity of representations. To address this issue, we propose a general solution: adding a learnable matrix to the original softmax attention matrix, followed by row-wise normalization. We term this approach  \textbf{enhanced attention} and name the overall model \textbf{FreEformer}. Despite its simplicity, the enhanced attention mechanism is proven effective both theoretically and empirically. The main contributions of this work are summarized as follows:

\begin{itemize}
    \item This paper presents a simple yet effective model, named FreEformer, for multivariate time series forecasting. FreEformer achieves robust cross-variate representation learning using the enhanced attention mechanism.
    \item Theoretical analysis and experimental results demonstrate that the enhanced attention mechanism increases the rank of the attention matrix and provides greater flexibility for gradient flow. As a plug-in module, it consistently enhances the performance of existing Transformer-based forecasters.
    \item Empirically, FreEformer consistently achieves state-of-the-art forecasting performance across 18 real-world benchmarks spanning diverse domains such as electricity, transportation, weather, healthcare and finance.
\end{itemize}

\begin{figure*}[t]
   \centering
   \includegraphics[width=0.9\linewidth]{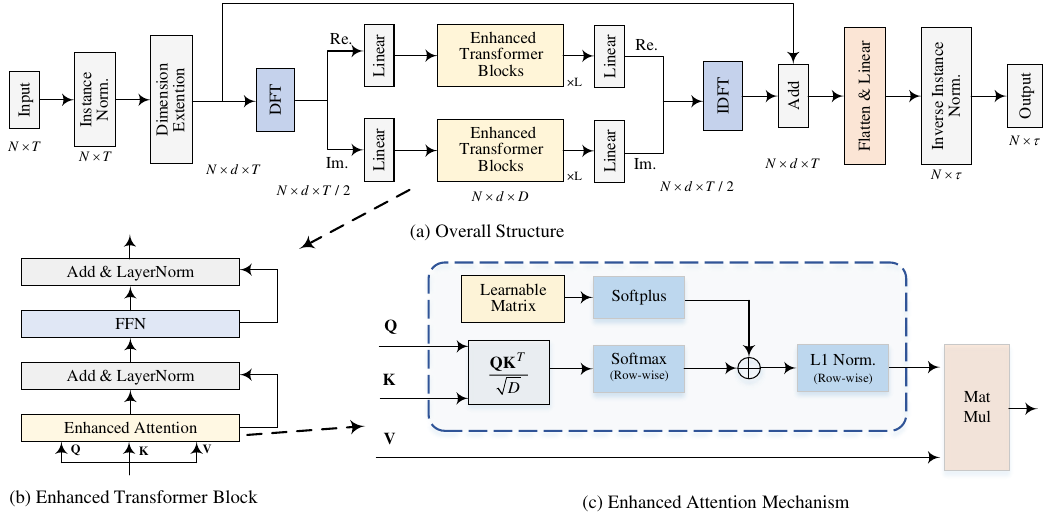}
   \caption{Overall structure of the FreEformer. We leverage the frequency spectrum to capture temporal patterns and employ an enhanced Transformer to model dependencies among multivariate spectra. The enhanced Transformer introduces a learnable matrix to the attention mechanism, which, as shown through theoretical analysis, addresses potential low-rank issues and improves gradient flow.}
   \label{fig2}
\end{figure*}

\section{Related Works}

\subsection{Transformer-Based Forecasters}

Classic works such as Autoformer \cite{autoformer}, Informer \cite{informer2021}, Pyraformer \cite{pyraformer}, FEDformer \cite{fedformer}, and PatchTST \cite{patchtst} represent early Transformer-based time series forecasters. iTransformer \cite{itransformer} introduces the inverted Transformer to capture multivariate dependencies, and achieves accurate forecasts. More recently, research has focused on jointly modeling cross-time and cross-variate dependencies \cite{crossformer,card,cicd}. Leddam \cite{Leddam_icml} uses a dual-attention module for decomposed seasonal components and linear layers for trend components. Unlike previous models in the time domain, we shift our focus to the frequency domain to explore dependencies among the frequency spectra of multiple variables for more robust representations.

\subsection{Frequency-Domain Forecasters}

Frequency analysis is an important tool in time series forecasting \cite{fre_survey}. FEDformer \cite{fedformer} performs DFT and sampling prior to Transformer. DEPTS~\cite{depts} uses the DFT to capture periodic patterns for better forecasts. FiLM~\cite{film} applies Fourier analysis to preserve historical information while mitigating noise. FreTS~\cite{frets} employs frequency-domain MLPs to model channel and temporal dependencies. FourierGNN~\cite{fouriergnn} transfers GNN operations from the time domain to the frequency domain. FITS~\cite{fits} applies a low-pass filter and complex-valued linear projection in the frequency domain. DERITS~\cite{derivative} introduces a Fourier derivative operator to address non-stationarity. Fredformer~\cite{fredformer} addresses frequency bias by dividing the frequency spectrum into patches. FAN~\cite{fan_norm} introduces frequency adaptive normalization for non-stationary data. In this work, we adopt a simple yet effective Transformer-based model to capture multivariate correlations in the frequency domain, outperforming existing methods.

\subsection{Transformer Variants}

Numerous variants of the vanilla Transformer have been developed to enhance efficiency and performance. Informer \cite{informer2021} introduces a ProbSparse self-attention mechanism with $\mathcal{O}(N \mathrm{log} N)$ complexity. Flowformer \cite{flowformer} proposes Flow-Attention, achieving linear complexity based on flow network theory. Reformer~\cite{reformer} reduces complexity by replacing dot-product attention with locality-sensitive hashing. Linear Transformers, such as FLatten \cite{flattentrans} and LSoftmax~\cite{linear_softmax}, achieve linear complexity by precomputing $\mathbf{K} ^T \mathbf{V}$ and designing various mapping functions. FlashAttention \cite{flashattention} accelerates computations by tiling to minimize GPU memory operations. LASER \cite{laser} mitigates the gradient vanishing issue using exponential transformations. In this work, we focus on the low-rank issue and adopt a simple yet effective strategy by adding a learnable matrix to the attention matrix. This improves both matrix rank and gradient flow with minimal modifications to the vanilla attention mechanism.

\section{Preliminaries} \label{preliminaries}

The discrete Fourier transform (DFT) \cite{signal_and_system} converts a signal $\mathbf{x} \in \mathbb{R}^N$ into its frequency spectrum $\mathcal{F} \in \mathbb{C}^N$. For $k=0,1,\cdots,N-1$, we have 

\begin{equation}
\begin{aligned}
&\mathcal{F}[k]  = \sum_{n=0}^{N-1}e^{-j\frac{2\pi}{N}nk}\mathbf{x}[n].
\end{aligned}
\label{eq1}
\end{equation}


\noindent Here, $j$ denotes the imaginary unit. For a real-valued vector $\mathbf{x}$, $\mathcal{F}[k]$ is complex-valued and satisfies the property of \textit{Hermitian symmetry} \cite{signal_and_system}: $\mathcal{F}[k] = \left(\mathcal{F}[N-k]\right)^*$ for $k=1, \cdots, N-1$, where $\left(\cdot\right)^*$ denotes the complex conjugate. The DFT is a linear and reversible transform, with the inverse discrete Fourier transform (IDFT) being:

\begin{equation}
\mathbf{x}[n]=\frac{1}{N}\sum_{k=0}^{N-1}\mathcal{F}[k]\cdot e^{j \cdot  2\pi\frac{k}{N}n},\: k=0,1,\cdots,N-1.
\label{eq2}
\end{equation}

Linear projections in the frequency domain are widely employed in works such as FreTS \cite{frets} and FITS \cite{fits}. The following theorem establishes their equivalent operations in the time domain.

\begin{theorem}[Frequency-domain linear projection and time-domain convolutions]\label{theorem1}
Given the time series $\mathbf{x}\in \mathbb{R}^N$ and its corresponding frequency spectrum $\mathcal{F}\in \mathbb{C}^N $. Let $\mathbf{W} \in \mathbb{C}^{N \times N}$ denote a weight matrix and $\mathbf{b} \in \mathbb{C}^N$ a bias vector. Under these definitions, the following DFT pair holds:

\begin{equation}
\begin{aligned}
\tilde{\mathcal{F}}=\mathbf{W}\mathcal{F}+\mathbf{b} \leftrightarrow \sum_{i=0}^{N-1}
{\Omega _i \circledast \mathcal{M}  _i(\mathbf{x})} + \mathrm{IDFT} (\mathbf{b}),
\end{aligned}
\label{eq3}
\end{equation}

\noindent where

\begin{equation}
\begin{aligned}
&w_i=\left [ 
\mathrm{diag} (\mathbf{W},i) , \mathrm{diag} (\mathbf{W},i-N) \right ] \in \mathbb{C}^N, \\
&\Omega _i = \mathrm{IDFT}\left( w_i\right)\in \mathbb{C}^N, \\
&\mathcal{M}  _i(\mathbf{x})=\mathbf{x} \odot \left [ \mathrm{e}^{-j\frac{2\pi }{N}ik }  \right ]_{k=0,1,\cdots N-1} \in \mathbb{C}^N. 
\end{aligned}
\label{eq4}
\end{equation}

\noindent Here, $\circledast$ denotes the circular convolution, and $\odot$ represents the Hadamard (element-wise) product. The notation $\left[\cdot, \cdot\right]$ indicates the concatenation of two vectors. $\mathrm{diag}(\mathbf{W},i) \in \mathbb{C}^{N-\left | i \right | } $ extracts the $i$-th diagonal of $\mathbf{W}$. $\mathcal{M}  _i(\mathbf{x})$ represents the $i$-th modulated version of $\mathbf{x}$, with $\mathcal{M}  _0(\mathbf{x})$ being $\mathbf{x}$ itself.

\end{theorem}

We provide the proof of this theorem in Section A of the appendix. \textbf{Theorem 1} extends Theorem 2 from FreTS \cite{frets} and demonstrates that a linear transformation in the frequency domain is equivalent to the sum of circular convolution operations applied to the series and its modulated versions. This equivalence highlights the computational simplicity of performing such operations in the frequency domain compared to the time domain.

\section{Method}

In multivariate time series forecasting, we consider historical series within a lookback window of $T$, each timestamp with $N$ variates: $\mathbf{x} =\left \{ \mathrm{x} _{1}, \cdots, \mathrm{x} _{T}  \right \} \in \mathbb{R}^{N \times T} $. Our task is to predict future $\tau$ timestamps  to closely approximate the ground truth $\mathbf{y} =\left \{ \mathrm{x}_{T+1},\cdots, \mathrm{x}_{T+\tau}   \right \} \in \mathbb{R}^{N \times \tau}$.

\subsection{Overall Architecture}

As shown in Figure \ref{fig2}, FreEformer employs a simple architecture. First, an instance normalization layer, specifically RevIN \cite{revin}, is used to normalize the input data and de-normalize the results at the final stage to mitigate non-stationarity. The constant mean component, represented by the zero-frequency point in the frequency domain, is set to zero during normalization. Subsequently, a dimension extension module is employed to enhance the model's representation capabilities. Specifically, the input $\mathbf{x}$ is expanded by a learnable weight vector $\phi_d \in \mathbb{R}^d$, yielding higher-dimensional and more expressive series data: $\bar{\mathbf{x}} = \mathbf{x} \times \phi_d \in \mathbb{R}^{N \times d \times T}$. We refer to $d$ as the embedding dimension.

\paragraph{Frequency-Domain Operations} Next, we apply the Discrete Fourier Transform (DFT) to convert the time series $\bar{\mathbf{x}}$ into its frequency spectrum along the temporal dimension:

\begin{equation}
\mathcal{F} = \mathrm{DFT} (\bar{\mathbf{x}} ) = \mathrm{Re} (\mathcal{F}) + j\cdot \mathrm{Im} (\mathcal{F}) \in \mathbb{C}^{N\times d\times T},
\label{eq5}
\end{equation}

\noindent where $\mathrm{Re}(\cdot)$ and $\mathrm{Im}(\cdot)$ represent the real and imaginary parts, respectively. Due to the \textit{conjugate symmetry} property of the frequency spectrum of a real-valued signal, only the first $\left \lceil (T+1)/2 \right \rceil$ elements of the real and imaginary parts need to be retained. Here, $\left \lceil \cdot \right \rceil$ denotes the ceiling operation.



\begin{table}[t]
\begin{center}
{\fontsize{7}{8}\selectfont
\setlength{\tabcolsep}{3pt}
\begin{tabular}{@{}c|ccc|cc|cc@{}}
\toprule
Dataset        & ECL                                   & Weather                               & Traffic                               & COVID-19                              & NASDAQ                                & COVID-19                              & NASDAQ                                \\ \midrule
$T-\tau$ & \multicolumn{3}{c|}{96-\{96,192,336,720\}}                                                                             & \multicolumn{2}{c|}{36-\{24,36,48,60\}}                                        & \multicolumn{2}{c}{12-\{3,6,9,12\}}                                           \\ \midrule
Concat.        & {\color[HTML]{FF0000} \textbf{0.162}} & 0.243                                 & 0.443                                 & 8.705                                 & 0.190                                 & 1.928                                 & {\color[HTML]{FF0000} \textbf{0.055}} \\
S.W.      & 0.165                                 & 0.240                                 & 0.440                                 & 8.520                                 & 0.189                                 & 1.895                                 & {\color[HTML]{FF0000} \textbf{0.055}} \\
N.S.W.  & {\color[HTML]{FF0000} \textbf{0.162}} & {\color[HTML]{FF0000} \textbf{0.239}} & {\color[HTML]{FF0000} \textbf{0.435}} & {\color[HTML]{FF0000} \textbf{8.435}} & {\color[HTML]{FF0000} \textbf{0.185}} & {\color[HTML]{FF0000} \textbf{1.892}} & {\color[HTML]{FF0000} \textbf{0.055}} \\ \bottomrule
\end{tabular}
}
\caption{Comparison of different processes for real and imaginary parts. Average MSEs are reported in this table. `S.W.' and `N.S.W.' denote `Shared Weights' and `Non-Shared Weights', respectively. `Concat.' denotes the concatenation method.}
\label{tab_real_imag}
\end{center}
\end{table}

To process the real and imaginary parts, common strategies include employing complex-valued layers \cite{frets,fits}, or concatenating the real and imaginary parts into a real-valued vector and subsequently projecting the results back \cite{fredformer}. In this work, we adopt a simple yet effective scheme: processing these two parts independently. As shown in Table \ref{tab_real_imag}, this independent processing scheme yields better performance for FreEformer.

After flattening the last two dimensions of the real and imaginary parts and projecting them into the hidden dimension $D$, we construct the frequency-domain variate tokens $\tilde{\mathrm{Re} }, \tilde{\mathrm{Im} } \in \mathbb{R}^{N \times D}$. These tokens are then fed into $L$ stacked Transformer blocks to capture multivariate dependencies among the spectra. Subsequently, the tokens are projected back to the lookback length. The real and imaginary parts are then regrouped to reconstruct the frequency spectrum. Then, the time-domain signal $\tilde{\mathbf{x}}$ is recovered using the IDFT. The entire process is summarized as follows:

\begin{equation}
\begin{aligned}
& \tilde{\mathcal{F}}  = \mathcal{F}[:,:,0:\left \lceil (T+1)/2 \right \rceil ]\in \mathbb{C}^{N\times d\times \left \lceil (T+1)/2 \right \rceil }, \\
&\tilde{\mathrm{Re}}^0=\mathrm{Linear} (\mathrm{Flatten} (\mathrm{Re}(\tilde{\mathcal{F}})))\in \mathbb{R}^{N\times D}\\
&\tilde{\mathrm{Re}}^{l+1}=\mathrm{TrmBlock} (\tilde{\mathrm{Re}}^l)\in \mathbb{R}^{N\times D }, l=0,\cdots ,L-1,\\
&\tilde{\mathrm{Re}}=\mathrm{Linear} (\mathrm{Reshape}(\tilde{\mathrm{Re}}^L))\in \mathbb{R}^{N\times d\times \left \lceil (T+1)/2 \right \rceil }, \\
&\tilde{\mathrm{Im}}^0=\mathrm{Linear} (\mathrm{Flatten} (\mathrm{Im}(\tilde{\mathcal{F}})))\in \mathbb{R}^{N\times D}\\
&\tilde{\mathrm{Im}}^{l+1}=\mathrm{TrmBlock} (\tilde{\mathrm{Im}}^l)\in \mathbb{R}^{N\times D }, l=0,\cdots ,L-1,\\
&\tilde{\mathrm{Im}}=\mathrm{Linear} (\mathrm{Reshape}(\tilde{\mathrm{Im}}^L))\in \mathbb{R}^{N\times d\times \left \lceil (T+1)/2 \right \rceil }, \\
&\tilde{\mathbf{x}}=\mathrm{IDFT} (\tilde{\mathrm{Re}}+j\cdot \tilde{\mathrm{Im}})\in \mathbb{R}^{N \times d \times T }. \\
\end{aligned}
\label{eq7}
\end{equation}

\noindent In the above equation, the final step is implemented via the {\tt irfft} function in PyTorch to ensure real-valued outputs.

\paragraph{Prediction Head} A shortcut connection is applied to sum $\tilde{\mathbf{x}}$ with the original $\bar{\mathbf{x}}$. Finally, a flatten layer and a linear head are used to ensure the output matches the desired size. The final result is obtained through a de-normalization step:

\begin{equation}
\mathbf{\hat{y}}=\mathrm{DeNorm} \left ( \mathrm{Linear}(\mathrm{Flatten}(\tilde{\mathbf{x}}+\bar{\mathbf{x}}) ) \right )  \in \mathbb{R}^{N\times \tau }.
\label{eq8}
\end{equation}

\begin{figure}[t]
   \centering
   \includegraphics[width=1.0\linewidth]{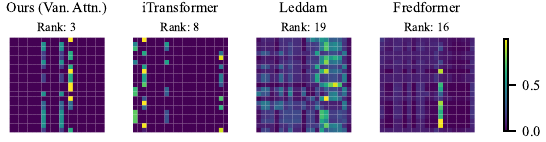}
   \caption{Attention matrices from state-of-the-art forecasters on the Weather dataset. The FreEformer with vanilla attention typically exhibits a low rank, likely due to the inherent sparsity of the frequency spectrum and the strong-value-focused nature of the Softmax function in vanilla~attention.}
   \label{fig_rank_compare}
\end{figure}

\subsection{Enhanced Attention}

\begin{figure}[t]
   \centering
   \includegraphics[width=1.0\linewidth]{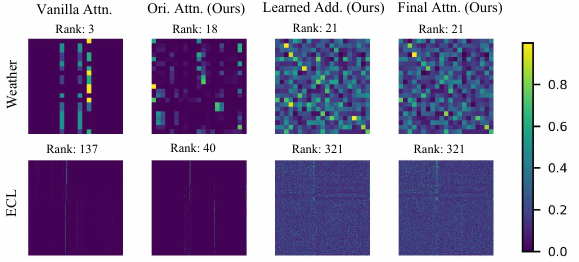}
   \caption{Attention matrices from vanilla and enhanced attention. The left column shows the low-rank attention matrix from the vanilla attention (Weather: 3, ECL: 137), with most entries near zero. The right three columns show  the original attention matrix ($\mathbf{A} $), the learned addition matrix ($\mathrm{Softplus} (\mathbf{B}) $), and the final attention matrix ($\mathrm{Norm} \left ( \mathbf{A}+\mathrm{Softplus} (\mathbf{B}) \right ) $). The final matrix exhibits more prominent values and higher ranks (Weather: 21, ECL: 321).}
   \label{fig3}
\end{figure}

In the Transformer block, as shown in Figure \ref{fig2}(b), we first employ the attention mechanism to capture cross-variate dependencies. Then the LayerNorm and FFN are used to update frequency representations in a variate-independent manner. According to Theorem \ref{theorem1}, the FFN corresponds to a series of convolution operations in the time domain for series representations. The vanilla attention mechanism is defined as:

\begin{equation}
\mathrm{Attn}\left ( \mathbf{Q} ,\mathbf{K} ,\mathbf{V}  \right )  =\underbrace{\mathrm{Softmax} \left ( \frac{\mathbf{Q}\mathbf{K}^T}{\sqrt{D} }  \right )}_{\text{Attention Matrix}\triangleq \mathbf{A}}  \mathbf{V} .
\label{eq10}
\end{equation}

\noindent Here, $\mathbf{Q},\mathbf{K},\mathbf{V}\in \mathbb{R}^{N\times D}$ are the query, key and value matrix, respectively, obtained through linear projections. We denote $D$ as the feature dimension and refer to $\mathrm{Softmax} \left ( \mathbf{Q} \mathbf{K}^T / \sqrt{D} \right )$, represented as $\mathbf{A}$, as the \textit{attention matrix}.

However, as shown in Figure \ref{fig_rank_compare}, compared to other state-of-the-art forecasters, FreEformer with the vanilla attention mechanism usually exhibits an attention matrix with a lower rank. This could arise from the inherent sparsity of the frequency spectrum \cite{signal_and_system} and the strong-value-focused properties of the vanilla attention mechanism \cite{laser,nystromformer}. While patching adjacent frequency points can mitigate sparsity (as in Fredformer), we address the underlying low-rank issue within the attention mechanism itself, offering a more general solution.


In this work, we adopt a straightforward yet effective solution: introducing a learnable matrix $\mathbf{B}$ to the attention matrix. The enhanced attention mechanism, denoted as $\mathrm{EnhAttn}\left ( \mathbf{Q} ,\mathbf{K} ,\mathbf{V} \right )$, is defined as \footnote{The variants are discussed in Section C.3 of the appendix.  }:

\begin{equation}
\mathrm{Norm} \left ( \mathrm{Softmax} \left ( \frac{\mathbf{Q}\mathbf{K}^T}{\sqrt{D} }  \right ) + \mathrm{Softplus} (\mathbf{B}) \right )  \mathbf{V},
\label{eq11}
\end{equation}

\noindent where $\mathrm{Norm}(\cdot)$ denotes row-wise L1 normalization. The $\mathrm{Softplus}(\cdot)$ function ensures positive entries, thereby preventing potential division-by-zero errors in $\mathrm{Norm}(\cdot)$.

\subsubsection{Theoretical Analysis} \label{theoretical_analysis}


\paragraph{Feature Diversity} 
According to Equation \eqref{eq11}, feature diversity is directly influenced by the rank of the final  attention matrix $\mathrm{Norm} \left ( \mathbf{A}+\bar{\mathbf{B}}  \right )$, where $\bar{\mathbf{B}} \triangleq \mathrm{Softplus} (\mathbf{B})$. Since row-wise L1 normalization does not alter the rank of a matrix, we have: $\mathrm{rank} \left ( \mathrm{Norm} \left ( \mathbf{A}+\bar{\mathbf{B}}  \right ) \right ) =\mathrm{rank} \left (  \mathbf{A}+\bar{\mathbf{B}} \right )$.  For further analysis, we present the following theorem:

\begin{theorem}\label{theorem2}

Let $\mathbf{A}$ and $\mathbf{B}$ be two matrices of the same size $N\times N$. The rank of their sum satisfies the following bounds:

\begin{equation}
\left | \mathrm{rank}(\mathbf{A})-\mathrm{rank}(\mathbf{B}) \right |  
\le \mathrm{rank}(\mathbf{A}+\mathbf{B})  
\le 
\mathrm{rank}(\mathbf{A})+\mathrm{rank}(\mathbf{B})
\label{eq12}
\end{equation}

\end{theorem}

The proof is provided in Section B of the appendix. As illustrated in Figure \ref{fig3}, the original attention matrix $\mathbf{A}$ often exhibits a low rank, whereas the learned matrix $\bar{\mathbf{B}}$ is nearly full-rank. According to Theorem \ref{theorem2}, the combined matrix $\mathbf{A} + \bar{\mathbf{B}}$ generally achieves a higher rank. This observation aligns with the results shown in Figure \ref{fig3}.

\begin{figure}[t]
   \centering
   \includegraphics[width=1.0\linewidth]{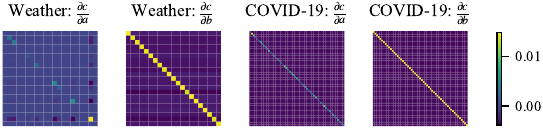}
   \caption{Illustration of the Jacobian matrices of $\mathbf{c}$ with respect to $\mathbf{a}$ and $\mathbf{b}$ for the Weather and COVID-19 datasets.}
   \label{fig_grad_mat}
\end{figure}

\paragraph{Gradient Flow} 
Let $\mathbf{a} \in \mathbb{R}^N$ denote a row in $\mathbf{Q}\mathbf{K}^T / \sqrt{D}$. For vanilla attention, the transformation is  $\tilde{\mathbf{a}} \triangleq \mathrm{Softmax}(\mathbf{a})$. Then the Jacobian matrix of $\tilde{\mathbf{a}}$ regarding $\mathbf{a}$ can be derived as:

\begin{equation}
\frac{\partial \tilde{\mathbf{a}} }{\partial \mathbf{a} }  = \mathrm{Diag} (\tilde{\mathbf{a}}) - \tilde{\mathbf{a}}\tilde{\mathbf{a}}^T,
\label{eq13}
\end{equation}

\noindent where $\mathrm{Diag} (\tilde{\mathbf{a}})$ is a diagonal matrix with $\tilde{\mathbf{a}}$ as its diagonal. 


For the enhanced attention, the transformation is given by:

\begin{equation}
\mathbf{c} =\mathrm{Norm} \left ( \mathrm{Softmax}(\mathbf{a}) + \mathbf{b}  \right ) ,
\label{eq14}
\end{equation}

\noindent where $\mathbf{b}$ represents a row of $\mathrm{Softplus} (\mathbf{B})$. The Jacobian matrices of $\mathbf{c}$ with respect to $\mathbf{a}$ and $\mathbf{b}$ can be derived as:

\begin{equation}
\begin{aligned}
\frac{\partial \mathbf{c}}{\partial \mathbf{a}} & = \frac{1}{\left \| \tilde{\mathbf{b}} \right \|_1 } 
\left ( \mathrm{Diag} (\tilde{\mathbf{a}}) - \tilde{\mathbf{a}}\tilde{\mathbf{a}}^T  \right ),\\
\frac{\partial \mathbf{c}}{\partial \mathbf{b}} & =\frac{1}{\left \| \tilde{\mathbf{b}} \right \|_1^2 } 
\left ( \left \| \tilde{\mathbf{b}} \right \|_1 \cdot \mathbf{I} - \tilde{\mathbf{b}}\mathbf{1}^T   \right ),
\end{aligned}
\label{eq15}
\end{equation}

\noindent where $\tilde{\mathbf{b}} = \tilde{\mathbf{a}} + \mathbf{b}=\mathrm{Softmax}(\mathbf{a})+ \mathbf{b}$, $\mathbf{I}\in \mathbb{R}^{N\times N}$ is the unit matrix; $\mathbf{1}=\left [ 1,\cdots ,1 \right ]^T \in \mathbb{R}^N$. The detailed proofs of Equations \eqref{eq13} and \eqref{eq15} are provided in Section C of the appendix.

We can see that $\partial \mathbf{c}/\partial \mathbf{a}$ in Equation \eqref{eq15} shares the same structure as that of vanilla attention in Equation \eqref{eq13}, except for the scaling factor $1/\left | \tilde{\mathbf{b}} \right |_1$.  Since $\left \| \tilde{\mathbf{b}} \right \|_1 = 1 + \left \| \mathbf{b} \right \|_1 > 1$ and $\tilde{\mathbf{b}}$ is learnable, the gradient is scaled down by a learnable factor, providing extra flexibility in gradient control.

Moreover, as shown in Figure \ref{fig_grad_mat}, $\partial \mathbf{c}/\partial \mathbf{b}$ exhibits a pronounced diagonal than $\partial \mathbf{c}/\partial \mathbf{a}$, suggesting a stronger dependence of $\mathbf{c}$ on $\mathbf{b}$ than $\mathbf{a}$. This aligns with the design, as $\mathbf{b}$ directly modulates the attention weights.

\paragraph{Combination of MLP and Vanilla Attention} 
We now provide a new perspective on the enhanced attention. In Equation \eqref{eq11}, the attention matrix is decomposed into two components: the input-independent, dataset-specific term $\bar{\mathbf{B}}$, and the input-dependent term $\mathbf{A}$. If $\mathbf{A}$ is zero, the enhanced attention reduces to a linear transformation of $\mathbf{V}$, effectively functioning as an MLP along the variate dimension. By jointly optimizing $\mathbf{A}$ and $\bar{\mathbf{B}}$, the enhanced attention can be interpreted as an adaptive combination of MLP and vanilla attention.

\section{Experiments}

\paragraph{Datasets and Implementation Details} 
We extensively evaluate the FreEformer using eighteen real-world datasets: \textbf{ETT} (four subsets), \textbf{Weather}, \textbf{ECL}, \textbf{Traffic}, \textbf{Exchange}, \textbf{Solar-Energy}, \textbf{PEMS} (four subsets), \textbf{ILI}, \textbf{COVID-19}, \textbf{ECG}, \textbf{METR-LA}, \textbf{NASDAQ} and \textbf{Wiki}. During training, we adopt the L1 loss function from CARD \cite{card}. The embedding dimension $d$ is fixed at 16, and the dimension $D$ is selected from $\left \{ 128, 256, 512 \right \}$. The dataset description and implementation details are provided in the appendix.

\subsection{Forecasting Performance}

\begin{table*}[t]
\begin{center}
{\fontsize{7}{8}\selectfont
\setlength{\tabcolsep}{2.8pt}
\begin{tabular}{@{}c|cc|cc|cc|cc|cc|cc|cc|cc|cc|cc|cc@{}}
\toprule
Model       & \multicolumn{2}{c|}{\begin{tabular}[c]{@{}c@{}}FreEformer\\ (Ours)\end{tabular}} & \multicolumn{2}{c|}{\begin{tabular}[c]{@{}c@{}}Leddam\\ \shortcite{Leddam_icml} \end{tabular}}    & \multicolumn{2}{c|}{\begin{tabular}[c]{@{}c@{}}CARD\\ \shortcite{card} \end{tabular}}         & \multicolumn{2}{c|}{\begin{tabular}[c]{@{}c@{}}Fredformer\\ \shortcite{fredf} \end{tabular}}   & \multicolumn{2}{c|}{\begin{tabular}[c]{@{}c@{}}iTrans.\\ \shortcite{itransformer} \end{tabular}}   & \multicolumn{2}{c|}{\begin{tabular}[c]{@{}c@{}}TimeMixer\\ \shortcite{timemixer} \end{tabular}} & \multicolumn{2}{c|}{\begin{tabular}[c]{@{}c@{}}PatchTST\\ \shortcite{patchtst} \end{tabular}} & \multicolumn{2}{c|}{\begin{tabular}[c]{@{}c@{}}Crossfm.\\ \shortcite{crossformer} \end{tabular}} & \multicolumn{2}{c|}{\begin{tabular}[c]{@{}c@{}}TimesNet\\ \shortcite{timesnet} \end{tabular}} & \multicolumn{2}{c|}{\begin{tabular}[c]{@{}c@{}}FreTS\\ \shortcite{frets} \end{tabular}} & \multicolumn{2}{c}{\begin{tabular}[c]{@{}c@{}}DLinear\\ \shortcite{linear} \end{tabular}} \\ \midrule
Metric       & MSE                                    & MAE                                    & MSE                                   & MAE                                & MSE                                   & MAE                                   & MSE                                   & MAE                                   & MSE                                   & MAE                                & MSE                                   & MAE                                & MSE                                 & MAE                                 & MSE                                 & MAE                                 & MSE                                 & MAE                                 & MSE                                & MAE                               & MSE                                 & MAE                                \\ \midrule
ETTm1        & {\color[HTML]{FF0000} \textbf{0.379}}  & {\color[HTML]{FF0000} \textbf{0.381}}  & 0.386                                 & 0.397                              & 0.383                                 & {\color[HTML]{0000FF} {\ul 0.384}}    & 0.384                                 & 0.395                                 & 0.407                                 & 0.410                              & {\color[HTML]{0000FF} {\ul 0.381}}    & 0.395                              & 0.387                               & 0.400                               & 0.513                               & 0.496                               & 0.400                               & 0.406                               & 0.407                              & 0.415                             & 0.403                               & 0.407                              \\
ETTm2        & {\color[HTML]{FF0000} \textbf{0.272}}  & {\color[HTML]{FF0000} \textbf{0.313}}  & 0.281                                 & 0.325                              & {\color[HTML]{FF0000} \textbf{0.272}} & {\color[HTML]{0000FF} {\ul 0.317}}    & {\color[HTML]{0000FF} {\ul 0.279}}    & 0.324                                 & 0.288                                 & 0.332                              & 0.275                                 & {\color[HTML]{0000FF} {\ul 0.323}} & 0.281                               & 0.326                               & 0.757                               & 0.610                               & 0.291                               & 0.333                               & 0.335                              & 0.379                             & 0.350                               & 0.401                              \\
ETTh1        & {\color[HTML]{0000FF} {\ul 0.433}}     & 0.431                                  & {\color[HTML]{FF0000} \textbf{0.431}} & {\color[HTML]{0000FF} {\ul 0.429}} & 0.442                                 & {\color[HTML]{0000FF} {\ul 0.429}}    & 0.435                                 & {\color[HTML]{FF0000} \textbf{0.426}} & 0.454                                 & 0.447                              & 0.447                                 & 0.440                              & 0.469                               & 0.454                               & 0.529                               & 0.522                               & 0.458                               & 0.450                               & 0.488                              & 0.474                             & 0.456                               & 0.452                              \\
ETTh2        & 0.372                                  & {\color[HTML]{0000FF} {\ul 0.393}}     & 0.373                                 & 0.399                              & 0.368                                 & {\color[HTML]{FF0000} \textbf{0.390}} & {\color[HTML]{0000FF} {\ul 0.365}}    & {\color[HTML]{0000FF} {\ul 0.393}}    & 0.383                                 & 0.407                              & {\color[HTML]{FF0000} \textbf{0.364}} & 0.395                              & 0.384                               & 0.405                               & 0.942                               & 0.684                               & 0.414                               & 0.427                               & 0.550                              & 0.515                             & 0.559                               & 0.515                              \\
ECL          & {\color[HTML]{FF0000} \textbf{0.162}}  & {\color[HTML]{FF0000} \textbf{0.251}}  & 0.169                                 & 0.263                              & {\color[HTML]{0000FF} {\ul 0.168}}    & {\color[HTML]{0000FF} {\ul 0.258}}    & 0.176                                 & 0.269                                 & 0.178                                 & 0.270                              & 0.182                                 & 0.272                              & 0.208                               & 0.295                               & 0.244                               & 0.334                               & 0.192                               & 0.295                               & 0.202                              & 0.290                             & 0.212                               & 0.300                              \\
Exchange     & {\color[HTML]{0000FF} {\ul 0.354}}     & {\color[HTML]{0000FF} {\ul 0.399}}     & {\color[HTML]{0000FF} {\ul 0.354}}    & 0.402                              & 0.362                                 & 0.402                                 & {\color[HTML]{FF0000} \textbf{0.333}} & {\color[HTML]{FF0000} \textbf{0.391}} & 0.360                                 & 0.403                              & 0.387                                 & 0.416                              & 0.367                               & 0.404                               & 0.940                               & 0.707                               & 0.416                               & 0.443                               & 0.416                              & 0.439                             & {\color[HTML]{0000FF} {\ul 0.354}}                               & 0.414                              \\
Traffic      & 0.435                                  & {\color[HTML]{FF0000} \textbf{0.251}}  & 0.467                                 & 0.294                              & 0.453                                 & {\color[HTML]{0000FF} {\ul 0.282}}    & {\color[HTML]{0000FF} {\ul 0.433}}    & 0.291                                 & {\color[HTML]{FF0000} \textbf{0.428}} & {\color[HTML]{0000FF} {\ul 0.282}} & 0.484                                 & 0.297                              & 0.531                               & 0.343                               & 0.550                               & 0.304                               & 0.620                               & 0.336                               & 0.538                              & 0.328                             & 0.625                               & 0.383                              \\
Weather      & {\color[HTML]{FF0000} \textbf{0.239}}  & {\color[HTML]{FF0000} \textbf{0.260}}  & 0.242                                 & 0.272                              & {\color[HTML]{FF0000} \textbf{0.239}} & {\color[HTML]{0000FF} {\ul 0.265}}    & 0.246                                 & 0.272                                 & 0.258                                 & 0.279                              & {\color[HTML]{0000FF} {\ul 0.240}}    & 0.271                              & 0.259                               & 0.281                               & 0.259                               & 0.315                               & 0.259                               & 0.287                               & 0.255                              & 0.298                             & 0.265                               & 0.317                              \\
Solar & {\color[HTML]{0000FF} {\ul 0.217}}     & {\color[HTML]{FF0000} \textbf{0.219}}  & 0.230                                 & 0.264                              & 0.237                                 & {\color[HTML]{0000FF} {\ul 0.237}}    & 0.226                                 & 0.262                                 & 0.233                                 & 0.262                              & {\color[HTML]{FF0000} \textbf{0.216}} & 0.280                              & 0.270                               & 0.307                               & 0.641                               & 0.639                               & 0.301                               & 0.319                               & 0.226                              & 0.254                             & 0.330                               & 0.401                              \\
PEMS03      & {\color[HTML]{FF0000} \textbf{0.102}}  & {\color[HTML]{FF0000} \textbf{0.206}}  & {\color[HTML]{0000FF} {\ul 0.107}}    & {\color[HTML]{0000FF} {\ul 0.210}} & 0.174                                 & 0.275                                 & 0.135                                 & 0.243                                 & 0.113                                 & 0.221                              & 0.167                                 & 0.267                              & 0.180                               & 0.291                               & 0.169                               & 0.281                               & 0.147                               & 0.248                               & 0.169                              & 0.278                             & 0.278                               & 0.375                              \\
PEMS04      & {\color[HTML]{FF0000} \textbf{0.094}}  & {\color[HTML]{FF0000} \textbf{0.196}}  & {\color[HTML]{0000FF} {\ul 0.103}}    & {\color[HTML]{0000FF} {\ul 0.210}} & 0.206                                 & 0.299                                 & 0.162                                 & 0.261                                 & 0.111                                 & 0.221                              & 0.185                                 & 0.287                              & 0.195                               & 0.307                               & 0.209                               & 0.314                               & 0.129                               & 0.241                               & 0.188                              & 0.294                             & 0.295                               & 0.388                              \\
PEMS07      & {\color[HTML]{FF0000} \textbf{0.080}}  & {\color[HTML]{FF0000} \textbf{0.167}}  & {\color[HTML]{0000FF} {\ul 0.084}}    & {\color[HTML]{0000FF} {\ul 0.180}} & 0.149                                 & 0.247                                 & 0.121                                 & 0.222                                 & 0.101                                 & 0.204                              & 0.181                                 & 0.271                              & 0.211                               & 0.303                               & 0.235                               & 0.315                               & 0.124                               & 0.225                               & 0.185                              & 0.282                             & 0.329                               & 0.395                              \\
PEMS08      & {\color[HTML]{FF0000} \textbf{0.110}}  & {\color[HTML]{FF0000} \textbf{0.194}}  & {\color[HTML]{0000FF} {\ul 0.122}}    & {\color[HTML]{0000FF} {\ul 0.211}} & 0.201                                 & 0.280                                 & 0.161                                 & 0.250                                 & 0.150                                 & 0.226                              & 0.226                                 & 0.299                              & 0.280                               & 0.321                               & 0.268                               & 0.307                               & 0.193                               & 0.271                               & 0.212                              & 0.297                             & 0.379                               & 0.416                              \\ \bottomrule
\end{tabular}}
\caption{Long-term time series forecasting results for $T=96$ and $\tau \in \left \{ 96,192,336,720 \right \}$. For PEMS, $\tau \in \left \{ 12,24,48,96 \right \}$.  Results are averaged across these prediction lengths. These settings are used throughout the following tables.}
\label{tab_long_term}
\end{center}
\end{table*}


\begin{table*}[t]
\begin{center}
{\fontsize{7}{8}\selectfont
\setlength{\tabcolsep}{3.2pt}
\begin{tabular}{@{}cc|cc|cc|cc|cc|cc|cc|cc|cc|cc|cc@{}}
\toprule
\multicolumn{2}{c|}{Model}      & \multicolumn{2}{c|}{\begin{tabular}[c]{@{}c@{}}FreEformer\\ (Ours)\end{tabular}} & \multicolumn{2}{c|}{\begin{tabular}[c]{@{}c@{}}Leddam\\ \shortcite{Leddam_icml} \end{tabular}}    & \multicolumn{2}{c|}{\begin{tabular}[c]{@{}c@{}}CARD\\ \shortcite{card} \end{tabular}}         & \multicolumn{2}{c|}{\begin{tabular}[c]{@{}c@{}}Fredformer\\ \shortcite{fredf} \end{tabular}}   & \multicolumn{2}{c|}{\begin{tabular}[c]{@{}c@{}}iTrans.\\ \shortcite{itransformer} \end{tabular}}   & \multicolumn{2}{c|}{\begin{tabular}[c]{@{}c@{}}TimeMixer\\ \shortcite{timemixer} \end{tabular}} & \multicolumn{2}{c|}{\begin{tabular}[c]{@{}c@{}}PatchTST\\ \shortcite{patchtst} \end{tabular}} &  \multicolumn{2}{c|}{\begin{tabular}[c]{@{}c@{}}TimesNet\\ \shortcite{timesnet} \end{tabular}} & \multicolumn{2}{c|}{\begin{tabular}[c]{@{}c@{}}DLinear\\ \shortcite{linear} \end{tabular}} 
& \multicolumn{2}{c}{\begin{tabular}[c]{@{}c@{}}FreTS\\ \shortcite{frets} \end{tabular}}    \\ \midrule
\multicolumn{2}{c|}{Metric}     & MSE                                   & MAE                                   & MSE                                & MAE                                & MSE                                & MAE                                & MSE                                   & MAE                                & MSE                                & MAE                                & MSE                                   & MAE                                   & MSE                                & MAE                                & MSE           & MAE           & MSE                                & MAE   & MSE                                   & MAE   \\ \midrule
                           & S1 & {\color[HTML]{FF0000} \textbf{1.140}} & {\color[HTML]{FF0000} \textbf{0.585}} & 1.468                              & 0.679                              & 1.658                              & 0.707                              & 1.518                                 & 0.696                              & {\color[HTML]{0000FF} {\ul 1.437}} & {\color[HTML]{0000FF} {\ul 0.659}} & 1.707                                 & 0.734                                 & 1.681                              & 0.723                              & 1.480         & 0.684         & 2.400                              & 1.034 & 1.839                                 & 0.782 \\
\multirow{-2}{*}{ILI}      & S2 & {\color[HTML]{FF0000} \textbf{1.906}} & {\color[HTML]{FF0000} \textbf{0.835}} & 1.982                              & {\color[HTML]{0000FF} {\ul 0.875}} & 2.260                              & 0.938                              & {\color[HTML]{0000FF} {\ul 1.947}}    & 0.899                              & 1.993                              & 0.887                              & 2.020                                 & 0.878                                 & 2.128                              & 0.885                              & 2.139         & 0.931         & 3.083                              & 1.217 & 3.036                                 & 1.174 \\ \midrule
                           & S1 & {\color[HTML]{FF0000} \textbf{1.892}} & {\color[HTML]{FF0000} \textbf{0.673}} & 2.064                              & 0.779                              & 2.059                              & 0.767                              & {\color[HTML]{0000FF} {\ul 1.902}}    & {\color[HTML]{0000FF} {\ul 0.765}} & 2.096                              & 0.795                              & 2.234                                 & 0.782                                 & 2.221                              & 0.820                              & 2.569         & 0.861         & 3.483                              & 1.102 & 2.516                                 & 0.862 \\
\multirow{-2}{*}{COVID-19} & S2 & {\color[HTML]{FF0000} \textbf{8.435}} & {\color[HTML]{FF0000} \textbf{1.764}} & {\color[HTML]{0000FF} {\ul 8.439}} & {\color[HTML]{0000FF} {\ul 1.792}} & 9.013                              & 1.862                              & 8.656                                 & 1.808                              & 8.506                              & {\color[HTML]{0000FF} {\ul 1.792}} & 9.604                                 & 1.918                                 & 9.451                              & 1.905                              & 9.644         & 1.877         & 13.075                             & 2.099 & 11.345                                & 1.958 \\ \midrule
                           & S1 & 0.336                                 & {\color[HTML]{FF0000} \textbf{0.221}} & {\color[HTML]{0000FF} {\ul 0.327}} & 0.243                              & 0.349                              & {\color[HTML]{0000FF} {\ul 0.233}} & 0.336                                 & 0.242                              & 0.338                              & 0.244                              & 0.334                                 & 0.245                                 & 0.335                              & 0.243                              & 0.344         & 0.253         & 0.341                              & 0.294 & {\color[HTML]{FF0000} \textbf{0.324}} & 0.279 \\
\multirow{-2}{*}{METR-LA}  & S2 & 0.840                                 & {\color[HTML]{FF0000} \textbf{0.406}} & 0.878                              & 0.490                              & 0.929                              & {\color[HTML]{0000FF} {\ul 0.466}} & 0.898                                 & 0.495                              & 0.916                              & 0.501                              & 0.881                                 & 0.499                                 & 0.893                              & 0.502                              & 0.890         & 0.488         & {\color[HTML]{0000FF} {\ul 0.819}} & 0.550 & {\color[HTML]{FF0000} \textbf{0.804}} & 0.543 \\ \midrule
                           & S1 & {\color[HTML]{FF0000} \textbf{0.055}} & {\color[HTML]{FF0000} \textbf{0.126}} & 0.059                              & 0.135                              & {\color[HTML]{0000FF} {\ul 0.057}} & {\color[HTML]{0000FF} {\ul 0.130}} & 0.059                                 & 0.135                              & 0.060                              & 0.137                              & {\color[HTML]{FF0000} \textbf{0.055}} & {\color[HTML]{FF0000} \textbf{0.126}} & 0.058                              & 0.132                              & 0.068         & 0.151         & 0.072                              & 0.170 & 0.080                                 & 0.184 \\
\multirow{-2}{*}{NASDAQ}   & S2 & {\color[HTML]{FF0000} \textbf{0.185}} & {\color[HTML]{FF0000} \textbf{0.277}} & 0.196                              & 0.286                              & 0.193                              & 0.284                              & 0.194                                 & 0.285                              & 0.207                              & 0.297                              & {\color[HTML]{0000FF} {\ul 0.186}}    & {\color[HTML]{0000FF} {\ul 0.281}}    & 0.198                              & 0.286                              & 0.255         & 0.343         & 0.228                              & 0.331 & 0.263                                 & 0.361 \\ \midrule
                           & S1 & 6.524                                 & {\color[HTML]{FF0000} \textbf{0.391}} & 6.547                              & 0.404                              & 6.553                              & {\color[HTML]{0000FF} {\ul 0.400}} & 6.705                                 & 0.406                              & 6.569                              & 0.405                              & 6.572                                 & 0.409                                 & {\color[HTML]{0000FF} {\ul 6.523}} & 0.404                              & 7.956         & 0.520         & 6.634                              & 0.481 & {\color[HTML]{FF0000} \textbf{6.521}} & 0.448 \\
\multirow{-2}{*}{Wiki}     & S2 & 6.259                                 & {\color[HTML]{FF0000} \textbf{0.442}} & 6.286                              & 0.463                              & 6.285                              & 0.453                              & {\color[HTML]{FF0000} \textbf{5.931}} & 0.453                              & 6.275                              & 0.458                              & 6.315                                 & 0.468                                 & 6.212                              & {\color[HTML]{0000FF} {\ul 0.444}} & 7.310         & 0.623         & 6.205                              & 0.539 & {\color[HTML]{0000FF} {\ul 6.147}}    & 0.505 \\ \bottomrule
\end{tabular}
}
\caption{Short-term time series forecasting results under two settings: S1 (Input-12, Predict-$\left \{ 3,6,9,12 \right \}$) and S2 (Input-36, Predict-$\left \{ 24,36,48,60 \right \}$). Average results are reported across four prediction lengths. S1 is the default setting in the following experiments. }
\label{tab_short_term}
\end{center}
\end{table*}

We choose 10 well-acknowledged deep forecasters as our baselines, including (1) Transformer-based models: Leddam \cite{Leddam_icml}, CARD \cite{card}, Fredformer \cite{fredformer}, iTransformer \cite{itransformer}, PatchTST \cite{patchtst}, Crossformer \cite{crossformer}; (2) Linear-based models: TimeMixer \cite{timemixer}, FreTS \cite{frets} and DLinear \cite{linear}; (3) TCN-based model: TimesNet \cite{timesnet}. 

Comprehensive results for long- and short-term forecasting are presented in Tables \ref{tab_long_term} and \ref{tab_short_term}, respectively, with the best results highlighted in {\color[HTML]{FF0000} \textbf{bold}} and the second-best {\color[HTML]{0000FF} {\ul underlined}}. FreEformer consistently outperforms state-of-the-art models across various prediction lengths and real-world domains. Compared with sophisticated time-domain-based models, such as Leddam and CARD, FreEformer achieves superior performance with a simpler architecture, benefiting from the global-level property of the frequency domain. Furthermore, its performance advantage over Fredformer, another Transformer- and frequency-based model, suggests that the deliberate patching of band-limited frequency spectra may introduce unnecessary noise, hindering forecasting accuracy. 

Notably, in Table \ref{tab_fre_models}, we compare FreEformer with additional frequency-based models, where it also demonstrates a clear performance advantage. The visualization results of FreEformer are presented in Figure \ref{fig5}. Furthermore, as shown in Table 10 of the appendix, FreEformer exhibits state-of-the-art performance with variable lookback lengths.

\begin{table}[t]
\begin{center}
{\fontsize{7}{8}\selectfont
\setlength{\tabcolsep}{2.6pt}
\begin{tabular}{@{}c|
cc|cc|cc|cc|cc}
\toprule
Models   & \multicolumn{2}{c|}{\begin{tabular}[c]{@{}c@{}}FreEformer\\ (Ours)\end{tabular}} & \multicolumn{2}{c|}{\begin{tabular}[c]{@{}c@{}}FITS\\ \shortcite{fits} \end{tabular}} & 
\multicolumn{2}{c|}{\begin{tabular}[c]{@{}c@{}}FAN\\ \shortcite{fan_norm} \end{tabular}} 
& \multicolumn{2}{c|}{\begin{tabular}[c]{@{}c@{}}FilterNet\\ \shortcite{filternet} \end{tabular}} 
& \multicolumn{2}{c}{\begin{tabular}[c]{@{}c@{}}FreDF\\ \shortcite{fredf} \end{tabular}}     
\\ \midrule
Metric   & MSE                                     & MAE                                    & MSE                                & MAE                               & MSE                               & MAE                               & MSE                                  & MAE                                  & MSE                                   & MAE                                \\ \midrule
ETT(Avg) & {\color[HTML]{FF0000} \textbf{0.364}}   & {\color[HTML]{FF0000} \textbf{0.380}}  & 0.408                              & 0.405                             & 0.405                             & 0.427                             & {\color[HTML]{0000FF} {\ul 0.367}}   & {\color[HTML]{0000FF} {\ul 0.384}}   & 0.369                                 & 0.384                              \\
ECL      & {\color[HTML]{FF0000} \textbf{0.162}}   & {\color[HTML]{FF0000} \textbf{0.251}}  & 0.384                              & 0.434                             & 0.208                             & 0.298                             & 0.201                                & 0.285                                & {\color[HTML]{0000FF} {\ul 0.170}}    & {\color[HTML]{0000FF} {\ul 0.259}} \\
Traffic  & {\color[HTML]{0000FF} {\ul 0.435}}      & {\color[HTML]{FF0000} \textbf{0.251}}  & 0.615                              & 0.370                             & 0.526                             & 0.357                             & 0.521                                & 0.340                                & {\color[HTML]{FF0000} \textbf{0.421}} & {\color[HTML]{0000FF} {\ul 0.279}} \\
Weather  & {\color[HTML]{FF0000} \textbf{0.239}}   & {\color[HTML]{FF0000} \textbf{0.260}}  & 0.273                              & 0.292                             & 0.247                             & 0.292                             & {\color[HTML]{0000FF} {\ul 0.248}}   & {\color[HTML]{0000FF} {\ul 0.274}}   & 0.254                                 & {\color[HTML]{0000FF} {\ul 0.274}} \\ \bottomrule
\end{tabular}
}
\caption{Comparison with additional state-of-the-art frequency-based models. Average results are reported across four prediction lengths. `Avg' refers to averages further computed over subsets.}
\label{tab_fre_models}
\end{center}
\end{table}

\begin{figure}[t]
   \centering
   \includegraphics[width=1.0\linewidth]{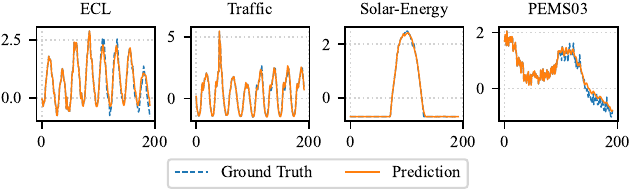}
   \caption{Visualization of the forecasting results under the `Input-96-Predict-96' setting, demonstrating accurate approximations. }
   \label{fig5}
\end{figure}

\subsection{Model Analysis}

\begin{table}[ht]
\begin{center}
{\fontsize{8}{10}\selectfont
\setlength{\tabcolsep}{4pt}
\begin{tabular}{@{}ccc|ccccc@{}}
\toprule
Layer  & Dim. & Patch. & ETTm1                                 & Weather                               & ECL                                   & Traffic                               & COVID-19                              \\ \midrule
Linear & Var. & \ding{55}     & 0.385                                 & 0.245                                 & 0.189                                 & 0.488                                 & 2.040                                 \\
Linear & Fre. & \ding{55}     & 0.386                                 & 0.246                                 & 0.184                                 & 0.482                                 & 2.086                                 \\ \midrule
Trans. & Fre. & \ding{52}    & {\color[HTML]{0000FF} {\ul 0.381}}    & 0.244                                 & 0.183                                 & 0.504                                 & 2.100                                 \\
Trans. & Fre. & \ding{55}     & 0.383                                 & 0.245                                 & {\color[HTML]{0000FF} {\ul 0.181}}    & 0.489                                 & 2.116                                 \\
Trans. & Var. & \ding{52}     & 0.385                                 & {\color[HTML]{0000FF} {\ul 0.241}}    & {\color[HTML]{FF0000} \textbf{0.162}} & {\color[HTML]{0000FF} {\ul 0.443}}    & {\color[HTML]{0000FF} {\ul 2.029}}    \\ \midrule
Trans. & Var. & \ding{55}     & {\color[HTML]{FF0000} \textbf{0.379}} & {\color[HTML]{FF0000} \textbf{0.239}} & {\color[HTML]{FF0000} \textbf{0.162}} & {\color[HTML]{FF0000} \textbf{0.435}} & {\color[HTML]{FF0000} \textbf{1.892}} \\ \bottomrule
\end{tabular}

}
\caption{Architecture ablations on layers, dimensions, and patching settings. Layers include linear and enhanced Transformer layers, while dimensions refer to frequency and variate dimensions. `Patch.' indicates patching along the frequency dimension, with patch length and stride set to 6 and 3 for COVID-19, and 16 and 8 for other datasets. Average MSEs are reported across four prediction horizons. The final row corresponds to the FreEformer configuration.}
\label{tab_arch_abl}
\end{center}
\end{table}

\paragraph{Architecture Ablations}
The FreEformer utilizes an enhanced Transformer architecture to capture cross-variate dependencies in the frequency domain. Table \ref{tab_arch_abl} presents a comparison of several FreEformer variants, evaluating the impact of linear and enhanced Transformer layers, different dimensional configurations, and patching along the frequency dimension. To ensure a fair comparison, the enhanced Transformer is used for all Transformer-based settings. The results indicate that: 1) Enhanced Transformer blocks outperform linear layers due to their superior representation capabilities; 2) Multivariate dependency learning generally outperforms inter-frequency learning, aligning with the claim in FreDF \cite{fredf} that correlations among frequency points are minimal; 3) Furthermore, patching does not improve FreEformer, likely because patching frequencies creates localized receptive fields, thereby limiting access to global information.

\begin{table}[ht]
\begin{center}
{\fontsize{8}{9}\selectfont
\begin{tabular}{@{}cc|ccccc@{}}
\toprule
Attn.                    & Domain & Traffic        & PEMS03         & Weather        & Solar          & ILI            \\ \midrule
\multirow{2}{*}{Ours}    & Fre.   & {\color[HTML]{FF0000} \textbf{0.435}} & {\color[HTML]{FF0000} \textbf{0.102}} & {\color[HTML]{FF0000} \textbf{0.239}} & {\color[HTML]{FF0000} \textbf{0.217}} & {\color[HTML]{FF0000} \textbf{1.140}} \\
                         & Time   & 0.443          & 0.122          & 0.243          & 0.228          & 1.375          \\ \midrule
\multirow{2}{*}{Vanilla} & Fre.   & 0.451          & {\color[HTML]{FF0000} \textbf{0.113}} & {\color[HTML]{FF0000} \textbf{0.245}} & {\color[HTML]{FF0000} \textbf{0.220}} & {\color[HTML]{FF0000} \textbf{1.510}} \\
                         & Time   & {\color[HTML]{FF0000}\textbf{0.441}} & 0.146          & 0.248          & 0.226          & 2.140          \\ \bottomrule
\end{tabular}
}
\caption{Performance comparison of the frequency-domain and time-domain representation learning under two attention settings. Average MSEs are reported across four prediction lengths.}
\label{tab_fre_time}
\end{center}
\end{table}


\paragraph{Frequency-Domain vs. Temporal Representation} To construct the time-domain variant of FreEformer, we remove the DFT and IDFT steps, as well as the imaginary branch. As shown in Table \ref{tab_fre_time}, the frequency-domain representation achieves an average improvement of 8.4\% and 10.7\% in MSE compared to the time-domain version under the enhanced and vanilla attention settings, respectively. Additionally, we show in Section I.1 of the appendix that Fourier bases generally outperform Wavelet and polynomial bases for our model.

\begin{table}[ht]
\begin{center}
{\fontsize{8}{9}\selectfont
\setlength{\tabcolsep}{3.5pt}
\begin{tabular}{@{}c|ccccccc@{}}
\toprule
Head & ETTm1                                 & Weather                               & ECL                                   & Traffic                               & Solar                                 & NASDAQ                                & COVID-19                              \\ \midrule
Fre. & {\color[HTML]{FF0000} \textbf{0.379}}                                 & 0.245                                 & {\color[HTML]{FF0000} \textbf{0.160}}                                 & 0.441                                 & {\color[HTML]{FF0000} \textbf{0.216}} & {\color[HTML]{FF0000} \textbf{0.055}}                                 & 1.930                                 \\
Time & {\color[HTML]{FF0000} \textbf{0.379}} & {\color[HTML]{FF0000} \textbf{0.239}} & 0.162 & {\color[HTML]{FF0000} \textbf{0.435}} & 0.217 & {\color[HTML]{FF0000} \textbf{0.055}} & {\color[HTML]{FF0000} \textbf{1.892}} \\ \bottomrule
\end{tabular}
}
\caption{Performance comparison of frequency-domain and temporal prediction heads. Average MSEs are reported.}
\label{tab_head}
\end{center}
\end{table}

\paragraph{Prediction Head}
In our model, after performing frequency domain representation learning, we apply a temporal prediction head to generate the final predictions. In contrast, some frequency-based forecasters (e.g., FITS and Fredformer) directly predict the future frequency spectrum and transform it back to the time domain as the final step. In FreEformer, the frequency prediction head is formulated as:

\begin{equation}
\mathbf{\hat{y}}=\mathrm{DeNorm}  \left ( \mathrm{IDFT}  \left ( \mathrm{FlatLin}(\tilde{\mathrm{Re}})+j \cdot \mathrm{FlatLin}(\tilde{\mathrm{Im}}) \right ) \right ),
\label{eq_fre_head}
\end{equation}

\noindent where $\tilde{\mathrm{Re}}$ and $\tilde{\mathrm{Im}}$ are defined in Equation \eqref{eq7}. As shown in Table \ref{tab_head}, the temporal head slightly outperforms the frequency-domain head, highlighting the challenges of accurately forecasting the frequency spectrum. Additionally, Equation \eqref{eq_fre_head} incurs higher computational costs in the IDFT step when $\tau > T$, as in long-term forecasting scenarios.

\subsection{Enhanced Attention Analysis}

\begin{table}[ht]
\begin{center}
{\fontsize{7.5}{9}\selectfont
\begin{tabular}{@{}c@{\hspace{4pt}}c@{\hspace{4pt}}c@{\hspace{4pt}}c@{\hspace{4pt}}c@{\hspace{4pt}}c@{\hspace{4pt}}c@{\hspace{4pt}}c@{\hspace{4pt}}c@{}}
\toprule
Dataset & \begin{tabular}[c]{@{}c@{}}Ours\end{tabular} & 
\begin{tabular}[c]{@{}c@{}}Trans.\\ \shortcite{transformer} \end{tabular} & 
\begin{tabular}[c]{@{}c@{}}Flowfm.\\ \shortcite{flowformer} \end{tabular} & 
\begin{tabular}[c]{@{}c@{}}Flashfm.\\ \shortcite{flashattention} \end{tabular} & 
\begin{tabular}[c]{@{}c@{}}Flatten\\ \shortcite{flattentrans} \end{tabular} & 
\begin{tabular}[c]{@{}c@{}}Mamba\\ \shortcite{mamba} \end{tabular} & 
\begin{tabular}[c]{@{}c@{}}LASER\\ \shortcite{laser} \end{tabular} & 
\begin{tabular}[c]{@{}c@{}}Lin.Attn.\\ \shortcite{linear_softmax} \end{tabular} \\ \midrule
Traffic  & {\color[HTML]{FF0000} \textbf{0.435}}                            & {\color[HTML]{FF0000} 0.451}                             & {\color[HTML]{FF0000} 0.453}                              & {\color[HTML]{FF0000} 0.448}                                                  & {\color[HTML]{FF0000} 0.453}                                                    & {\color[HTML]{FF0000} {\ul 0.443}}                      & {\color[HTML]{FF0000} 0.451}                                                 & {\color[HTML]{FF0000} 0.452}                                                     \\
PEMS03   & {\color[HTML]{FF0000} \textbf{0.102}}                            & 0.113                                                    & 0.113                                                     & 0.114                                                      & 0.114                                                     & 0.115                                                   & {\color[HTML]{0000FF} {\ul 0.111}}                      & 0.115                                                       \\
Weather  & {\color[HTML]{FF0000} \textbf{0.239}}                            & 0.245                                                    & {\color[HTML]{0000FF} {\ul 0.242}}                        & 0.245                                                      & 0.248                                                     & 0.243                                                   & 0.244                                                   & 0.245                                                       \\
Solar    & {\color[HTML]{FF0000} \textbf{0.217}}                            & {\color[HTML]{FF0000} {\ul 0.220}}                       & {\color[HTML]{FF0000} 0.224}                              & {\color[HTML]{FF0000} 0.221}                               & {\color[HTML]{FF0000} 0.229}                              & {\color[HTML]{FF0000} 0.228}                            & {\color[HTML]{FF0000} 0.219}                            & {\color[HTML]{FF0000} 0.230}                                \\
ILI      & {\color[HTML]{FF0000} \textbf{1.140}}                            & 1.510                                                    & {\color[HTML]{FF0000} {\ul 1.288}}                        & 1.547                                                      & 1.842                                                     & 1.508                                                   & 1.596                                                   & {\color[HTML]{FF0000} 1.453}                                \\ \bottomrule
\end{tabular}
}
\caption{Comparison of the enhanced Transformer with state-of-the-art attention models and Mamba. Average MSEs are reported across four prediction lengths. Results outperforming state-of-the-art forecasters Leddam and CARD are highlighted in {\color[HTML]{FF0000} red}.}
\label{tab_attns}
\end{center}
\end{table}

We compare the enhanced Transformer with vanilla Transformer,  state-of-the-art Transformer variants and Mamba \cite{mamba} in Table \ref{tab_attns}. The enhanced Transformer consistently outperforms other models, verifying the effectiveness of the enhanced attention mechanism.

\begin{table}[ht]
\begin{center}
{\fontsize{7.5}{9}\selectfont
\begin{tabular}{@{}c@{\hspace{5pt}}|
c@{\hspace{5pt}}c|
c@{\hspace{5pt}}c|
c@{\hspace{5pt}}c|
c@{\hspace{5pt}}c@{}}
\toprule
\multirow{2}{*}{Dataset} & \multicolumn{2}{c|}{iTrans.} & \multicolumn{2}{c|}{PatchTST} & \multicolumn{2}{c|}{Leddam} & \multicolumn{2}{c}{Fredformer}  \\ \cmidrule(l){2-9} 
                         & Van.     & E.A.             & Van.      & E.A.             & Van.     & E.A.            & Van.           & E.A.          \\ \midrule
ETTm1                     & 0.407 & {\color[HTML]{FF0000} \textbf{0.389}} & 0.387                                 & {\color[HTML]{FF0000} \textbf{0.381}} & 0.386                                 & {\color[HTML]{FF0000} \textbf{0.384}} & {\color[HTML]{FF0000} \textbf{0.384}} & 0.385                                 \\
ECL                       & 0.178 & {\color[HTML]{FF0000} \textbf{0.165}} & 0.208                                 & {\color[HTML]{FF0000} \textbf{0.181}} & 0.169                                 & {\color[HTML]{FF0000} \textbf{0.167}} & 0.176                                 & {\color[HTML]{FF0000} \textbf{0.169}} \\
PEMS07                    & 0.101 & {\color[HTML]{FF0000} \textbf{0.086}} & 0.211                                 & {\color[HTML]{FF0000} \textbf{0.156}} & 0.084                                 & {\color[HTML]{FF0000} \textbf{0.080}} & 0.121                                 & {\color[HTML]{FF0000} \textbf{0.103}} \\
Solar                     & 0.233 & {\color[HTML]{FF0000} \textbf{0.226}} & 0.270                                 & {\color[HTML]{FF0000} \textbf{0.232}} & 0.230                                 & {\color[HTML]{FF0000} \textbf{0.228}} & 0.226                                 & {\color[HTML]{FF0000} \textbf{0.222}} \\
Weather                   & 0.258 & {\color[HTML]{FF0000} \textbf{0.249}} & 0.259                                 & {\color[HTML]{FF0000} \textbf{0.245}} & {\color[HTML]{FF0000} \textbf{0.242}} & {\color[HTML]{FF0000} \textbf{0.242}} & 0.246                                 & {\color[HTML]{FF0000} \textbf{0.242}} \\
METR-LA                   & 0.338 & {\color[HTML]{FF0000} \textbf{0.329}} & {\color[HTML]{FF0000} \textbf{0.335}} & {\color[HTML]{FF0000} \textbf{0.335}} & 0.327                                 & {\color[HTML]{FF0000} \textbf{0.321}} & 0.336                                 & {\color[HTML]{FF0000} \textbf{0.334}} \\ \bottomrule
\end{tabular}
}
\caption{Comparison of state-of-the-art models using vanilla  (Van.) and enhanced attention (E.A.). Only the attention mechanism is updated, with other components and the loss function kept unchanged. Average MSEs across four prediction lengths are reported.}
\label{tab_EA}
\end{center}
\end{table}

We further apply the enhanced attention mechanism to state-of-the-art forecasters, as shown in Table \ref{tab_EA}. This yields average MSE improvements of 5.9\% for iTransformer, 9.9\% for PatchTST, 1.4\% for Leddam (with updates only to the `cross-channel attention' module), and 3.8\% for FreEformer. These results demonstrate the versatility and effectiveness of the enhanced attention mechanism. Moreover, comparing Tables \ref{tab_long_term} and \ref{tab_EA}, FreEformer consistently outperforms these improved forecasters, underscoring its architectural advantages.

In the appendix, we further demonstrate FreEformer's performance superiority on more metrics (e.g., MASE, correlation coefficient).  Remarkably, FreEformer, trained from scratch, achieves superior or comparable performance to a pre-trained model fine-tuned on the same training data.

\section{Conclusion}
In this work, we present a simple yet effective multivariate time series forecasting model based on a frequency-domain enhanced Transformer. The enhanced attention mechanism is demonstrated to be effective both theoretically and empirically. It can consistently bring performance improvements for state-of-the-art Transformer-based forecasters. We hope that FreEformer will serve as a strong baseline for the time series forecasting community.

\nocite{non-stationary,rlinear,timer,moment,Chronos,google_timesfm,Moirai,swintrans}

\clearpage
\bibliographystyle{named}
\bibliography{ijcai25}

\clearpage

\appendix

\section{Theorem 1 and Its Proof}

\begin{theorem}[Frequency-domain linear projection and time-domain convolutions]
Given the time series $\mathbf{x}\in \mathbb{R}^N$ and its corresponding frequency spectrum $\mathcal{F}\in \mathbb{C}^N $. Let $\mathbf{W} \in \mathbb{C}^{N \times N}$ denote a weight matrix and $\mathbf{b} \in \mathbb{C}^N$ a bias vector. Under these definitions, the following DFT pair holds:

\begin{equation}
\begin{aligned}
\tilde{\mathcal{F}}=\mathbf{W}\mathcal{F}+\mathbf{b} \leftrightarrow \sum_{i=0}^{N-1}
{\Omega _i \circledast \mathcal{M}  _i(\mathbf{x})} + \mathrm{IDFT} (\mathbf{b}),
\end{aligned}
\label{eq1_appd}
\end{equation}

\noindent where

\begin{equation}
\begin{aligned}
&w_i=\left [ 
\mathrm{diag} (\mathbf{W},i) , \mathrm{diag} (\mathbf{W},i-N) \right ] \in \mathbb{C}^N, \\
&\Omega _i = \mathrm{IDFT}\left( w_i\right)\in \mathbb{C}^N, \\
&\mathcal{M}  _i(\mathbf{x})=\mathbf{x} \odot \left [ \mathrm{e}^{-j\frac{2\pi }{N}ik }  \right ]_{k=0,1,\cdots N-1} \in \mathbb{C}^N. 
\end{aligned}
\label{eq1_2}
\end{equation}

\noindent Here, $j$ is the imaginary unit, $\leftrightarrow$ denotes the DFT pair relationship, $\circledast$ represents the circular convolution, $\odot$ indicates the Hadamard (element-wise) product, and $\left [ \cdot ,\cdot \right ] $ represents the concatenation of two vectors. $\mathrm{diag}(\mathbf{W},i) \in \mathbb{C}^{N-\left | i \right | } $ extracts the $i$-th diagonal of $\mathbf{W}$, and $\mathrm{diag} (\mathbf{W},N)$ is defined as $\emptyset$. If $i=0$, $\mathrm{diag}(\mathbf{W},i)$ corresponds to the main diagonal; if $i>0$ (or $i<0$), it corresponds to a diagonal above (or below) the main diagonal. $\mathcal{M}  _i(\mathbf{x})$ represents the $i$-th modulated version of $\mathbf{x}$, with $\mathcal{M}  _0(\mathbf{x})$ being $\mathbf{x}$ itself.

\end{theorem}

\begin{proof}

To prove Theorem 1, we first introduce two supporting lemmas. 

1) A circular shift in the frequency domain corresponds to a multiplication by a complex exponential in the time domain. This can be expressed as follows: 

\begin{equation}
\mathrm{Roll} (\mathcal{F},-k ) \leftrightarrow \mathbf{x} \odot \left [ \mathrm{e}^{-j\frac{2\pi }{N}nk }  \right ]_{n=0,1,\cdots N-1},
\label{eq2_appd}
\end{equation}

\noindent where $\mathrm{Roll} (\mathcal{F},-k)$ denotes a circular shift of $\mathcal{F}$ to the left by $k$ elements, e.g., $\mathrm{Roll} (\mathcal{F},-1) = [f_2, \cdots, f_N, f_1]^T$. 

2) Multiplication in the frequency domain corresponds to convolution in the time domain. If $\mathcal{F} _1 \leftrightarrow \mathbf{x} _1$ and $\mathcal{F} _2 \leftrightarrow \mathbf{x} _2$, we have 

\begin{equation}
\mathcal{F} _1\odot \mathcal{F} _2 \leftrightarrow \mathbf{x} _1 \circledast \mathbf{x} _2.
\label{eq3_appd}
\end{equation}

The proof of Equations \eqref{eq2_appd} and \eqref{eq3_appd} can be found in the literature \cite{signal_and_system,dft}.

Now we analyze the matrix multiplication $\mathbf{W}\mathcal{F}$.  Let $\mathrm{DiagMat}(\mathbf{W}, (i, i-N))$ denote the matrix that retains only the specified two diagonals of $\mathbf{W}$ while setting all other elements to zero. Using this notation, the matrix $\mathbf{W}$ can be expressed~as:

\begin{equation}
\mathbf{W}=\mathrm{DiagMat} (\mathbf{W},0)+\sum_{i=1}^{N-1} \mathrm{DiagMat} \left ( \mathbf{W},(i,i-N) \right ). 
\label{eq4_0}
\end{equation}

Considering the definition of $w_i$ in Equation \eqref{eq1_2}, we~have

\begin{equation}
\begin{aligned}
&\mathrm{DiagMat} (\mathbf{W},0)\mathcal{F}  = \mathrm{diag} (\mathbf{W},0)\odot \mathcal{F}=w_0\odot \mathrm{Roll} (\mathcal{F},0),\\
&\mathrm{DiagMat} (\mathbf{W},(i,i-N))\mathcal{F}  = w_i\odot \mathrm{Roll} (\mathcal{F},-i ).
\end{aligned}
\label{eq4_1}
\end{equation}

By combining these results, $\tilde{\mathcal{F}} = \mathbf{W}\mathcal{F}$ can be reformulated as the sum of a series of element-wise vector products:

\begin{equation}
\mathbf{W}\mathcal{F}=\sum_{i=0}^{N-1} w_i\odot \mathrm{Roll} (\mathcal{F},-i).
\label{eq4_appd}
\end{equation}

\noindent  Given that $\Omega_i$ is defined as the IDFT of $w_i$, we derive from Equations \eqref{eq2} and \eqref{eq3} that:

\begin{equation}
\mathbf{W}\mathcal{F}\leftrightarrow \sum_{i=0}^{N-1}
{\Omega _i \circledast \mathcal{M}  _i(\mathbf{x})}.
\label{eq4_2}
\end{equation}

By further applying the linearity property of DFT, we obtain Equation \eqref{eq1_appd}, which concludes the proof.

\end{proof}

\textbf{Theorem 1} demonstrates that a linear transformation in the frequency domain is equivalent to a series of convolution operations applied to the time series and its modulated versions.

\section{Theorem 2 and Its Proof}

\begin{theorem}[Rank and condition number of matrix sums]

Let $\mathbf{A}$ and $\mathbf{B}$ be two matrices of the same size $N\times N$. Let $\mathrm{rank}({\cdot})$ and $\kappa({\cdot})$ represent the rank and condition number of a matrix, respectively. The condition number $\kappa({\cdot})$ is defined as $\kappa({\cdot}) = \frac{\sigma_1(\cdot)}{\sigma_N(\cdot)}$, where $\sigma_1(\cdot)$ and $\sigma_N(\cdot)$ are the largest and smallest singular values of the matrix, respectively. The following bounds hold for $\mathrm{rank}(\mathbf{A}+\mathbf{B})$ and $\kappa(\mathbf{A}+\mathbf{B})$:

\begin{equation}
\left | \mathrm{rank}(\mathbf{A})-\mathrm{rank}(\mathbf{B}) \right |  
\le \mathrm{rank}(\mathbf{A}+\mathbf{B})  
\le 
\mathrm{rank}(\mathbf{A})+\mathrm{rank}(\mathbf{B})
\label{eq5_appd}
\end{equation}

\noindent and 

\begin{equation}
\begin{aligned}
\frac{\max_{i=1,\cdots N} \left | \sigma_i(\mathbf{A})-\sigma_i(\mathbf{B}) \right |}{
  \min_{i+j=N+1}\left \{ \sigma_i(\mathbf{A})+\sigma_j(\mathbf{B}) \right \}
} \le \kappa (\mathbf{A}+\mathbf{B}) \\
\le \frac{\sigma_1(\mathbf{A})+\sigma_1(\mathbf{B})}{
\mathrm{max}\left \{ \sigma_N(\mathbf{A})-\sigma_1(\mathbf{B}),\sigma_N(\mathbf{B})-\sigma_1(\mathbf{A}),0 \right \}  
}. \\
\end{aligned}
\label{eq6_appd}
\end{equation}

\end{theorem}

\begin{proof}
 We first prove the upper and lower bounds for $\mathrm{rank}(\mathbf{A}+\mathbf{B})$. 

\textbf{Upper bound of $\mathrm{rank}(\mathbf{A}+\mathbf{B})$}. Let $\mathrm{Col} (\mathbf{A} )$ and $\mathrm{Col} (\mathbf{B} )$ denote the column spaces of $\mathbf{A}$ and $\mathbf{B}$, respectively. Naturally, the column space of $\mathbf{A}+\mathbf{B}$ satisfies

\begin{equation}
\mathrm{Col} (\mathbf{A}+\mathbf{B} ) \subseteq \mathrm{Col} (\mathbf{A}) \cup \mathrm{Col} (\mathbf{B} )
\label{eq7_appd}
\end{equation}

Therefore, we have 

\begin{equation}
\begin{aligned}
\mathrm{rank}(\mathbf{A}+\mathbf{B}) & = \mathrm{dim} (\mathrm{Col} (\mathbf{A}+\mathbf{B} )) \\
&\le \mathrm{dim} (\mathrm{Col} (\mathbf{A}) \cup \mathrm{Col} (\mathbf{B} )) \\
&\le \mathrm{dim} (\mathrm{Col} (\mathbf{A})) + \mathrm{dim} (\mathrm{Col} (\mathbf{B}))\\
&=\mathrm{rank}(\mathbf{A})+\mathrm{rank}(\mathbf{B}).
\end{aligned}
\label{eq8_appd}
\end{equation}

\textbf{Lower bound of $\mathrm{rank}(\mathbf{A}+\mathbf{B})$}. Let  $\mathrm{Null} (\mathbf{A} )$ and $\mathrm{Null} (\mathbf{B} )$ denote the null spaces of $\mathbf{A}$ and $\mathbf{B}$, respectively. For any vector $\mathbf{x} \in \mathrm{Col} (\mathbf{A} )\cap \mathrm{Null} (\mathbf{B} )$, we have 

\begin{equation}
(\mathbf{A} + \mathbf{B})\mathbf{x} =\mathbf{A}\mathbf{x}+\mathbf{B}\mathbf{x}=\mathbf{A}\mathbf{x}\ne \mathbf{0},
\label{eq9_appd}
\end{equation}

\noindent i.e., $\mathbf{x} \in \mathrm{Col} (\mathbf{A}+\mathbf{B} )$. Therefore, it holds that

\begin{equation}
\mathrm{Col} (\mathbf{A} )\cap \mathrm{Null} (\mathbf{B} )\subseteq \mathrm{Col} (\mathbf{A}+\mathbf{B} ), 
\label{eq10_appd}
\end{equation}

\noindent which implies that

\begin{equation}
\begin{aligned}
\mathrm{rank}(\mathbf{A}+\mathbf{B}) & = \mathrm{dim} (\mathrm{Col} (\mathbf{A}+\mathbf{B} )) \\
&\ge \mathrm{dim} (\mathrm{Col} (\mathbf{A}) \cap \mathrm{Null} (\mathbf{B} )) \\
&\ge \mathrm{dim} (\mathrm{Col} (\mathbf{A}))+\mathrm{dim} (\mathrm{Null} (\mathbf{B}))-N \\
&=\mathrm{rank}(\mathbf{A})-\mathrm{rank}(\mathbf{B}). 
\end{aligned}
\label{eq11_appd}
\end{equation}

Given the equivalence of $\mathbf{A}$ and $\mathbf{B}$, we also obtain

\begin{equation}
\mathrm{rank}(\mathbf{A}+\mathbf{B}) \ge \mathrm{rank}(\mathbf{B})-\mathrm{rank}(\mathbf{A}). 
\label{eq12_appd}
\end{equation}

Combining these results, we have 

\begin{equation}
\begin{aligned}
\mathrm{rank}(\mathbf{A}+\mathbf{B}) &\ge \mathrm{max}  \left \{ \mathrm{rank}(\mathbf{B})-\mathrm{rank}(\mathbf{A}), 
\mathrm{rank}(\mathbf{A})-\mathrm{rank}(\mathbf{B}) \right \} \\
&=\left | \mathrm{rank}(\mathbf{A})-\mathrm{rank}(\mathbf{B}) \right |.
\end{aligned}
\label{eq13_appd}
\end{equation}

Thus, we complete the proof of Equation \eqref{eq5_appd}.

To analyze the upper and lower bounds of $\kappa(\mathbf{A+B})$, we first recall Weyl's inequality \cite{matrix}, which provides an upper bound on the singular values of the sum of two matrices. Specifically, for the singular values of $\mathbf{A}$, $\mathbf{B}$, and $\mathbf{A+B}$ arranged in decreasing order, the inequality states~that:

\begin{equation}
\sigma _{i+j-1}(\mathbf{A} +\mathbf{B} ) \le \sigma _{i}(\mathbf{A} ) + \sigma _{j}(\mathbf{B} ),
\label{eq14_appd}
\end{equation}

\noindent for $1\le i,j\le N$ and $i+j\le N+1$, where $\sigma _1(\cdot )\ge \sigma _2(\cdot )\ge \cdots  \ge \sigma _N(\cdot )$  denotes the singular values in non-increasing order. 
From Equation \eqref{eq14_appd} (Weyl's inequality), we can deduce:

\begin{equation}
\sigma _{i}(\mathbf{A} ) \le \sigma _{i}(\mathbf{A} + \mathbf{B}) + \sigma _{j}(\mathbf{-B} ),
\label{eq15_appd}
\end{equation}
which  implies that:

\begin{equation}
\begin{aligned}
\sigma _{i}(\mathbf{A}  +\mathbf{B} ) \ge & \sigma _{i+j-1}(\mathbf{A}) - \sigma _{j}(\mathbf{-B} )\\
=&\sigma _{i+j-1}(\mathbf{A}) - \sigma _{j}(\mathbf{B} ).
\end{aligned}
\label{eq16_appd}
\end{equation}
Here, we use the fact that $\mathbf{B}$ and $\mathbf{-B}$ share identical singular values. Using Equations \eqref{eq14_appd} and \eqref{eq16_appd}, we have

\begin{equation}
\begin{aligned}
\max_{i  = 1,\cdots N} \left | \sigma_i(\mathbf{A})-\sigma_i(\mathbf{B}) \right | \le & \sigma_1(\mathbf{A}+\mathbf{B}) \\
\le & \sigma_1(\mathbf{A}) + \sigma_1(\mathbf{B}),
\end{aligned}
\label{eq17_appd}
\end{equation}
and
\begin{equation}
\begin{aligned}
&\mathrm{max}\left \{ \sigma_N(\mathbf{A})-\sigma_1(\mathbf{B}),\sigma_N(\mathbf{B})-\sigma_1(\mathbf{A}),0 \right \}   \\
&\le  \sigma_N(\mathbf{A}+\mathbf{B}) \\
&\le  \min_{i+j=N+1}\left \{ \sigma_i(\mathbf{A})+\sigma_j(\mathbf{B}) \right \} 
\end{aligned}
\label{eq18_appd}
\end{equation}

In Equation \eqref{eq18_appd}, we also use the property that $\sigma_N(\mathbf{A}+\mathbf{B})\ge 0$. By combining Equations \eqref{eq17_appd} and \eqref{eq18_appd}, we can obtain Equation \eqref{eq6_appd}, thus completing the proof.

\end{proof}

\textbf{Theorem 2} provides the lower and upper bounds for the rank and condition number of the sum of two matrices. From Equation \eqref{eq5_appd}, we observe that the rank of the summed matrix tends to approach that of the higher-rank matrix, particularly when one matrix is of low rank and the other is of high rank. From Equation \eqref{eq6_appd}, we observe that if one matrix is well-conditioned with a higher $\sigma_N$, while the other has a smaller $\sigma_1$, such that $\sigma_N(\mathbf{A}) \gg \sigma_1(\mathbf{B})$, then $\kappa(\mathbf{A}+\mathbf{B})$ is approximately upper-bounded by $\kappa(\mathbf{A})$. This indicates that the condition number of $\mathbf{A} + \mathbf{B}$ is dominated by the better-conditioned matrix~$\mathbf{A}$.

\section{Gradient Analysis of Enhanced Attention}

In this section, we analyze the gradient of the proposed enhanced attention, and compare it with the vanilla attention mechanism.

\subsection{Gradient (Jacobian) Matrix of Softmax}

In the vanilla self-attention mechanism, the softmax is applied row-wise on $\frac{\mathbf{Q}\mathbf{K}^T}{\sqrt{D}} \in \mathbb{R}^{N \times N}$. Without loss of generality, we consider a single row $\mathbf{a} \in \mathbb{R}^N$ in $\frac{\mathbf{Q}\mathbf{K}^T}{\sqrt{D}}$ to derive the gradient matrix. Let $\tilde{\mathbf{a}} = \mathrm{Softmax}(\mathbf{a})$. Then the $i$-th element of $\tilde{\mathbf{a}}$ can be written as

\begin{equation}
\tilde{\mathbf{a}} _i = \frac{e^{\mathbf{a}_i} }{\sum_{k=0}^{N-1}e^{\mathbf{a}_k} }. 
\label{eq19}
\end{equation}

\noindent From this, we compute the partial derivative of $\tilde{\mathbf{a}}_i$ with respect to $\mathbf{a}_i$:

\begin{equation}
\begin{aligned}
\frac{\partial \tilde{\mathbf{a}} _i}{\partial \mathbf{a} _i} & = \frac{e^{\mathbf{a}_i} }{\sum_{k  = 0}^{N-1}e^{\mathbf{a}_k} }-\frac{e^{\mathbf{a}_i}\cdot e^{\mathbf{a}_i}}{(\sum_{k  = 0}^{N-1}e^{\mathbf{a}_k} )^2}  \\
  & = \tilde{\mathbf{a}} _i - \tilde{\mathbf{a}} _i^2.
\end{aligned} 
\label{eq20}
\end{equation}

\noindent For $j \neq i$, the partial derivative is:

\begin{equation}
\begin{aligned}
\frac{\partial \tilde{\mathbf{a}} _i}{\partial \mathbf{a} _j} & = -\frac{e^{\mathbf{a}_i}\cdot e^{\mathbf{a}_j}}{(\sum_{k = 0}^{N-1}e^{\mathbf{a}_k} )^2}  \\
  & = -\tilde{\mathbf{a}} _i\tilde{\mathbf{a}} _j
\end{aligned} 
\label{eq21}
\end{equation}

\noindent Combining Equations \eqref{eq20} and \eqref{eq21}, the Jacobian matrix can be expressed as:

\begin{equation}
\frac{\partial \tilde{\mathbf{a}} }{\partial \mathbf{a} }  = \mathrm{Diag} (\tilde{\mathbf{a}}) - \tilde{\mathbf{a}}\tilde{\mathbf{a}}^T,
\label{eq22}
\end{equation}
where $\mathrm{Diag} (\tilde{\mathbf{a}})$ is a diagonal matrix with $\tilde{\mathbf{a}}$ as its diagonal, and all other elements set to zero.

Equation \eqref{eq22} reveals that the off-diagonal elements of the Jacobian matrix are the pairwise products $\tilde{\mathbf{a}}_i \tilde{\mathbf{a}}_j$. Since $\tilde{\mathbf{a}}$ often contains small probabilities \cite{laser}, these products become even smaller, resulting in a Jacobian matrix with near-zero entries. This behavior significantly diminishes the gradient flow during back-propagation, thereby hindering training efficiency and model optimization.

\subsection{Jacobian Matrix of Enhanced Attention} \label{C.2}

In this paper, we propose a simple yet effective attention mechanism, which can be expressed as:

\begin{equation}
\mathbf{c} =\mathrm{Norm} \left ( \mathrm{Softmax}(\mathbf{a}) + \mathbf{b}  \right ) ,
\label{eq23}
\end{equation}

\noindent where $\mathrm{Norm} (\cdot) = \frac{\cdot}{\left \| \cdot  \right \| _1}$ denotes the L1 normalization operator, and $\mathbf{b}_i > 0$ for $i = 0, \cdots, N-1$. Define $\tilde{\mathbf{b}} = \mathrm{Softmax}(\mathbf{a}) + \mathbf{b}$. Since $\mathrm{Softmax}(\mathbf{a})$ is non-negative and $\mathbf{b}_i > 0$, it follows that $\tilde{\mathbf{b}}_i > 0$ for all $i$. Consequently, the entries of $\mathbf{c}$ can be written as:

\begin{equation}
\mathbf{c}_i = \frac{\tilde{\mathbf{b}}_i}{\sum_{k=0}^{N-1}\tilde{\mathbf{b}}_k},
\label{eq24}
\end{equation}

\noindent Then, the partial derivative of $\mathbf{c}_i$ with respect to $\tilde{\mathbf{b}}_i$ is:

\begin{equation}
\begin{aligned}
\frac{\partial \mathbf{c}_i}{\partial \tilde{\mathbf{b}}_i} &= \frac{1}{\sum_{k=0}^{N-1}\tilde{\mathbf{b}}_k} -\frac{\tilde{\mathbf{b}}_i}{(\sum_{k=0}^{N-1}\tilde{\mathbf{b}}_k)^2} \\
 &=\frac{1}{\left \| \tilde{\mathbf{b}} \right \|_1^2 } \left ( \left \| \tilde{\mathbf{b}} \right \|_1- \tilde{\mathbf{b}}_i \right ) 
\end{aligned}
\label{eq25}
\end{equation}

\noindent For $j \neq i$, the partial derivative is:

\begin{equation}
\begin{aligned}
\frac{\partial \mathbf{c}_i}{\partial \tilde{\mathbf{b}}_j}   = 
 -\frac{\tilde{\mathbf{b}}_i}{(\sum_{k  = 0}^{N-1}\tilde{\mathbf{b}}_k)^2} 
 = -\frac{1}{\left \| \tilde{\mathbf{b}} \right \|_1^2 }  \tilde{\mathbf{b}}_i
\end{aligned}
\label{eq26}
\end{equation}

Combination of Equations \eqref{eq25} and \eqref{eq26} yields the following expression for $\frac{\partial \mathbf{c}}{\partial \tilde{\mathbf{b}}}$:

\begin{equation}
\begin{aligned}
\frac{\partial \mathbf{c}}{\partial \tilde{\mathbf{b}}} 
 = \frac{1}{\left \| \tilde{\mathbf{b}} \right \|_1^2 } 
\left ( \left \| \tilde{\mathbf{b}} \right \|_1\cdot \mathbf{I}- \tilde{\mathbf{b}} \mathbf{1}^T  \right ), 
\end{aligned}
\label{eq27}
\end{equation}

\noindent where $\mathbf{I}\in \mathbb{R}^{N\times N}$ is the unit matrix; $\mathbf{1}=\left [ 1,\cdots ,1 \right ]^T \in \mathbb{R}^N$. Let $\tilde{\mathbf{a}}$ denote $\mathrm{Softmax}(\mathbf{a})$. Using the chain rule, we obtain:

\begin{equation}
\begin{aligned}
\frac{\partial \mathbf{c}}{\partial \mathbf{a}} &=\frac{\partial \mathbf{c}}{\partial \tilde{\mathbf{b}} }\cdot \frac{\partial \tilde{\mathbf{b}}}{\partial \mathbf{a} } \\
 &= \frac{1}{\left \| \tilde{\mathbf{b}} \right \|_1^2 } 
\left ( \left \| \tilde{\mathbf{b}} \right \|_1\cdot \mathbf{I}- \tilde{\mathbf{b}}  \mathbf{1}^T  \right )
\left ( \mathrm{Diag} (\tilde{\mathbf{a}}) - \tilde{\mathbf{a}}\tilde{\mathbf{a}}^T  \right ) \\
&=\frac{1}{\left \| \tilde{\mathbf{b}} \right \|_1^2 } \left ( 
\left \| \tilde{\mathbf{b}} \right \|_1 \mathrm{Diag} (\tilde{\mathbf{a}})
-\left \| \tilde{\mathbf{b}} \right \|_1 \tilde{\mathbf{a}}\tilde{\mathbf{a}}^T
-\tilde{\mathbf{b}}\tilde{\mathbf{a}}^T+\tilde{\mathbf{b}}\tilde{\mathbf{a}}^T 
 \right ) \\
&=\frac{1}{\left \| \tilde{\mathbf{b}} \right \|_1 } 
\left ( \mathrm{Diag} (\tilde{\mathbf{a}}) - \tilde{\mathbf{a}}\tilde{\mathbf{a}}^T  \right ).
\end{aligned}
\label{eq28}
\end{equation}

\noindent Here, we utilize the facts that $\mathbf{1}^T\mathrm{Diag} (\tilde{\mathbf{a}})=\tilde{\mathbf{a}}^T$ and $\mathbf{1}^T\tilde{\mathbf{a}}=1$.

Based on Equation \eqref{eq27}, we naturally have

\begin{equation}
\begin{aligned}
\frac{\partial \mathbf{c}}{\partial \mathbf{b}}  =\frac{1}{\left \| \tilde{\mathbf{b}} \right \|_1^2 } 
\left ( \left \| \tilde{\mathbf{b}} \right \|_1 \cdot \mathbf{I} - \tilde{\mathbf{b}}\mathbf{1}^T   \right ) 
\end{aligned}
\label{eq_jacobian_c_b}
\end{equation}

Comparing Equations \eqref{eq22} and \eqref{eq28}, we observe that the Jacobian matrix of the enhanced attention shares the same structure with the vanilla counterpart, except for an extra learnable scaling item $1/\left \| \tilde{\mathbf{b}} \right \|_1 $. This additional flexibility improves the adaptability of the gradient back-propagation process, potentially leading to more efficient optimization.

\subsection{Variants of Enhanced Attention}

\begin{figure}[t]
   \centering
   \includegraphics[width=1.0\linewidth]{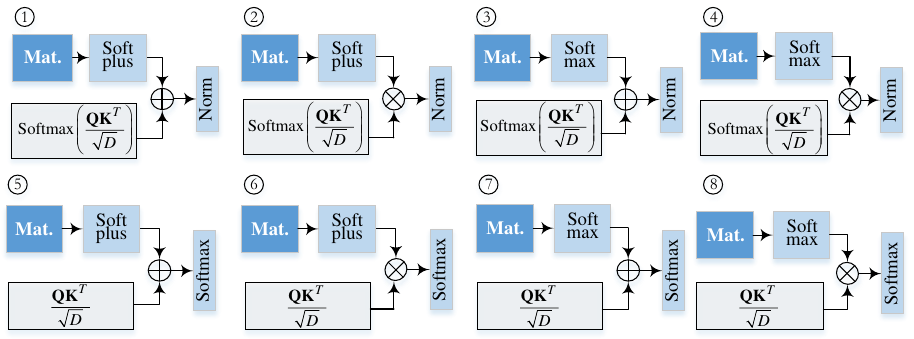}
   \caption{Illustration of enhanced attention variants. By ensuring $\mathbf{B}$ is positive using either Softplus or Softmax, combining it with $\mathbf{A}$ through addition or multiplication, and applying Softmax or Norm for final normalization, we derive seven variants. The first subfigure represents the version adopted by default, while the others are denoted as $\mathrm{Var}\left \{ 1, \cdots, 7 \right \}$. Detailed formulas are provided in Table \ref{tab_var_formula}.}
   \label{fig4}
\end{figure}

We present the variants of enhanced attention in Figure \ref{fig4} and Table \ref{tab_var_formula}. For simplicity, our analysis omits the effect of the query matrix $\mathbf{V}$, as it is included in all variants. Notably, the Swin Transformer \cite{swintrans} introduces relative position bias to the vanilla attention and employs $  \mathrm{Softmax} \left ( \mathbf{Q}\mathbf{K}^T / \sqrt{D} + \mathbf{B} \right )   \mathbf{V}$, where $\mathbf{B}$ represents the relative position bias among tokens, with correlations existing between its entries. In contrast, Var4 and Var6 in Table \ref{tab_var_formula}, although they exhibit a structure similar to that of the Swin Transformer, allow greater flexibility in $\mathbf{B}$, where each entry is independent. Next, we analyze the Jacobian matrix of the variants without distinguishing between the choices of Softplus and Softmax.

\begin{table}[t]
\begin{center}

\begin{tabular}{@{}cc@{}}
\toprule
Variants &  $\mathrm{EnhAttn}\left ( \mathbf{Q} ,\mathbf{K} ,\mathbf{V}  \right )$  \\ \midrule
Ours   &    $\mathrm{Norm} \left ( \mathrm{Softmax} \left ( \frac{\mathbf{Q}\mathbf{K}^T}{\sqrt{D} }  \right ) + \mathrm{Softplus} (\mathbf{B}) \right )  \mathbf{V}$    \\
Var1   &    $\mathrm{Norm} \left ( \mathrm{Softmax} \left ( \frac{\mathbf{Q}\mathbf{K}^T}{\sqrt{D} }  \right ) \odot  \mathrm{Softplus} (\mathbf{B}) \right )  \mathbf{V} $    \\ 
Var2   &    $\mathrm{Norm} \left ( \mathrm{Softmax} \left ( \frac{\mathbf{Q}\mathbf{K}^T}{\sqrt{D} }  \right ) +  \mathrm{Softmax} (\mathbf{B}) \right )  \mathbf{V} $                         \\
Var3   &       $\mathrm{Norm} \left ( \mathrm{Softmax} \left ( \frac{\mathbf{Q}\mathbf{K}^T}{\sqrt{D} }  \right ) \odot   \mathrm{Softmax} (\mathbf{B}) \right )  \mathbf{V} $                      \\
Var4   &        $\mathrm{Softmax} \left (  \frac{\mathbf{Q}\mathbf{K}^T}{\sqrt{D} }   +   \mathrm{Softplus} (\mathbf{B}) \right )  \mathbf{V} $                     \\
Var5   &         $\mathrm{Softmax} \left (  \frac{\mathbf{Q}\mathbf{K}^T}{\sqrt{D} }   \odot    \mathrm{Softplus} (\mathbf{B}) \right )  \mathbf{V} $                     \\
Var6   &         $\mathrm{Softmax} \left (  \frac{\mathbf{Q}\mathbf{K}^T}{\sqrt{D} }   +    \mathrm{Softmax} (\mathbf{B}) \right )  \mathbf{V} $                     \\
Var7   &          $\mathrm{Softmax} \left (  \frac{\mathbf{Q}\mathbf{K}^T}{\sqrt{D} }   \odot    \mathrm{Softmax} (\mathbf{B}) \right )  \mathbf{V} $                    \\ \bottomrule
\end{tabular}

\caption{Equations for different variants of the enhanced attention mechanism. These variants correspond to Figure 4 in the main paper. }
\label{tab_var_formula}
\end{center}
\end{table}

\subsubsection{Var1 and Var3}
For Var1 and Var3 in Table \ref{tab_var_formula}, the transformation can be simplified to the vector form 

\begin{equation}
\begin{aligned}
\mathbf{c} = \mathrm{Norm} \left( \mathrm{Softmax}(\mathbf{a}) \odot \mathbf{b} \right) \in \mathbb{R}^N,
\end{aligned}
\label{eq_var2}
\end{equation}

\noindent where $\mathbf{b}_i > 0$ for $i = 0, \cdots, N-1$. For simplicity, let $\tilde{\mathbf{a}} \triangleq \mathrm{Softmax}(\mathbf{a})$ and $\tilde{\mathbf{b}} \triangleq \tilde{\mathbf{a}} \odot \mathbf{b}$. The Jacobian matrix of $\mathbf{c}$ with respect to $\mathbf{a}$ can then be expressed as:

\begin{equation}
\begin{aligned}
&\frac{\partial \mathbf{c} }{\partial \mathbf{a} }  = 
\frac{\partial \mathbf{c} }{\partial \tilde{\mathbf{b}} }
\frac{\partial \tilde{\mathbf{b}} }{\partial \tilde{\mathbf{a}} }
\frac{\partial \tilde{\mathbf{a}} }{\partial \mathbf{a} } \\
&=
\frac{1}{\left \| \tilde{\mathbf{b}} \right \|_1^2 } 
\left ( \left \| \tilde{\mathbf{b}} \right \|_1\cdot \mathbf{I}- \tilde{\mathbf{b}}  \mathbf{1}^T  \right )
 \mathrm{Diag}(\mathbf{b})
\left ( \mathrm{Diag} (\tilde{\mathbf{a}}) - \tilde{\mathbf{a}}\tilde{\mathbf{a}}^T  \right )\\
&=\frac{\mathrm{Diag}(\mathbf{b})\left ( \mathrm{Diag} (\tilde{\mathbf{a}}) - \tilde{\mathbf{a}}\tilde{\mathbf{a}}^T  \right )}{\left \| \tilde{\mathbf{b}} \right \|_1 }
+\frac{(\mathbf{b}^T\tilde{\mathbf{a}} ) \tilde{\mathbf{b}}\tilde{\mathbf{a}}^T - \tilde{\mathbf{b}}\tilde{\mathbf{b}}^T }
{\left \| \tilde{\mathbf{b}} \right \|_1^2 }.
\end{aligned}
\label{eq31}
\end{equation}

\noindent Here, the equations $ \mathbf{1}^T\mathrm{Diag}(\mathbf{b}) =\mathbf{b}^T$ and $ \mathbf{1}^T\mathrm{Diag}(\mathbf{b}) \mathrm{Diag}(\tilde{\mathbf{a}})=\mathbf{b}^T\mathrm{Diag}(\tilde{\mathbf{a}})=\tilde{\mathbf{b}} ^T$ are utilized.

Similarly, the Jacobian matrix of $\mathbf{c}$ with respect to $\mathbf{b}$ is:

\begin{equation}
\begin{aligned}
\frac{\partial \mathbf{c} }{\partial \mathbf{b} }  &= 
\frac{\partial \mathbf{c} }{\partial \tilde{\mathbf{b}} }
\frac{\partial \tilde{\mathbf{b}} }{\partial \mathbf{b} }\\
&=
\frac{1}{\left \| \tilde{\mathbf{b}} \right \|_1^2 } 
\left ( \left \| \tilde{\mathbf{b}} \right \|_1\cdot \mathbf{I}- \tilde{\mathbf{b}}  \mathbf{1}^T  \right )
 \mathrm{Diag}(\tilde{\mathbf{a}})\\
&=\frac{1}{\left \| \tilde{\mathbf{b}} \right \|_1^2 } 
\left ( \left \| \tilde{\mathbf{b}} \right \|_1 \cdot \mathrm{Diag}(\tilde{\mathbf{a}}) - \tilde{\mathbf{b}} \tilde{\mathbf{a}}^T  \right ).
\end{aligned}
\label{eq_jacob_var2}
\end{equation}

Despite the complexity of Equation \eqref{eq31}, the resulting Jacobian matrix introduces additional weighting and bias factors compared to the vanilla counterpart associated with $\mathrm{Softmax}(\mathbf{a})$, providing greater flexibility to dynamically adjust the gradients.

\subsubsection{Var4 and Var6}
For Var4 and Var6 in Table \ref{tab_var_formula}, the transformation can be simplified to the vector form: 

\begin{equation}
\begin{aligned}
\mathbf{c} = \mathrm{Softmax}(\mathbf{a}+\mathbf{b})\in \mathbb{R}^N ,
\end{aligned}
\label{eq_var5}
\end{equation}

\noindent where $\mathbf{b}_i > 0$ for $i=0,\cdots ,N-1$. According to Equation \eqref{eq22}, the Jacobian matrices of $\mathbf{c}$ with respect to $\mathbf{a}$ and $\mathbf{ab}$ can be expressed as:

\begin{equation}
\frac{\partial \mathbf{c} }{\partial \mathbf{a} }  = \frac{\partial \mathbf{c} }{\partial \mathbf{b} } =  \mathrm{Diag} (\mathbf{c}) - \mathbf{c}\mathbf{c}^T.
\label{eq29}
\end{equation}

\subsubsection{Var5 and Var7}
For Var5 and Var7 in Table \ref{tab_var_formula}, the transformation can be simplified to the vector form 

\begin{equation}
\mathbf{c} = \mathrm{Softmax}(\mathbf{a} \odot \mathbf{b}) \in \mathbb{R}^N,
\label{eq_var6}
\end{equation}

\noindent where $\mathbf{b}_i > 0$ for $i = 0, \cdots, N-1$. For simplicity, let $\tilde{\mathbf{b}} \triangleq \mathbf{a} \odot \mathbf{b}$. The Jacobian matrix of $\mathbf{c}$ with respect to $\mathbf{a}$ is then derived as:
 
\begin{equation}
\begin{aligned}
\frac{\partial \mathbf{c} }{\partial \mathbf{a} }  &= 
\frac{\partial \mathbf{c} }{\partial \tilde{\mathbf{b}} }\frac{\partial \tilde{\mathbf{b}} }{\partial \mathbf{a} } \\
&=\left ( \mathrm{Diag} (\mathbf{c}) - \mathbf{c} \mathbf{c} ^T \right ) \mathrm{Diag}(\mathbf{b}), \\
\frac{\partial \mathbf{c} }{\partial \mathbf{b} }  &= 
\frac{\partial \mathbf{c} }{\partial \tilde{\mathbf{b}} }\frac{\partial \tilde{\mathbf{b}} }{\partial \mathbf{b} } \\
&=\left ( \mathrm{Diag} (\mathbf{c}) - \mathbf{c} \mathbf{c} ^T \right ) \mathrm{Diag}(\mathbf{a}),
\end{aligned}
\label{eq30}
\end{equation}

\noindent From Equation \eqref{eq30}, we observe that each column of the Jacobian matrix is scaled by the  $\mathbf{b}$, introducing additional flexibility in adjusting the gradients. 

Performance comparison of these variants are presented in Tables \ref{tab_attn_var_appd}, \ref{tab_attn_var_pems_appd}, and \ref{tab_attn_var_short_appd}.

\section{Dataset Description}

In this work, we evaluate the performance of FreEformer on the following real-world datasets:

\begin{itemize}

\item \textbf{ETT} \cite{informer2021} records seven variables related to electricity transformers from July 2016 to July 2018. It is divided into four datasets: ETTh1 and ETTh2, with hourly readings, and ETTm1 and ETTm2, with readings taken at 15-minute intervals.

\item \textbf{Weather} \footnote{https://www.bgc-jena.mpg.de/wetter/} contains 21 meteorological factors (e.g., air temperature, humidity) recorded every 10 minutes at the Weather Station of the Max Planck Biogeochemistry Institute in 2020.

\item \textbf{ECL} \cite{autoformer} documents the hourly electricity consumption of 321 clients from 2012 to 2014.

\item \textbf{Traffic} \footnote{http://pems.dot.ca.gov} includes hourly road occupancy ratios collected from 862 sensors in the San Francisco Bay Area between January 2015 and December 2016.

\item \textbf{Exchange} \cite{autoformer} features daily exchange rates for eight countries from 1990 to 2016.

\item \textbf{Solar-Energy} \cite{rnn} captures solar power output from 137 photovoltaic plants in 2006, with data sampled every 10 minutes.

\item \textbf{PEMS} \cite{scinet} provides public traffic data from California, collected every 5 minutes. We use the same four subsets (PEMS03, PEMS04, PEMS07, PEMS08) as in SCINet \cite{scinet}.

\item \textbf{ILI} \footnote{https://gis.cdc.gov/grasp/fluview/fluportaldashboard.html} comprises weekly records of influenza-like illness (ILI) patient counts from the Centers for Disease Control and Prevention in the United States from 2002 to 2021.

\item \textbf{COVID-19} \cite{chen2022tamp} includes daily COVID-19 hospitalization data in California (CA) from February to December 2020, provided by Johns Hopkins University.


\item \textbf{METR-LA} \footnote{https://github.com/liyaguang/DCRNN} includes traffic network data for Los Angeles, collected from March to June 2012. It contains 207 sensors, with data sampled every 5 minutes.

\item \textbf{NASDAQ} \footnote{https://www.kaggle.com/datasets/sai14karthik/nasdq-dataset} comprises daily NASDAQ index data, combined with key economic indicators (e.g., interest rates, exchange rates, gold prices) from 2010 to 2024.

\item \textbf{Wiki} \footnote{https://www.kaggle.com/datasets/sandeshbhat/wikipedia-web-traffic-201819} contains daily page views for 60,000 Wikipedia articles in eight different languages over two years (2018–2019). We select the first 99 articles as our experimental dataset.

\end{itemize}

The statistics of these datasets are summarized in Table \ref{tab_stat}. We split all datasets into training, validation, and test sets in chronological order, using the same splitting ratios as in iTransformer \cite{itransformer}. The splitting ratio is set to 6:2:2 for the ETT and PEMS datasets, as well as for the COVID-19 dataset under the `input-36-predict-{24,36,48,60}' setting. For all other cases (including other datasets and different settings for COVID-19), we use a 7:1:2 split. The reason for this specific setting in the COVID-19 dataset is that its total length is insufficient to accommodate a 7:1:2 split under the `input-36-predict-{24,36,48,60}' setting.

\begin{table}[t]
\begin{center}
\begin{tabular}{@{}cccc@{}}
\toprule
Datasets     & \#Variants & Total Len. & Sampling Fre. \\ \midrule
ETT\{h1, h2\} & 7          & 17420      & 1 hour         \\
ETT\{m1, m2\} & 7          & 69680      & 15 min         \\
Weather      & 21         & 52696      & 10 min         \\
ECL          & 321        & 26304      & 1 hour         \\
Traffic      & 862        & 17544      & 1 hour         \\
Exchange     & 8          & 7588       & 1 day          \\
Solar-Energy & 137        & 52560      & 10 min         \\
PEMS03       & 358        & 26209      & 5 min          \\
PEMS04       & 307        & 16992      & 5 min          \\
PEMS07       & 883        & 28224      & 5 min          \\
PEMS08       & 170        & 17856      & 5 min          \\
ILI          & 7          & 966        & 1 week         \\
COVID-19     & 55         & 335        & 1 day          \\
METR-LA      & 207        & 34272      & 1 day          \\
NASDAQ       & 12         & 3914       & 1 day          \\
Wiki         & 99         & 730        & 1 day          \\ \bottomrule
\end{tabular}
\caption{Statistics of the real-world datasets utilized in our study}
\label{tab_stat}
\end{center}
\end{table}

\section{Implementation Details}

All experiments are implemented in PyTorch \cite{pytorch} and conducted on a single NVIDIA 4090 GPU. We optimize the model using the ADAM optimizer \cite{adam_opt} with an initial learning rate selected from $\left \{ 10^{-4}, 5\times 10^{-4} \right \}$. The model dimension $D$ is chosen from $\left \{ 128, 256, 512 \right \}$, while the embedding dimension $d$ is fixed at 16. The batch size is determined based on the dataset size and selected from $\left \{ 4, 8, 16, 32 \right \}$. The training is conducted for a maximum of 50 epochs with an early stopping mechanism, which terminates training if the validation performance does not improve for 10 consecutive epochs. We adopt the weighted L1 loss function following CARD \cite{card}. For baseline models, we use the reported values from the original papers when available; otherwise, we run the official code. For the recent work FAN \cite{fan_norm}, we adopt DLinear as the base forecaster, which is the recommended model in the original paper. Our code is available at this repository: \href{https://anonymous.4open.science/r/FreEformer}{https://anonymous.4open.science/r/FreEformer}.

\section{Robustness of FreEformer}

We report the standard deviation of FreEformer performance under seven runs with different random seeds in Table \ref{tab_robust}, demonstrating that the performance of FreEformer is stable.

\begin{table*}[t]
\begin{center}
\begin{tabular}{@{}c|cccccc@{}}
\toprule
Dataset & \multicolumn{2}{c}{ECL}          & \multicolumn{2}{c}{ETTh2}   & \multicolumn{2}{c}{Exchange} \\ \cmidrule(l){2-7} 
Horizon & MSE             & MAE            & MSE          & MAE          & MSE           & MAE          \\ \midrule
96      & 0.133±0.001     & 0.223±0.001    & 0.286±0.002  & 0.335±0.002  & 0.083±0.001   & 0.200±0.001  \\
192     & 0.152±0.001     & 0.240±0.001    & 0.363±0.003  & 0.381±0.002  & 0.174±0.002   & 0.296±0.002  \\
336     & 0.165±0.001     & 0.256±0.001    & 0.416±0.003  & 0.420±0.001  & 0.325±0.003   & 0.411±0.002  \\
720     & 0.198±0.002     & 0.286±0.002    & 0.422±0.002  & 0.436±0.001  & 0.833±0.009   & 0.687±0.006  \\ \midrule
Dataset & \multicolumn{2}{c}{Solar-Energy} & \multicolumn{2}{c}{Traffic} & \multicolumn{2}{c}{Weather}  \\ \cmidrule(l){2-7} 
Horizon & MSE             & MAE            & MSE          & MAE          & MSE           & MAE          \\ \midrule
96      & 0.180±0.001     & 0.191±0.000    & 0.395±0.002  & 0.233±0.000  & 0.153±0.001   & 0.189±0.001  \\
192     & 0.213±0.001     & 0.215±0.000    & 0.423±0.004  & 0.245±0.000  & 0.201±0.002   & 0.236±0.002  \\
336     & 0.233±0.001     & 0.232±0.000    & 0.443±0.005  & 0.254±0.000  & 0.261±0.003   & 0.282±0.002  \\
720     & 0.241±0.001     & 0.237±0.000    & 0.480±0.002  & 0.274±0.000  & 0.341±0.004   & 0.334±0.003  \\ \bottomrule
\end{tabular}
\caption{Robustness of FreEformer performance. The standard deviations are obtained from seven random seeds.}
\label{tab_robust}
\end{center}
\end{table*}

\section{Full Results}

\subsection{Long and Short-Term Forecasting Performance}

\begin{table*}[t]
\begin{center}
{\fontsize{7}{9}\selectfont
\begin{tabular}{@{}c@{\hspace{3pt}}c@{\hspace{2pt}}|@{\hspace{2pt}}
c@{\hspace{5pt}}c@{\hspace{2pt}}|@{\hspace{2pt}}
c@{\hspace{5pt}}c@{\hspace{2pt}}|@{\hspace{2pt}}
c@{\hspace{5pt}}c@{\hspace{2pt}}|@{\hspace{2pt}}
c@{\hspace{5pt}}c@{\hspace{2pt}}|@{\hspace{2pt}}
c@{\hspace{5pt}}c@{\hspace{2pt}}|@{\hspace{2pt}}
c@{\hspace{5pt}}c@{\hspace{2pt}}|@{\hspace{2pt}}
c@{\hspace{5pt}}c@{\hspace{2pt}}|@{\hspace{2pt}}
c@{\hspace{5pt}}c@{\hspace{2pt}}|@{\hspace{2pt}}
c@{\hspace{5pt}}c@{\hspace{2pt}}|@{\hspace{2pt}}
c@{\hspace{5pt}}c@{\hspace{2pt}}|@{\hspace{2pt}}
c@{\hspace{5pt}}c@{}}
\toprule
\multicolumn{2}{c|@{\hspace{2pt}}}{Model}          & \multicolumn{2}{c|@{\hspace{2pt}}}{\begin{tabular}[c]{@{}c@{}}FreEformer\\ (Ours)\end{tabular}} & 
\multicolumn{2}{c|@{\hspace{2pt}}}{\begin{tabular}[c]{@{}c@{}}Leddam\\ \shortcite{Leddam_icml} \end{tabular}}   & 
\multicolumn{2}{c|@{\hspace{2pt}}}{\begin{tabular}[c]{@{}c@{}}CARD\\ \shortcite{card}\end{tabular}}        & 
\multicolumn{2}{c|@{\hspace{2pt}}}{\begin{tabular}[c]{@{}c@{}}Fredformer\\ \shortcite{fredformer}\end{tabular}}  & 
\multicolumn{2}{c|@{\hspace{2pt}}}{\begin{tabular}[c]{@{}c@{}}iTrans.\\ \shortcite{itransformer}\end{tabular}}  & 
\multicolumn{2}{c|@{\hspace{2pt}}}{\begin{tabular}[c]{@{}c@{}}TimeMixer\\ \shortcite{timemixer}\end{tabular}} & 
\multicolumn{2}{c|@{\hspace{2pt}}}{\begin{tabular}[c]{@{}c@{}}PatchTST\\ \shortcite{patchtst}\end{tabular}}    & 
\multicolumn{2}{c|@{\hspace{2pt}}}{\begin{tabular}[c]{@{}c@{}}Crossfm.\\ \shortcite{crossformer} \end{tabular}} & 
\multicolumn{2}{c|@{\hspace{2pt}}}{\begin{tabular}[c]{@{}c@{}}TimesNet\\ \shortcite{timesnet} \end{tabular}} & 
\multicolumn{2}{c|@{\hspace{2pt}}}{\begin{tabular}[c]{@{}c@{}}FreTS\\ \shortcite{frets} \end{tabular}} & 
\multicolumn{2}{c}{\begin{tabular}[c]{@{}c@{}}DLinear\\ \shortcite{linear} \end{tabular}} \\ \midrule
\multicolumn{2}{c|@{\hspace{2pt}}}{Metric}          & MSE                                    & MAE                                    & MSE                                   & MAE                                & MSE                                   & MAE                                   & MSE                                   & MAE                                   & MSE                                   & MAE                                & MSE                                    & MAE                                & MSE                                   & MAE                                   & MSE                                  & MAE                                 & MSE                                  & MAE                                 & MSE                                & MAE                                & MSE                                               & MAE                  \\ \midrule
                               & 96  & {\color[HTML]{FF0000} \textbf{0.306}}  & {\color[HTML]{FF0000} \textbf{0.340}}  & 0.319                                 & 0.359                              & {\color[HTML]{0000FF} {\ul 0.316}}    & {\color[HTML]{0000FF} {\ul 0.347}}    & 0.326                                 & 0.361                                 & 0.334                                 & 0.368                              & 0.320                                  & 0.357                              & 0.329                                 & 0.367                                 & 0.404                                & 0.426                               & 0.338                                & 0.375                               & 0.339                              & 0.374                              & 0.345                                             & 0.372                \\
                               & 192 & {\color[HTML]{FF0000} \textbf{0.359}}  & {\color[HTML]{FF0000} \textbf{0.367}}  & 0.369                                 & 0.383                              & 0.363                                 & 0.370                                 & 0.363                                 & 0.380                                 & 0.377                                 & 0.391                              & {\color[HTML]{0000FF} {\ul 0.361}}     & {\color[HTML]{0000FF} {\ul 0.381}} & 0.367                                 & 0.385                                 & 0.450                                & 0.451                               & 0.374                                & 0.387                               & 0.382                              & 0.397                              & 0.380                                             & 0.389                \\
                               & 336 & {\color[HTML]{0000FF} {\ul 0.392}}     & {\color[HTML]{0000FF} {\ul 0.391}}     & 0.394                                 & 0.402                              & {\color[HTML]{0000FF} {\ul 0.392}}    & {\color[HTML]{FF0000} \textbf{0.390}} & 0.395                                 & 0.403                                 & 0.426                                 & 0.420                              & {\color[HTML]{FF0000} \textbf{0.390}}  & 0.404                              & 0.399                                 & 0.410                                 & 0.532                                & 0.515                               & 0.410                                & 0.411                               & 0.421                              & 0.426                              & 0.413                                             & 0.413                \\
                               & 720 & 0.458                                  & {\color[HTML]{0000FF} {\ul 0.428}}     & 0.460                                 & 0.442                              & 0.458                                 & {\color[HTML]{FF0000} \textbf{0.425}} & {\color[HTML]{FF0000} \textbf{0.453}} & 0.438                                 & 0.491                                 & 0.459                              & {\color[HTML]{0000FF} {\ul 0.454}}     & 0.441                              & {\color[HTML]{0000FF} {\ul 0.454}}    & 0.439                                 & 0.666                                & 0.589                               & 0.478                                & 0.450                               & 0.485                              & 0.462                              & 0.474                                             & 0.453                \\ \cmidrule(l){2-24} 
\multirow{-5}{*}{\rotatebox[origin=c]{90}{ETTm1}}        & Avg & {\color[HTML]{FF0000} \textbf{0.379}}  & {\color[HTML]{FF0000} \textbf{0.381}}  & 0.386                                 & 0.397                              & 0.383                                 & {\color[HTML]{0000FF} {\ul 0.384}}    & 0.384                                 & 0.395                                 & 0.407                                 & 0.410                              & {\color[HTML]{0000FF} {\ul 0.381}}     & 0.395                              & 0.387                                 & 0.400                                 & 0.513                                & 0.496                               & 0.400                                & 0.406                               & 0.407                              & 0.415                              & 0.403                                             & 0.407                \\ \midrule
                               & 96  & {\color[HTML]{FF0000} \textbf{0.169}}  & {\color[HTML]{FF0000} \textbf{0.248}}  & 0.176                                 & {\color[HTML]{0000FF} {\ul 0.257}} & {\color[HTML]{FF0000} \textbf{0.169}} & {\color[HTML]{FF0000} \textbf{0.248}} & 0.177                                 & 0.259                                 & 0.180                                 & 0.264                              & {\color[HTML]{0000FF} {\ul 0.175}}     & 0.258                              & {\color[HTML]{0000FF} {\ul 0.175}}    & 0.259                                 & 0.287                                & 0.366                               & 0.187                                & 0.267                               & 0.190                              & 0.282                              & 0.193                                             & 0.292                \\
                               & 192 & {\color[HTML]{FF0000} \textbf{0.233}}  & {\color[HTML]{FF0000} \textbf{0.289}}  & 0.243                                 & 0.303                              & {\color[HTML]{0000FF} {\ul 0.234}}    & {\color[HTML]{0000FF} {\ul 0.292}}    & 0.243                                 & 0.301                                 & 0.250                                 & 0.309                              & 0.237                                  & 0.299                              & 0.241                                 & 0.302                                 & 0.414                                & 0.492                               & 0.249                                & 0.309                               & 0.260                              & 0.329                              & 0.284                                             & 0.362                \\
                               & 336 & {\color[HTML]{FF0000} \textbf{0.292}}  & {\color[HTML]{FF0000} \textbf{0.328}}  & 0.303                                 & 0.341                              & {\color[HTML]{0000FF} {\ul 0.294}}    & {\color[HTML]{0000FF} {\ul 0.339}}    & 0.302                                 & 0.340                                 & 0.311                                 & 0.348                              & 0.298                                  & 0.340                              & 0.305                                 & 0.343                                 & 0.597                                & 0.542                               & 0.321                                & 0.351                               & 0.373                              & 0.405                              & 0.369                                             & 0.427                \\
                               & 720 & 0.395                                  & {\color[HTML]{0000FF} {\ul 0.389}}     & 0.400                                 & 0.398                              & {\color[HTML]{FF0000} \textbf{0.390}} & {\color[HTML]{FF0000} \textbf{0.388}} & 0.397                                 & 0.396                                 & 0.412                                 & 0.407                              & {\color[HTML]{0000FF} {\ul 0.391}}     & 0.396                              & 0.402                                 & 0.400                                 & 1.730                                & 1.042                               & 0.408                                & 0.403                               & 0.517                              & 0.499                              & 0.554                                             & 0.522                \\ \cmidrule(l){2-24} 
\multirow{-5}{*}{\rotatebox[origin=c]{90}{ETTm2}}        & Avg & {\color[HTML]{FF0000} \textbf{0.272}}  & {\color[HTML]{FF0000} \textbf{0.313}}  & 0.281                                 & 0.325                              & {\color[HTML]{FF0000} \textbf{0.272}} & {\color[HTML]{0000FF} {\ul 0.317}}    & 0.279                                 & 0.324                                 & 0.288                                 & 0.332                              & {\color[HTML]{0000FF} {\ul 0.275}}     & 0.323                              & 0.281                                 & 0.326                                 & 0.757                                & 0.610                               & 0.291                                & 0.333                               & 0.335                              & 0.379                              & 0.350                                             & 0.401                \\ \midrule
                               & 96  & {\color[HTML]{FF0000} \textbf{0.371}}  & {\color[HTML]{FF0000} \textbf{0.390}}  & 0.377                                 & 0.394                              & 0.383                                 & 0.391                                 & {\color[HTML]{0000FF} {\ul 0.373}}    & {\color[HTML]{0000FF} {\ul 0.392}}    & 0.386                                 & 0.405                              & 0.375                                  & 0.400                              & 0.414                                 & 0.419                                 & 0.423                                & 0.448                               & 0.384                                & 0.402                               & 0.399                              & 0.412                              & 0.386                                             & 0.400                \\
                               & 192 & {\color[HTML]{FF0000} \textbf{0.424}}  & {\color[HTML]{FF0000} \textbf{0.420}}  & {\color[HTML]{FF0000} \textbf{0.424}} & 0.422                              & 0.435                                 & {\color[HTML]{FF0000} \textbf{0.420}} & 0.433                                 & {\color[HTML]{FF0000} \textbf{0.420}} & 0.441                                 & 0.436                              & {\color[HTML]{0000FF} {\ul 0.429}}     & {\color[HTML]{0000FF} {\ul 0.421}} & 0.460                                 & 0.445                                 & 0.471                                & 0.474                               & 0.436                                & 0.429                               & 0.453                              & 0.443                              & 0.437                                             & 0.432                \\
                               & 336 & {\color[HTML]{0000FF} {\ul 0.466}}     & 0.443                                  & {\color[HTML]{FF0000} \textbf{0.459}} & {\color[HTML]{0000FF} {\ul 0.442}} & 0.479                                 & {\color[HTML]{0000FF} {\ul 0.442}}    & 0.470                                 & {\color[HTML]{FF0000} \textbf{0.437}} & 0.487                                 & 0.458                              & 0.484                                  & 0.458                              & 0.501                                 & 0.466                                 & 0.570                                & 0.546                               & 0.491                                & 0.469                               & 0.503                              & 0.475                              & 0.481                                             & 0.459                \\
                               & 720 & 0.471                                  & 0.470                                  & {\color[HTML]{FF0000} \textbf{0.463}} & {\color[HTML]{0000FF} {\ul 0.459}} & 0.471                                 & 0.461                                 & {\color[HTML]{0000FF} {\ul 0.467}}    & {\color[HTML]{FF0000} \textbf{0.456}} & 0.503                                 & 0.491                              & 0.498                                  & 0.482                              & 0.500                                 & 0.488                                 & 0.653                                & 0.621                               & 0.521                                & 0.500                               & 0.596                              & 0.565                              & 0.519                                             & 0.516                \\ \cmidrule(l){2-24} 
\multirow{-5}{*}{\rotatebox[origin=c]{90}{ETTh1}}        & Avg & {\color[HTML]{0000FF} {\ul 0.433}}     & 0.431                                  & {\color[HTML]{FF0000} \textbf{0.431}} & {\color[HTML]{0000FF} {\ul 0.429}} & 0.442                                 & {\color[HTML]{0000FF} {\ul 0.429}}    & 0.435                                 & {\color[HTML]{FF0000} \textbf{0.426}} & 0.454                                 & 0.447                              & 0.447                                  & 0.440                              & 0.469                                 & 0.454                                 & 0.529                                & 0.522                               & 0.458                                & 0.450                               & 0.488                              & 0.474                              & 0.456                                             & 0.452                \\ \midrule
                               & 96  & {\color[HTML]{0000FF} {\ul 0.286}}     & {\color[HTML]{0000FF} {\ul 0.335}}     & 0.292                                 & 0.343                              & {\color[HTML]{FF0000} \textbf{0.281}} & {\color[HTML]{FF0000} \textbf{0.330}} & 0.293                                 & 0.342                                 & 0.297                                 & 0.349                              & 0.289                                  & 0.341                              & 0.292                                 & 0.342                                 & 0.745                                & 0.584                               & 0.340                                & 0.374                               & 0.350                              & 0.403                              & 0.333                                             & 0.387                \\
                               & 192 & {\color[HTML]{FF0000} \textbf{0.363}}  & {\color[HTML]{FF0000} \textbf{0.381}}  & {\color[HTML]{0000FF} {\ul 0.367}}    & {\color[HTML]{0000FF} {\ul 0.389}} & {\color[HTML]{FF0000} \textbf{0.363}} & {\color[HTML]{FF0000} \textbf{0.381}} & 0.371                                 & {\color[HTML]{0000FF} {\ul 0.389}}    & 0.380                                 & 0.400                              & 0.372                                  & 0.392                              & 0.387                                 & 0.400                                 & 0.877                                & 0.656                               & 0.402                                & 0.414                               & 0.472                              & 0.475                              & 0.477                                             & 0.476                \\
                               & 336 & 0.416                                  & 0.420                                  & 0.412                                 & 0.424                              & 0.411                                 & 0.418                                 & {\color[HTML]{FF0000} \textbf{0.382}} & {\color[HTML]{FF0000} \textbf{0.409}} & 0.428                                 & 0.432                              & {\color[HTML]{0000FF} {\ul 0.386}}     & {\color[HTML]{0000FF} {\ul 0.414}} & 0.426                                 & 0.433                                 & 1.043                                & 0.731                               & 0.452                                & 0.452                               & 0.564                              & 0.528                              & 0.594                                             & 0.541                \\
                               & 720 & 0.422                                  & 0.436                                  & 0.419                                 & 0.438                              & 0.416                                 & {\color[HTML]{FF0000} \textbf{0.431}} & {\color[HTML]{0000FF} {\ul 0.415}}    & {\color[HTML]{0000FF} {\ul 0.434}}    & 0.427                                 & 0.445                              & {\color[HTML]{FF0000} \textbf{0.412}}  & {\color[HTML]{0000FF} {\ul 0.434}} & 0.431                                 & 0.446                                 & 1.104                                & 0.763                               & 0.462                                & 0.468                               & 0.815                              & 0.654                              & 0.831                                             & 0.657                \\ \cmidrule(l){2-24} 
\multirow{-5}{*}{\rotatebox[origin=c]{90}{ETTh2}}        & Avg & 0.372                                  & {\color[HTML]{0000FF} {\ul 0.393}}     & 0.373                                 & 0.399                              & 0.368                                 & {\color[HTML]{FF0000} \textbf{0.390}} & {\color[HTML]{0000FF} {\ul 0.365}}    & {\color[HTML]{0000FF} {\ul 0.393}}    & 0.383                                 & 0.407                              & {\color[HTML]{FF0000} \textbf{0.364}}  & 0.395                              & 0.384                                 & 0.405                                 & 0.942                                & 0.684                               & 0.414                                & 0.427                               & 0.550                              & 0.515                              & 0.559                                             & 0.515                \\ \midrule
                               & 96  & {\color[HTML]{FF0000} \textbf{0.133}}  & {\color[HTML]{FF0000} \textbf{0.223}}  & {\color[HTML]{0000FF} {\ul 0.141}}    & {\color[HTML]{0000FF} {\ul 0.235}} & {\color[HTML]{0000FF} {\ul 0.141}}    & {\color[HTML]{FF0000} \textbf{0.233}} & 0.147                                 & 0.241                                 & 0.148                                 & 0.240                              & 0.153                                  & 0.247                              & 0.161                                 & 0.250                                 & 0.219                                & 0.314                               & 0.168                                & 0.272                               & 0.183                              & 0.269                              & 0.197                                             & 0.282                \\
                               & 192 & {\color[HTML]{FF0000} \textbf{0.152}}  & {\color[HTML]{FF0000} \textbf{0.240}}  & {\color[HTML]{0000FF} {\ul 0.159}}    & 0.252                              & 0.160                                 & {\color[HTML]{0000FF} {\ul 0.250}}    & 0.165                                 & 0.258                                 & 0.162                                 & 0.253                              & 0.166                                  & 0.256                              & 0.199                                 & 0.289                                 & 0.231                                & 0.322                               & 0.184                                & 0.289                               & 0.187                              & 0.276                              & 0.196                                             & 0.285                \\
                               & 336 & {\color[HTML]{FF0000} \textbf{0.165}}  & {\color[HTML]{FF0000} \textbf{0.256}}  & {\color[HTML]{0000FF} {\ul 0.173}}    & 0.268                              & {\color[HTML]{0000FF} {\ul 0.173}}    & {\color[HTML]{0000FF} {\ul 0.263}}    & 0.177                                 & 0.273                                 & 0.178                                 & 0.269                              & 0.185                                  & 0.277                              & 0.215                                 & 0.305                                 & 0.246                                & 0.337                               & 0.198                                & 0.300                               & 0.202                              & 0.292                              & 0.209                                             & 0.301                \\
                               & 720 & {\color[HTML]{0000FF} {\ul 0.198}}     & {\color[HTML]{0000FF} {\ul 0.286}}     & 0.201                                 & 0.295                              & {\color[HTML]{FF0000} \textbf{0.197}} & {\color[HTML]{FF0000} \textbf{0.284}} & 0.213                                 & 0.304                                 & 0.225                                 & 0.317                              & 0.225                                  & 0.310                              & 0.256                                 & 0.337                                 & 0.280                                & 0.363                               & 0.220                                & 0.320                               & 0.237                              & 0.325                              & 0.245                                             & 0.333                \\ \cmidrule(l){2-24} 
\multirow{-5}{*}{\rotatebox[origin=c]{90}{ECL}}          & Avg & {\color[HTML]{FF0000} \textbf{0.162}}  & {\color[HTML]{FF0000} \textbf{0.251}}  & 0.169                                 & 0.263                              & {\color[HTML]{0000FF} {\ul 0.168}}    & {\color[HTML]{0000FF} {\ul 0.258}}    & 0.176                                 & 0.269                                 & 0.178                                 & 0.270                              & 0.182                                  & 0.272                              & 0.208                                 & 0.295                                 & 0.244                                & 0.334                               & 0.192                                & 0.295                               & 0.202                              & 0.290                              & 0.212                                             & 0.300                \\ \midrule
                               & 96  & {\color[HTML]{FF0000} \textbf{0.083}}  & {\color[HTML]{FF0000} \textbf{0.200}}  & 0.086                                 & 0.207                              & {\color[HTML]{0000FF} {\ul 0.084}}    & {\color[HTML]{0000FF} {\ul 0.202}}    & {\color[HTML]{0000FF} {\ul 0.084}}    & {\color[HTML]{0000FF} {\ul 0.202}}    & 0.086                                 & 0.206                              & 0.086                                  & 0.205                              & 0.088                                 & 0.205                                 & 0.256                                & 0.367                               & 0.107                                & 0.234                               & 0.086                              & 0.212                              & 0.088                                             & 0.218                \\
                               & 192 & {\color[HTML]{FF0000} \textbf{0.174}}  & {\color[HTML]{FF0000} \textbf{0.296}}  & {\color[HTML]{0000FF} {\ul 0.175}}    & 0.301                              & 0.179                                 & {\color[HTML]{0000FF} {\ul 0.298}}    & 0.178                                 & 0.302                                 & 0.177                                 & 0.299                              & 0.193                                  & 0.312                              & 0.176                                 & 0.299                                 & 0.470                                & 0.509                               & 0.226                                & 0.344                               & 0.217                              & 0.344                              & 0.176                                             & 0.315                \\
                               & 336 & 0.325                                  & 0.411                                  & 0.325                                 & 0.415                              & 0.333                                 & 0.418                                 & 0.319                                 & {\color[HTML]{0000FF} {\ul 0.408}}    & 0.331                                 & 0.417                              & 0.356                                  & 0.433                              & {\color[HTML]{FF0000} \textbf{0.301}} & {\color[HTML]{FF0000} \textbf{0.397}} & 1.268                                & 0.883                               & 0.367                                & 0.448                               & 0.415                              & 0.475                              & {\color[HTML]{0000FF} {\ul 0.313}}                & 0.427                \\
                               & 720 & 0.833                                  & 0.687                                  & {\color[HTML]{0000FF} {\ul 0.831}}    & {\color[HTML]{0000FF} {\ul 0.686}} & 0.851                                 & 0.691                                 & {\color[HTML]{FF0000} \textbf{0.749}} & {\color[HTML]{FF0000} \textbf{0.651}} & 0.847                                 & 0.691                              & 0.912                                  & 0.712                              & 0.901                                 & 0.714                                 & 1.767                                & 1.068                               & 0.964                                & 0.746                               & 0.947                              & 0.725                              & 0.839                                             & 0.695                \\ \cmidrule(l){2-24} 
\multirow{-5}{*}{\rotatebox[origin=c]{90}{Exchange}}     & Avg & {\color[HTML]{0000FF} {\ul 0.354}}     & {\color[HTML]{0000FF} {\ul 0.399}}     & {\color[HTML]{0000FF} {\ul 0.354}}    & 0.402                              & 0.362                                 & 0.402                                 & {\color[HTML]{FF0000} \textbf{0.333}} & {\color[HTML]{FF0000} \textbf{0.391}} & 0.360                                 & 0.403                              & 0.387                                  & 0.416                              & 0.367                                 & 0.404                                 & 0.940                                & 0.707                               & 0.416                                & 0.443                               & 0.416                              & 0.439                              & {\color[HTML]{0000FF} {\ul 0.354}}                & 0.414                \\ \midrule
                               & 96  & {\color[HTML]{FF0000} \textbf{0.395}}  & {\color[HTML]{FF0000} \textbf{0.233}}  & 0.426                                 & 0.276                              & 0.419                                 & 0.269                                 & {\color[HTML]{0000FF} {\ul 0.406}}    & 0.277                                 & {\color[HTML]{FF0000} \textbf{0.395}} & {\color[HTML]{0000FF} {\ul 0.268}} & 0.462                                  & 0.285                              & 0.446                                 & 0.283                                 & 0.522                                & 0.290                               & 0.593                                & 0.321                               & 0.519                              & 0.315                              & 0.650                                             & 0.396                \\
                               & 192 & {\color[HTML]{0000FF} {\ul 0.423}}     & {\color[HTML]{FF0000} \textbf{0.245}}  & 0.458                                 & 0.289                              & 0.443                                 & {\color[HTML]{0000FF} {\ul 0.276}}    & 0.426                                 & 0.290                                 & {\color[HTML]{FF0000} \textbf{0.417}} & {\color[HTML]{0000FF} {\ul 0.276}} & 0.473                                  & 0.296                              & 0.540                                 & 0.354                                 & 0.530                                & 0.293                               & 0.617                                & 0.336                               & 0.521                              & 0.325                              & 0.598                                             & 0.370                \\
                               & 336 & 0.443                                  & {\color[HTML]{FF0000} \textbf{0.254}}  & 0.486                                 & 0.297                              & 0.460                                 & 0.283                                 & {\color[HTML]{0000FF} {\ul 0.437}}    & 0.292                                 & {\color[HTML]{FF0000} \textbf{0.433}} & {\color[HTML]{0000FF} {\ul 0.283}} & 0.498                                  & 0.296                              & 0.551                                 & 0.358                                 & 0.558                                & 0.305                               & 0.629                                & 0.336                               & 0.533                              & 0.327                              & 0.605                                             & 0.373                \\
                               & 720 & 0.480                                  & {\color[HTML]{FF0000} \textbf{0.274}}  & 0.498                                 & 0.313                              & 0.490                                 & {\color[HTML]{0000FF} {\ul 0.299}}    & {\color[HTML]{FF0000} \textbf{0.462}} & 0.305                                 & {\color[HTML]{0000FF} {\ul 0.467}}    & 0.302                              & 0.506                                  & 0.313                              & 0.586                                 & 0.375                                 & 0.589                                & 0.328                               & 0.640                                & 0.350                               & 0.580                              & 0.344                              & 0.645                                             & 0.394                \\ \cmidrule(l){2-24} 
\multirow{-5}{*}{\rotatebox[origin=c]{90}{Traffic}}      & Avg & 0.435                                  & {\color[HTML]{FF0000} \textbf{0.251}}  & 0.467                                 & 0.294                              & 0.453                                 & {\color[HTML]{0000FF} {\ul 0.282}}    & {\color[HTML]{0000FF} {\ul 0.433}}    & 0.291                                 & {\color[HTML]{FF0000} \textbf{0.428}} & {\color[HTML]{0000FF} {\ul 0.282}} & 0.484                                  & 0.297                              & 0.531                                 & 0.343                                 & 0.550                                & 0.304                               & 0.620                                & 0.336                               & 0.538                              & 0.328                              & 0.625                                             & 0.383                \\ \midrule
                               & 96  & {\color[HTML]{0000FF} {\ul 0.153}}     & {\color[HTML]{0000FF} {\ul 0.189}}     & 0.156                                 & 0.202                              & {\color[HTML]{FF0000} \textbf{0.150}} & {\color[HTML]{FF0000} \textbf{0.188}} & 0.163                                 & 0.207                                 & 0.174                                 & 0.214                              & 0.163                                  & 0.209                              & 0.177                                 & 0.218                                 & 0.158                                & 0.230                               & 0.172                                & 0.220                               & 0.184                              & 0.239                              & 0.196                                             & 0.255                \\
                               & 192 & {\color[HTML]{FF0000} \textbf{0.201}}  & {\color[HTML]{FF0000} \textbf{0.236}}  & 0.207                                 & 0.250                              & {\color[HTML]{0000FF} {\ul 0.202}}    & {\color[HTML]{0000FF} {\ul 0.238}}    & 0.211                                 & 0.251                                 & 0.221                                 & 0.254                              & 0.208                                  & 0.250                              & 0.225                                 & 0.259                                 & 0.206                                & 0.277                               & 0.219                                & 0.261                               & 0.223                              & 0.275                              & 0.237                                             & 0.296                \\
                               & 336 & 0.261                                  & {\color[HTML]{FF0000} \textbf{0.282}}  & 0.262                                 & 0.291                              & {\color[HTML]{0000FF} {\ul 0.260}}    & {\color[HTML]{FF0000} \textbf{0.282}} & 0.267                                 & 0.292                                 & 0.278                                 & 0.296                              & {\color[HTML]{FF0000} \textbf{0.251}}  & {\color[HTML]{0000FF} {\ul 0.287}} & 0.278                                 & 0.297                                 & 0.272                                & 0.335                               & 0.280                                & 0.306                               & 0.272                              & 0.316                              & 0.283                                             & 0.335                \\
                               & 720 & {\color[HTML]{0000FF} {\ul 0.341}}     & {\color[HTML]{FF0000} \textbf{0.334}}  & 0.343                                 & 0.343                              & 0.343                                 & 0.353                                 & 0.343                                 & {\color[HTML]{0000FF} {\ul 0.341}}    & 0.358                                 & 0.349                              & {\color[HTML]{FF0000} \textbf{0.339}}  & {\color[HTML]{0000FF} {\ul 0.341}} & 0.354                                 & 0.348                                 & 0.398                                & 0.418                               & 0.365                                & 0.359                               & 0.340                              & 0.363                              & 0.345                                             & 0.381                \\ \cmidrule(l){2-24} 
\multirow{-5}{*}{\rotatebox[origin=c]{90}{Weather}}      & Avg & {\color[HTML]{FF0000} \textbf{0.239}}  & {\color[HTML]{FF0000} \textbf{0.260}}  & 0.242                                 & 0.272                              & {\color[HTML]{FF0000} \textbf{0.239}} & {\color[HTML]{0000FF} {\ul 0.265}}    & 0.246                                 & 0.272                                 & 0.258                                 & 0.279                              & {\color[HTML]{0000FF} {\ul 0.240}}     & 0.271                              & 0.259                                 & 0.281                                 & 0.259                                & 0.315                               & 0.259                                & 0.287                               & 0.255                              & 0.298                              & 0.265                                             & 0.317                \\ \midrule
                               & 96  & {\color[HTML]{FF0000} \textbf{0.180}}  & {\color[HTML]{FF0000} \textbf{0.191}}  & 0.197                                 & 0.241                              & 0.197                                 & {\color[HTML]{0000FF} {\ul 0.211}}    & {\color[HTML]{0000FF} {\ul 0.185}}    & 0.233                                 & 0.203                                 & 0.237                              & 0.189                                  & 0.259                              & 0.234                                 & 0.286                                 & 0.310                                & 0.331                               & 0.250                                & 0.292                               & 0.192                              & 0.225                              & 0.290                                             & 0.378                \\
                               & 192 & {\color[HTML]{FF0000} \textbf{0.213}}  & {\color[HTML]{FF0000} \textbf{0.215}}  & 0.231                                 & 0.264                              & 0.234                                 & 0.234                                 & 0.227                                 & {\color[HTML]{0000FF} {\ul 0.253}}    & 0.233                                 & 0.261                              & {\color[HTML]{0000FF} {\ul 0.222}}     & 0.283                              & 0.267                                 & 0.310                                 & 0.734                                & 0.725                               & 0.296                                & 0.318                               & 0.229                              & 0.252                              & 0.320                                             & 0.398                \\
                               & 336 & {\color[HTML]{0000FF} {\ul 0.233}}     & {\color[HTML]{FF0000} \textbf{0.232}}  & 0.241                                 & 0.268                              & 0.256                                 & {\color[HTML]{0000FF} {\ul 0.250}}    & 0.246                                 & 0.284                                 & 0.248                                 & 0.273                              & {\color[HTML]{FF0000} \textbf{0.231}}  & 0.292                              & 0.290                                 & 0.315                                 & 0.750                                & 0.735                               & 0.319                                & 0.330                               & 0.242                              & 0.269                              & 0.353                                             & 0.415                \\
                               & 720 & {\color[HTML]{0000FF} {\ul 0.241}}     & {\color[HTML]{FF0000} \textbf{0.237}}  & 0.250                                 & 0.281                              & 0.260                                 & {\color[HTML]{0000FF} {\ul 0.254}}    & 0.247                                 & 0.276                                 & 0.249                                 & 0.275                              & {\color[HTML]{FF0000} \textbf{0.223}}  & 0.285                              & 0.289                                 & 0.317                                 & 0.769                                & 0.765                               & 0.338                                & 0.337                               & 0.240                              & 0.272                              & 0.356                                             & 0.413                \\ \cmidrule(l){2-24} 
\multirow{-5}{*}{\rotatebox[origin=c]{90}{Solar-Energy}} & Avg & {\color[HTML]{0000FF} {\ul 0.217}}     & {\color[HTML]{FF0000} \textbf{0.219}}  & 0.230                                 & 0.264                              & 0.237                                 & {\color[HTML]{0000FF} {\ul 0.237}}    & 0.226                                 & 0.262                                 & 0.233                                 & 0.262                              & {\color[HTML]{FF0000} \textbf{0.216}}  & 0.280                              & 0.270                                 & 0.307                                 & 0.641                                & 0.639                               & 0.301                                & 0.319                               & 0.226                              & 0.254                              & 0.330                                             & 0.401                \\ \midrule 

\multicolumn{2}{c|@{\hspace{2pt}}}{1st Count}       & 21                                     & 30                                     & 4                                     & 0                                  & 8                                     & 13                                    & 5                                     & 7                                     & 4                                     & 0                                  & 8                                      & 0                                  & 1                                     & 1                                     & 0                                    & 0                                   & 0                                    & 0                                   & 0                                  & 0                                  & 0                                                 & 0                    \\ \bottomrule

\end{tabular}
}
\caption{Full results of long-term multivariate time series forecasting results for prediction lengths $\tau \in \left \{ 96,192,336,720 \right \}$ with the lookback length being 96. Results for the PEMS datasets are provided in Table \ref{tab_long_term_appd_p2}.}
\label{tab_long_term_appd_p1}
\end{center}
\end{table*}

\begin{table*}[t]
\begin{center}
{\fontsize{7}{9}\selectfont
\begin{tabular}{@{}c@{\hspace{3pt}}c@{\hspace{2pt}}|@{\hspace{2pt}}
c@{\hspace{5pt}}c@{\hspace{2pt}}|@{\hspace{2pt}}
c@{\hspace{5pt}}c@{\hspace{2pt}}|@{\hspace{2pt}}
c@{\hspace{5pt}}c@{\hspace{2pt}}|@{\hspace{2pt}}
c@{\hspace{5pt}}c@{\hspace{2pt}}|@{\hspace{2pt}}
c@{\hspace{5pt}}c@{\hspace{2pt}}|@{\hspace{2pt}}
c@{\hspace{5pt}}c@{\hspace{2pt}}|@{\hspace{2pt}}
c@{\hspace{5pt}}c@{\hspace{2pt}}|@{\hspace{2pt}}
c@{\hspace{5pt}}c@{\hspace{2pt}}|@{\hspace{2pt}}
c@{\hspace{5pt}}c@{\hspace{2pt}}|@{\hspace{2pt}}
c@{\hspace{5pt}}c@{\hspace{2pt}}|@{\hspace{2pt}}
c@{\hspace{5pt}}c@{}}
\toprule
\multicolumn{2}{c|@{\hspace{2pt}}}{Model}          & \multicolumn{2}{c|@{\hspace{2pt}}}{\begin{tabular}[c]{@{}c@{}}FreEformer\\ (Ours)\end{tabular}} & 
\multicolumn{2}{c|@{\hspace{2pt}}}{\begin{tabular}[c]{@{}c@{}}Leddam\\ \shortcite{Leddam_icml} \end{tabular}}   & 
\multicolumn{2}{c|@{\hspace{2pt}}}{\begin{tabular}[c]{@{}c@{}}CARD\\ \shortcite{card}\end{tabular}}        & 
\multicolumn{2}{c|@{\hspace{2pt}}}{\begin{tabular}[c]{@{}c@{}}Fredformer\\ \shortcite{fredformer}\end{tabular}}  & 
\multicolumn{2}{c|@{\hspace{2pt}}}{\begin{tabular}[c]{@{}c@{}}iTrans.\\ \shortcite{itransformer}\end{tabular}}  & 
\multicolumn{2}{c|@{\hspace{2pt}}}{\begin{tabular}[c]{@{}c@{}}TimeMixer\\ \shortcite{timemixer}\end{tabular}} & 
\multicolumn{2}{c|@{\hspace{2pt}}}{\begin{tabular}[c]{@{}c@{}}PatchTST\\ \shortcite{patchtst}\end{tabular}}    & 
\multicolumn{2}{c|@{\hspace{2pt}}}{\begin{tabular}[c]{@{}c@{}}Crossfm.\\ \shortcite{crossformer} \end{tabular}} & 
\multicolumn{2}{c|@{\hspace{2pt}}}{\begin{tabular}[c]{@{}c@{}}TimesNet\\ \shortcite{timesnet} \end{tabular}} & 
\multicolumn{2}{c|@{\hspace{2pt}}}{\begin{tabular}[c]{@{}c@{}}FreTS\\ \shortcite{frets} \end{tabular}} & 
\multicolumn{2}{c}{\begin{tabular}[c]{@{}c@{}}DLinear\\ \shortcite{linear} \end{tabular}} \\ \midrule
\multicolumn{2}{c|@{\hspace{2pt}}}{Metric}          & MSE                                    & MAE                                    & MSE                                   & MAE                                & MSE                                   & MAE                                   & MSE                                   & MAE                                   & MSE                                   & MAE                                & MSE                                    & MAE                                & MSE                                   & MAE                                   & MSE                                  & MAE                                 & MSE                                  & MAE                                 & MSE                                & MAE                                & MSE                                               & MAE                  \\ \midrule

                               & 12  & {\color[HTML]{FF0000} \textbf{0.060}}  & {\color[HTML]{FF0000} \textbf{0.160}}  & {\color[HTML]{0000FF} {\ul 0.063}}    & {\color[HTML]{0000FF} {\ul 0.164}} & 0.072                                 & 0.177                                 & 0.068                                 & 0.174                                 & 0.071                                 & 0.174                              & 0.076                                  & 0.188                              & 0.099                                 & 0.216                                 & 0.090                                & 0.203                               & 0.085                                & 0.192                               & 0.083                              & 0.194                              & 0.122                                             & 0.243                \\
                               & 24  & {\color[HTML]{FF0000} \textbf{0.077}}  & {\color[HTML]{FF0000} \textbf{0.181}}  & {\color[HTML]{0000FF} {\ul 0.080}}    & {\color[HTML]{0000FF} {\ul 0.185}} & 0.107                                 & 0.217                                 & 0.094                                 & 0.205                                 & 0.093                                 & 0.201                              & 0.113                                  & 0.226                              & 0.142                                 & 0.259                                 & 0.121                                & 0.240                               & 0.118                                & 0.223                               & 0.127                              & 0.241                              & 0.201                                             & 0.317                \\
                               & 48  & {\color[HTML]{FF0000} \textbf{0.112}}  & {\color[HTML]{FF0000} \textbf{0.218}}  & {\color[HTML]{0000FF} {\ul 0.124}}    & {\color[HTML]{0000FF} {\ul 0.226}} & 0.194                                 & 0.302                                 & 0.152                                 & 0.262                                 & 0.125                                 & 0.236                              & 0.191                                  & 0.292                              & 0.211                                 & 0.319                                 & 0.202                                & 0.317                               & 0.155                                & 0.260                               & 0.202                              & 0.310                              & 0.333                                             & 0.425                \\
                               & 96  & {\color[HTML]{FF0000} \textbf{0.159}}  & {\color[HTML]{FF0000} \textbf{0.265}}  & {\color[HTML]{0000FF} {\ul 0.160}}    & {\color[HTML]{0000FF} {\ul 0.266}} & 0.323                                 & 0.402                                 & 0.228                                 & 0.330                                 & 0.164                                 & 0.275                              & 0.288                                  & 0.363                              & 0.269                                 & 0.370                                 & 0.262                                & 0.367                               & 0.228                                & 0.317                               & 0.265                              & 0.365                              & 0.457                                             & 0.515                \\ \cmidrule(l){2-24} 
\multirow{-5}{*}{\rotatebox[origin=c]{90}{PEMS03}}      & Avg & {\color[HTML]{FF0000} \textbf{0.102}}  & {\color[HTML]{FF0000} \textbf{0.206}}  & {\color[HTML]{0000FF} {\ul 0.107}}    & {\color[HTML]{0000FF} {\ul 0.210}} & 0.174                                 & 0.275                                 & 0.135                                 & 0.243                                 & 0.113                                 & 0.221                              & 0.167                                  & 0.267                              & 0.180                                 & 0.291                                 & 0.169                                & 0.281                               & 0.147                                & 0.248                               & 0.169                              & 0.278                              & 0.278                                             & 0.375                \\ \midrule
                               & 12  & {\color[HTML]{FF0000} \textbf{0.068}}  & {\color[HTML]{FF0000} \textbf{0.164}}  & {\color[HTML]{0000FF} {\ul 0.071}}    & {\color[HTML]{0000FF} {\ul 0.172}} & 0.089                                 & 0.194                                 & 0.085                                 & 0.189                                 & 0.078                                 & 0.183                              & 0.092                                  & 0.204                              & 0.105                                 & 0.224                                 & 0.098                                & 0.218                               & 0.087                                & 0.195                               & 0.097                              & 0.209                              & 0.148                                             & 0.272                \\
                               & 24  & {\color[HTML]{FF0000} \textbf{0.079}}  & {\color[HTML]{FF0000} \textbf{0.179}}  & {\color[HTML]{0000FF} {\ul 0.087}}    & {\color[HTML]{0000FF} {\ul 0.193}} & 0.128                                 & 0.234                                 & 0.117                                 & 0.224                                 & 0.095                                 & 0.205                              & 0.128                                  & 0.243                              & 0.153                                 & 0.275                                 & 0.131                                & 0.256                               & 0.103                                & 0.215                               & 0.144                              & 0.258                              & 0.224                                             & 0.340                \\
                               & 48  & {\color[HTML]{FF0000} \textbf{0.099}}  & {\color[HTML]{FF0000} \textbf{0.204}}  & {\color[HTML]{0000FF} {\ul 0.113}}    & {\color[HTML]{0000FF} {\ul 0.222}} & 0.224                                 & 0.321                                 & 0.174                                 & 0.276                                 & 0.120                                 & 0.233                              & 0.213                                  & 0.315                              & 0.229                                 & 0.339                                 & 0.205                                & 0.326                               & 0.136                                & 0.250                               & 0.223                              & 0.328                              & 0.355                                             & 0.437                \\
                               & 96  & {\color[HTML]{FF0000} \textbf{0.129}}  & {\color[HTML]{FF0000} \textbf{0.238}}  & {\color[HTML]{0000FF} {\ul 0.141}}    & {\color[HTML]{0000FF} {\ul 0.252}} & 0.382                                 & 0.445                                 & 0.273                                 & 0.354                                 & 0.150                                 & 0.262                              & 0.307                                  & 0.384                              & 0.291                                 & 0.389                                 & 0.402                                & 0.457                               & 0.190                                & 0.303                               & 0.288                              & 0.379                              & 0.452                                             & 0.504                \\ \cmidrule(l){2-24} 
\multirow{-5}{*}{\rotatebox[origin=c]{90}{PEMS04}}      & Avg & {\color[HTML]{FF0000} \textbf{0.094}}  & {\color[HTML]{FF0000} \textbf{0.196}}  & {\color[HTML]{0000FF} {\ul 0.103}}    & {\color[HTML]{0000FF} {\ul 0.210}} & 0.206                                 & 0.299                                 & 0.162                                 & 0.261                                 & 0.111                                 & 0.221                              & 0.185                                  & 0.287                              & 0.195                                 & 0.307                                 & 0.209                                & 0.314                               & 0.129                                & 0.241                               & 0.188                              & 0.294                              & 0.295                                             & 0.388                \\  \midrule

& 12  & {\color[HTML]{FF0000} \textbf{0.053}}  & {\color[HTML]{FF0000} \textbf{0.138}}  & {\color[HTML]{0000FF} {\ul 0.055}}    & {\color[HTML]{0000FF} {\ul 0.145}} & 0.068                                 & 0.166                                 & 0.063                                 & 0.158                                 & 0.067                                 & 0.165                              & 0.073                                  & 0.184                              & 0.095                                 & 0.207                                 & 0.094                                & 0.200                               & 0.082                                & 0.181                               & 0.078                              & 0.185                              & 0.115                                             & 0.242                \\
                               & 24  & {\color[HTML]{FF0000} \textbf{0.067}}  & {\color[HTML]{FF0000} \textbf{0.153}}  & {\color[HTML]{0000FF} {\ul 0.070}}    & {\color[HTML]{0000FF} {\ul 0.164}} & 0.103                                 & 0.206                                 & 0.089                                 & 0.192                                 & 0.088                                 & 0.190                              & 0.111                                  & 0.219                              & 0.150                                 & 0.262                                 & 0.139                                & 0.247                               & 0.101                                & 0.204                               & 0.127                              & 0.239                              & 0.210                                             & 0.329                \\
                               & 48  & {\color[HTML]{FF0000} \textbf{0.087}}  & {\color[HTML]{FF0000} \textbf{0.174}}  & {\color[HTML]{0000FF} {\ul 0.094}}    & {\color[HTML]{0000FF} {\ul 0.192}} & 0.165                                 & 0.268                                 & 0.136                                 & 0.241                                 & 0.110                                 & 0.215                              & 0.237                                  & 0.328                              & 0.253                                 & 0.340                                 & 0.311                                & 0.369                               & 0.134                                & 0.238                               & 0.220                              & 0.317                              & 0.398                                             & 0.458                \\
                               & 96  & {\color[HTML]{FF0000} \textbf{0.115}}  & {\color[HTML]{FF0000} \textbf{0.205}}  & {\color[HTML]{0000FF} {\ul 0.117}}    & {\color[HTML]{0000FF} {\ul 0.217}} & 0.258                                 & 0.346                                 & 0.197                                 & 0.298                                 & 0.139                                 & 0.245                              & 0.303                                  & 0.354                              & 0.346                                 & 0.404                                 & 0.396                                & 0.442                               & 0.181                                & 0.279                               & 0.316                              & 0.386                              & 0.594                                             & 0.553                \\ \cmidrule(l){2-24} 
\multirow{-5}{*}{\rotatebox[origin=c]{90}{PEMS07}}      & Avg & {\color[HTML]{FF0000} \textbf{0.080}}  & {\color[HTML]{FF0000} \textbf{0.167}}  & {\color[HTML]{0000FF} {\ul 0.084}}    & {\color[HTML]{0000FF} {\ul 0.180}} & 0.149                                 & 0.247                                 & 0.121                                 & 0.222                                 & 0.101                                 & 0.204                              & 0.181                                  & 0.271                              & 0.211                                 & 0.303                                 & 0.235                                & 0.315                               & 0.124                                & 0.225                               & 0.185                              & 0.282                              & 0.329                                             & 0.395                \\ \midrule
                               & 12  & {\color[HTML]{FF0000} \textbf{0.066}}  & {\color[HTML]{FF0000} \textbf{0.157}}  & {\color[HTML]{0000FF} {\ul 0.071}}    & {\color[HTML]{0000FF} {\ul 0.171}} & 0.080                                 & 0.181                                 & 0.081                                 & 0.185                                 & 0.079                                 & 0.182                              & 0.091                                  & 0.201                              & 0.168                                 & 0.232                                 & 0.165                                & 0.214                               & 0.112                                & 0.212                               & 0.096                              & 0.204                              & 0.154                                             & 0.276                \\
                               & 24  & {\color[HTML]{FF0000} \textbf{0.085}}  & {\color[HTML]{FF0000} \textbf{0.176}}  & {\color[HTML]{0000FF} {\ul 0.091}}    & {\color[HTML]{0000FF} {\ul 0.189}} & 0.118                                 & 0.220                                 & 0.112                                 & 0.214                                 & 0.115                                 & 0.219                              & 0.137                                  & 0.246                              & 0.224                                 & 0.281                                 & 0.215                                & 0.260                               & 0.141                                & 0.238                               & 0.152                              & 0.256                              & 0.248                                             & 0.353                \\
                               & 48  & {\color[HTML]{FF0000} \textbf{0.119}}  & {\color[HTML]{FF0000} \textbf{0.207}}  & {\color[HTML]{0000FF} {\ul 0.128}}    & {\color[HTML]{0000FF} {\ul 0.219}} & 0.199                                 & 0.289                                 & 0.174                                 & 0.267                                 & 0.186                                 & 0.235                              & 0.265                                  & 0.343                              & 0.321                                 & 0.354                                 & 0.315                                & 0.355                               & 0.198                                & 0.283                               & 0.247                              & 0.331                              & 0.440                                             & 0.470                \\
                               & 96  & {\color[HTML]{FF0000} \textbf{0.169}}  & {\color[HTML]{FF0000} \textbf{0.238}}  & {\color[HTML]{0000FF} {\ul 0.198}}    & {\color[HTML]{0000FF} {\ul 0.266}} & 0.405                                 & 0.431                                 & 0.277                                 & 0.335                                 & 0.221                                 & 0.267                              & 0.410                                  & 0.407                              & 0.408                                 & 0.417                                 & 0.377                                & 0.397                               & 0.320                                & 0.351                               & 0.354                              & 0.395                              & 0.674                                             & 0.565                \\ \cmidrule(l){2-24} 
\multirow{-5}{*}{\rotatebox[origin=c]{90}{PEMS08}}      & Avg & {\color[HTML]{FF0000} \textbf{0.110}}  & {\color[HTML]{FF0000} \textbf{0.194}}  & {\color[HTML]{0000FF} {\ul 0.122}}    & {\color[HTML]{0000FF} {\ul 0.211}} & 0.201                                 & 0.280                                 & 0.161                                 & 0.250                                 & 0.150                                 & 0.226                              & 0.226                                  & 0.299                              & 0.280                                 & 0.321                                 & 0.268                                & 0.307                               & 0.193                                & 0.271                               & 0.212                              & 0.297                              & 0.379                                             & 0.416                \\ \midrule

\multicolumn{2}{c|@{\hspace{2pt}}}{1st Count}       & 20                                     & 20                                     & 0                                     & 0                                  & 0                                     & 0                                    & 0                                     & 0                                     & 0                                     & 0                                  & 0                                      & 0                                  & 0                                     & 0                                     & 0                                    & 0                                   & 0                                    & 0                                   & 0                                  & 0                                  & 0                                                 & 0                    \\ \bottomrule

\end{tabular}
}
\caption{Full results of long-term multivariate time series forecasting for the PEMS datasets. The prediction lengths are $\tau \in \left \{ 96,192,336,720 \right \}$ and the lookback length is 96. The bast and second-best results are highlighted in {\color[HTML]{FF0000} \textbf{bold}} and {\color[HTML]{0000FF} {\ul underlined}}, respectively.}
\label{tab_long_term_appd_p2}
\end{center}
\end{table*}

\begin{table*}[t]
\begin{center}
{\fontsize{7.5}{9}\selectfont
\begin{tabular}{@{}c@{\hspace{3pt}}c@{\hspace{2pt}}|@{\hspace{2pt}}
c@{\hspace{5pt}}c@{\hspace{2pt}}|@{\hspace{2pt}}
c@{\hspace{5pt}}c@{\hspace{2pt}}|@{\hspace{2pt}}
c@{\hspace{5pt}}c@{\hspace{2pt}}|@{\hspace{2pt}}
c@{\hspace{5pt}}c@{\hspace{2pt}}|@{\hspace{2pt}}
c@{\hspace{5pt}}c@{\hspace{2pt}}|@{\hspace{2pt}}
c@{\hspace{5pt}}c@{\hspace{2pt}}|@{\hspace{2pt}}
c@{\hspace{5pt}}c@{\hspace{2pt}}|@{\hspace{2pt}}
c@{\hspace{5pt}}c@{\hspace{2pt}}|@{\hspace{2pt}}
c@{\hspace{5pt}}c@{\hspace{2pt}}|@{\hspace{2pt}}
c@{\hspace{5pt}}c@{}}
\toprule
\multicolumn{2}{c|@{\hspace{2pt}}}{Model}          & \multicolumn{2}{c|@{\hspace{2pt}}}{\begin{tabular}[c]{@{}c@{}}FreEformer\\ (Ours)\end{tabular}} & 
\multicolumn{2}{c|@{\hspace{2pt}}}{\begin{tabular}[c]{@{}c@{}}Leddam\\ \shortcite{Leddam_icml} \end{tabular}}   & 
\multicolumn{2}{c|@{\hspace{2pt}}}{\begin{tabular}[c]{@{}c@{}}CARD\\ \shortcite{card}\end{tabular}}        & 
\multicolumn{2}{c|@{\hspace{2pt}}}{\begin{tabular}[c]{@{}c@{}}Fredformer\\ \shortcite{fredformer}\end{tabular}}  & 
\multicolumn{2}{c|@{\hspace{2pt}}}{\begin{tabular}[c]{@{}c@{}}iTrans.\\ \shortcite{itransformer}\end{tabular}}  & 
\multicolumn{2}{c|@{\hspace{2pt}}}{\begin{tabular}[c]{@{}c@{}}TimeMixer\\ \shortcite{timemixer}\end{tabular}} & 
\multicolumn{2}{c|@{\hspace{2pt}}}{\begin{tabular}[c]{@{}c@{}}PatchTST\\ \shortcite{patchtst}\end{tabular}}    &  
\multicolumn{2}{c|@{\hspace{2pt}}}{\begin{tabular}[c]{@{}c@{}}TimesNet\\ \shortcite{timesnet} \end{tabular}} & 
\multicolumn{2}{c|@{\hspace{2pt}}}{\begin{tabular}[c]{@{}c@{}}DLinear\\ \shortcite{linear} \end{tabular}} & 
\multicolumn{2}{c}{\begin{tabular}[c]{@{}c@{}}FreTS\\ \shortcite{frets} \end{tabular}} 
 \\ \midrule
\multicolumn{2}{c|@{\hspace{2pt}}}{Metric}          & MSE                                    & MAE                                    & MSE                                   & MAE                                & MSE                                   & MAE                                   & MSE                                   & MAE                                   & MSE                                   & MAE                                & MSE                                    & MAE                                & MSE                                   & MAE                                   & MSE                                  & MAE                                 & MSE                                  & MAE                                 & MSE                                & MAE  \\ \midrule

 & 3   & {\color[HTML]{FF0000} \textbf{0.508}} & {\color[HTML]{FF0000} \textbf{0.366}} & 0.551                                  & {\color[HTML]{0000FF} {\ul 0.388}}    & 0.597                              & 0.392                              & {\color[HTML]{0000FF} {\ul 0.528}}    & 0.403                                 & 0.555                               & 0.395                              & 0.659                                 & 0.435                                 & 0.646                                 & 0.417                                 & 0.627  & 0.420                              & 1.280                              & 0.747 & 0.647                                 & 0.416 \\
                           & 6   & {\color[HTML]{FF0000} \textbf{0.898}} & {\color[HTML]{FF0000} \textbf{0.522}} & {\color[HTML]{0000FF} {\ul 1.021}}     & {\color[HTML]{0000FF} {\ul 0.569}}    & 1.246                              & 0.610                              & 1.128                                 & 0.604                                 & 1.124                               & 0.586                              & 1.306                                 & 0.643                                 & 1.269                                 & 0.629                                 & 1.147  & 0.610                              & 2.054                              & 0.967 & 1.503                                 & 0.720 \\
                           & 9   & {\color[HTML]{FF0000} \textbf{1.277}} & {\color[HTML]{FF0000} \textbf{0.648}} & 1.881                                  & 0.795                                 & 2.041                              & 0.829                              & 1.804                                 & 0.786                                 & {\color[HTML]{0000FF} {\ul 1.794}}  & {\color[HTML]{0000FF} {\ul 0.772}} & 2.070                                 & 0.842                                 & 2.021                                 & 0.830                                 & 1.796  & 0.785                              & 2.771                              & 1.138 & 2.250                                 & 0.935 \\
                           & 12  & {\color[HTML]{FF0000} \textbf{1.876}} & {\color[HTML]{FF0000} \textbf{0.805}} & 2.421                                  & 0.964                                 & 2.746                              & 0.996                              & 2.610                                 & 0.989                                 & {\color[HTML]{0000FF} {\ul 2.273}}  & {\color[HTML]{0000FF} {\ul 0.884}} & 2.792                                 & 1.018                                 & 2.788                                 & 1.016                                 & 2.349  & 0.920                              & 3.497                              & 1.284 & 2.956                                 & 1.058 \\
                           & Avg & {\color[HTML]{FF0000} \textbf{1.140}} & {\color[HTML]{FF0000} \textbf{0.585}} & 1.468                                  & 0.679                                 & 1.658                              & 0.707                              & 1.518                                 & 0.696                                 & {\color[HTML]{0000FF} {\ul 1.437}}  & {\color[HTML]{0000FF} {\ul 0.659}} & 1.707                                 & 0.734                                 & 1.681                                 & 0.723                                 & 1.480  & 0.684                              & 2.400                              & 1.034 & 1.839                                 & 0.782 \\ \cmidrule(l){2-22} 
                           & 24  & {\color[HTML]{FF0000} \textbf{1.980}} & {\color[HTML]{FF0000} \textbf{0.849}} & 2.085                                  & 0.883                                 & 2.407                              & 0.970                              & 2.098                                 & 0.894                                 & {\color[HTML]{0000FF} {\ul 2.004}}  & {\color[HTML]{0000FF} {\ul 0.860}} & 2.110                                 & 0.879                                 & 2.046                                 & {\color[HTML]{FF0000} \textbf{0.849}} & 2.317  & 0.934                              & 3.158                              & 1.243 & 2.804                                 & 1.107 \\
                           & 36  & {\color[HTML]{0000FF} {\ul 1.879}}    & {\color[HTML]{FF0000} \textbf{0.823}} & 2.017                                  & 0.892                                 & 2.324                              & 0.948                              & {\color[HTML]{FF0000} \textbf{1.712}} & {\color[HTML]{0000FF} {\ul 0.867}}    & 1.910                               & 0.880                              & 2.084                                 & 0.890                                 & 2.344                                 & 0.912                                 & 1.972  & 0.920                              & 3.009                              & 1.200 & 2.881                                 & 1.136 \\
                           & 48  & {\color[HTML]{FF0000} \textbf{1.825}} & {\color[HTML]{FF0000} \textbf{0.815}} & {\color[HTML]{0000FF} {\ul 1.860}}     & {\color[HTML]{0000FF} {\ul 0.847}}    & 2.133                              & 0.911                              & 2.054                                 & 0.922                                 & 2.036                               & 0.891                              & 1.961                                 & 0.866                                 & 2.123                                 & 0.883                                 & 2.238  & 0.940                              & 2.994                              & 1.194 & 3.074                                 & 1.195 \\
                           & 60  & 1.940                                 & {\color[HTML]{FF0000} \textbf{0.852}} & 1.967                                  & 0.879                                 & 2.177                              & 0.921                              & {\color[HTML]{FF0000} \textbf{1.925}} & 0.913                                 & 2.022                               & 0.919                              & {\color[HTML]{0000FF} {\ul 1.926}}    & {\color[HTML]{0000FF} {\ul 0.878}}    & 2.001                                 & 0.895                                 & 2.027  & 0.928                              & 3.172                              & 1.232 & 3.384                                 & 1.258 \\
\multirow{-10}{*}{\rotatebox[origin=c]{90}{ILI}}     & Avg & {\color[HTML]{FF0000} \textbf{1.906}} & {\color[HTML]{FF0000} \textbf{0.835}} & 1.982                                  & {\color[HTML]{0000FF} {\ul 0.875}}    & 2.260                              & 0.938                              & {\color[HTML]{0000FF} {\ul 1.947}}    & 0.899                                 & 1.993                               & 0.887                              & 2.020                                 & 0.878                                 & 2.128                                 & 0.885                                 & 2.139  & 0.931                              & 3.083                              & 1.217 & 3.036                                 & 1.174 \\ \midrule
                           & 3   & {\color[HTML]{FF0000} \textbf{1.094}} & {\color[HTML]{FF0000} \textbf{0.498}} & 1.216                                  & 0.570                                 & {\color[HTML]{0000FF} {\ul 1.103}} & {\color[HTML]{0000FF} {\ul 0.521}} & 1.165                                 & 0.548                                 & 1.193                               & 0.561                              & 1.237                                 & 0.547                                 & 1.220                                 & 0.573                                 & 2.021  & 0.704                              & 2.386                              & 0.909 & 1.196                                 & 0.552 \\
                           & 6   & {\color[HTML]{0000FF} {\ul 1.726}}    & {\color[HTML]{FF0000} \textbf{0.624}} & 1.782                                  & 0.689                                 & 1.919                              & 0.735                              & {\color[HTML]{FF0000} \textbf{1.465}} & {\color[HTML]{0000FF} {\ul 0.685}}    & 1.933                               & 0.755                              & 2.003                                 & 0.739                                 & 1.982                                 & 0.762                                 & 2.405  & 0.808                              & 3.220                              & 1.053 & 2.221                                 & 0.794 \\
                           & 9   & {\color[HTML]{0000FF} {\ul 2.238}}    & {\color[HTML]{FF0000} \textbf{0.741}} & 2.407                                  & 0.866                                 & 2.358                              & {\color[HTML]{0000FF} {\ul 0.841}} & {\color[HTML]{FF0000} \textbf{2.145}} & 0.845                                 & 2.441                               & 0.879                              & 2.594                                 & 0.860                                 & 2.633                                 & 0.916                                 & 2.858  & 0.969                              & 3.803                              & 1.160 & 3.193                                 & 1.018 \\
                           & 12  & {\color[HTML]{FF0000} \textbf{2.509}} & {\color[HTML]{FF0000} \textbf{0.828}} & 2.851                                  & 0.991                                 & 2.857                              & 0.971                              & 2.833                                 & 0.984                                 & {\color[HTML]{0000FF} {\ul 2.819}}  & 0.984                              & 3.103                                 & 0.981                                 & 3.050                                 & 1.030                                 & 2.993  & {\color[HTML]{0000FF} {\ul 0.964}} & 4.524                              & 1.288 & 3.455                                 & 1.086 \\
                           & Avg & {\color[HTML]{FF0000} \textbf{1.892}} & {\color[HTML]{FF0000} \textbf{0.673}} & 2.064                                  & 0.779                                 & 2.059                              & 0.767                              & {\color[HTML]{0000FF} {\ul 1.902}}    & {\color[HTML]{0000FF} {\ul 0.765}}    & 2.096                               & 0.795                              & 2.234                                 & 0.782                                 & 2.221                                 & 0.820                                 & 2.569  & 0.861                              & 3.483                              & 1.102 & 2.516                                 & 0.862 \\ \cmidrule(l){2-22} 
                           & 24  & {\color[HTML]{FF0000} \textbf{4.452}} & {\color[HTML]{FF0000} \textbf{1.230}} & 4.860                                  & 1.342                                 & 5.133                              & 1.394                              & 4.799                                 & 1.347                                 & {\color[HTML]{0000FF} {\ul 4.715}}  & {\color[HTML]{0000FF} {\ul 1.321}} & 6.335                                 & 1.554                                 & 5.528                                 & 1.450                                 & 5.634  & 1.442                              & 9.780                              & 1.851 & 6.587                                 & 1.528 \\
                           & 36  & {\color[HTML]{FF0000} \textbf{6.844}} & {\color[HTML]{FF0000} \textbf{1.621}} & 7.378                                  & 1.708                                 & 7.377                              & 1.725                              & 7.536                                 & 1.727                                 & {\color[HTML]{0000FF} {\ul 7.299}}  & {\color[HTML]{0000FF} {\ul 1.681}} & 8.222                                 & 1.787                                 & 8.351                                 & 1.830                                 & 9.114  & 1.848                              & 12.804                             & 2.083 & 10.431                                & 1.913 \\
                           & 48  & 10.209                                & 2.009                                 & {\color[HTML]{0000FF} {\ul 10.051}}    & {\color[HTML]{0000FF} {\ul 1.999}}    & 11.013                             & 2.103                              & {\color[HTML]{FF0000} \textbf{9.833}} & {\color[HTML]{FF0000} \textbf{1.951}} & 10.141                              & 2.012                              & 11.669                                & 2.157                                 & 11.259                                & 2.114                                 & 10.940 & 2.033                              & 14.244                             & 2.189 & 13.354                                & 2.135 \\
                           & 60  & 12.235                                & 2.196                                 & {\color[HTML]{FF0000} \textbf{11.467}} & {\color[HTML]{FF0000} \textbf{2.119}} & 12.528                             & 2.227                              & 12.455                                & 2.209                                 & {\color[HTML]{0000FF} {\ul 11.871}} & {\color[HTML]{0000FF} {\ul 2.156}} & 12.188                                & 2.173                                 & 12.666                                & 2.225                                 & 12.888 & 2.186                              & 15.472                             & 2.275 & 15.008                                & 2.255 \\
\multirow{-10}{*}{\rotatebox[origin=c]{90}{COVID-19}}   & Avg & {\color[HTML]{FF0000} \textbf{8.435}} & {\color[HTML]{FF0000} \textbf{1.764}} & {\color[HTML]{0000FF} {\ul 8.439}}     & {\color[HTML]{0000FF} {\ul 1.792}}    & 9.013                              & 1.862                              & 8.656                                 & 1.808                                 & 8.506                               & {\color[HTML]{0000FF} {\ul 1.792}} & 9.604                                 & 1.918                                 & 9.451                                 & 1.905                                 & 9.644  & 1.877                              & 13.075                             & 2.099 & 11.345                                & 1.958 \\ \midrule
                           & 3   & 0.208                                 & {\color[HTML]{FF0000} \textbf{0.172}} & {\color[HTML]{0000FF} {\ul 0.204}}     & 0.191                                 & 0.210                              & {\color[HTML]{0000FF} {\ul 0.180}} & 0.205                                 & 0.188                                 & 0.205                               & 0.188                              & 0.205                                 & 0.192                                 & {\color[HTML]{0000FF} {\ul 0.204}}    & 0.190                                 & 0.221  & 0.204                              & 0.218                              & 0.231 & {\color[HTML]{FF0000} \textbf{0.202}} & 0.211 \\
                           & 6   & 0.301                                 & {\color[HTML]{FF0000} \textbf{0.208}} & {\color[HTML]{0000FF} {\ul 0.293}}     & 0.227                                 & 0.311                              & {\color[HTML]{0000FF} {\ul 0.219}} & 0.298                                 & 0.227                                 & 0.300                               & 0.229                              & 0.297                                 & 0.230                                 & 0.298                                 & 0.227                                 & 0.308  & 0.238                              & 0.307                              & 0.278 & {\color[HTML]{FF0000} \textbf{0.292}} & 0.261 \\
                           & 9   & 0.383                                 & {\color[HTML]{FF0000} \textbf{0.239}} & {\color[HTML]{FF0000} \textbf{0.369}}  & 0.264                                 & 0.401                              & {\color[HTML]{0000FF} {\ul 0.253}} & 0.385                                 & 0.263                                 & 0.386                               & 0.265                              & 0.381                                 & 0.264                                 & 0.382                                 & 0.263                                 & 0.387  & 0.273                              & 0.386                              & 0.316 & {\color[HTML]{0000FF} {\ul 0.371}}    & 0.305 \\
                           & 12  & 0.453                                 & {\color[HTML]{FF0000} \textbf{0.264}} & {\color[HTML]{0000FF} {\ul 0.442}}     & 0.292                                 & 0.474                              & {\color[HTML]{0000FF} {\ul 0.281}} & 0.457                                 & 0.292                                 & 0.460                               & 0.295                              & 0.455                                 & 0.295                                 & 0.456                                 & 0.292                                 & 0.462  & 0.298                              & 0.452                              & 0.353 & {\color[HTML]{FF0000} \textbf{0.431}} & 0.340 \\
                           & Avg & 0.336                                 & {\color[HTML]{FF0000} \textbf{0.221}} & {\color[HTML]{0000FF} {\ul 0.327}}     & 0.243                                 & 0.349                              & {\color[HTML]{0000FF} {\ul 0.233}} & 0.336                                 & 0.242                                 & 0.338                               & 0.244                              & 0.334                                 & 0.245                                 & 0.335                                 & 0.243                                 & 0.344  & 0.253                              & 0.341                              & 0.294 & {\color[HTML]{FF0000} \textbf{0.324}} & 0.279 \\ \cmidrule(l){2-22} 
                           & 24  & 0.648                                 & {\color[HTML]{FF0000} \textbf{0.339}} & 0.680                                  & 0.405                                 & 0.700                              & {\color[HTML]{0000FF} {\ul 0.378}} & 0.676                                 & 0.408                                 & 0.700                               & 0.413                              & 0.671                                 & 0.413                                 & 0.679                                 & 0.410                                 & 0.698  & 0.415                              & {\color[HTML]{0000FF} {\ul 0.645}} & 0.458 & {\color[HTML]{FF0000} \textbf{0.637}} & 0.448 \\
                           & 36  & 0.792                                 & {\color[HTML]{FF0000} \textbf{0.389}} & 0.841                                  & 0.471                                 & 0.874                              & {\color[HTML]{0000FF} {\ul 0.448}} & 0.852                                 & 0.477                                 & 0.867                               & 0.480                              & 0.841                                 & 0.480                                 & 0.845                                 & 0.484                                 & 0.856  & 0.475                              & {\color[HTML]{0000FF} {\ul 0.785}} & 0.533 & {\color[HTML]{FF0000} \textbf{0.769}} & 0.520 \\
                           & 48  & 0.910                                 & {\color[HTML]{FF0000} \textbf{0.431}} & 0.963                                  & 0.528                                 & 1.017                              & {\color[HTML]{0000FF} {\ul 0.498}} & 0.982                                 & 0.526                                 & 1.017                               & 0.539                              & 0.964                                 & 0.531                                 & 0.972                                 & 0.536                                 & 0.972  & 0.518                              & {\color[HTML]{0000FF} {\ul 0.885}} & 0.585 & {\color[HTML]{FF0000} \textbf{0.872}} & 0.582 \\
                           & 60  & 1.010                                 & {\color[HTML]{FF0000} \textbf{0.464}} & 1.029                                  & 0.556                                 & 1.126                              & {\color[HTML]{0000FF} {\ul 0.541}} & 1.084                                 & 0.569                                 & 1.079                               & 0.572                              & 1.047                                 & 0.573                                 & 1.077                                 & 0.578                                 & 1.033  & 0.543                              & {\color[HTML]{0000FF} {\ul 0.959}} & 0.623 & {\color[HTML]{FF0000} \textbf{0.938}} & 0.624 \\
\multirow{-10}{*}{\rotatebox[origin=c]{90}{METR-LA}} & Avg & 0.840                                 & {\color[HTML]{FF0000} \textbf{0.406}} & 0.878                                  & 0.490                                 & 0.929                              & {\color[HTML]{0000FF} {\ul 0.466}} & 0.898                                 & 0.495                                 & 0.916                               & 0.501                              & 0.881                                 & 0.499                                 & 0.893                                 & 0.502                                 & 0.890  & 0.488                              & {\color[HTML]{0000FF} {\ul 0.819}} & 0.550 & {\color[HTML]{FF0000} \textbf{0.804}} & 0.543 \\ \midrule
                           & 3   & {\color[HTML]{FF0000} \textbf{0.035}} & {\color[HTML]{FF0000} \textbf{0.092}} & 0.040                                  & 0.103                                 & {\color[HTML]{0000FF} {\ul 0.037}} & 0.096                              & 0.039                                 & 0.102                                 & 0.040                               & 0.105                              & {\color[HTML]{FF0000} \textbf{0.035}} & {\color[HTML]{0000FF} {\ul 0.093}}    & 0.038                                 & 0.099                                 & 0.049  & 0.123                              & 0.044                              & 0.123 & 0.045                                 & 0.129 \\
                           & 6   & {\color[HTML]{FF0000} \textbf{0.049}} & {\color[HTML]{FF0000} \textbf{0.118}} & 0.054                                  & 0.128                                 & {\color[HTML]{0000FF} {\ul 0.052}} & {\color[HTML]{0000FF} {\ul 0.123}} & 0.053                                 & 0.126                                 & 0.054                               & 0.129                              & {\color[HTML]{FF0000} \textbf{0.049}} & {\color[HTML]{FF0000} \textbf{0.118}} & 0.053                                 & 0.124                                 & 0.061  & 0.142                              & 0.062                              & 0.155 & 0.076                                 & 0.180 \\
                           & 9   & {\color[HTML]{FF0000} \textbf{0.062}} & {\color[HTML]{FF0000} \textbf{0.137}} & 0.066                                  & 0.147                                 & {\color[HTML]{0000FF} {\ul 0.063}} & 0.141                              & 0.067                                 & 0.147                                 & 0.068                               & 0.150                              & {\color[HTML]{FF0000} \textbf{0.062}} & {\color[HTML]{0000FF} {\ul 0.139}}    & 0.065                                 & 0.145                                 & 0.073  & 0.161                              & 0.082                              & 0.189 & 0.103                                 & 0.216 \\
                           & 12  & {\color[HTML]{0000FF} {\ul 0.075}}    & {\color[HTML]{FF0000} \textbf{0.156}} & 0.078                                  & 0.164                                 & {\color[HTML]{0000FF} {\ul 0.075}} & {\color[HTML]{0000FF} {\ul 0.158}} & 0.079                                 & 0.165                                 & 0.078                               & 0.165                              & {\color[HTML]{FF0000} \textbf{0.073}} & {\color[HTML]{FF0000} \textbf{0.156}} & 0.077                                 & 0.161                                 & 0.088  & 0.179                              & 0.100                              & 0.215 & 0.098                                 & 0.213 \\
                           & Avg & {\color[HTML]{FF0000} \textbf{0.055}} & {\color[HTML]{FF0000} \textbf{0.126}} & 0.059                                  & 0.135                                 & {\color[HTML]{0000FF} {\ul 0.057}} & {\color[HTML]{0000FF} {\ul 0.130}} & 0.059                                 & 0.135                                 & 0.060                               & 0.137                              & {\color[HTML]{FF0000} \textbf{0.055}} & {\color[HTML]{FF0000} \textbf{0.126}} & 0.058                                 & 0.132                                 & 0.068  & 0.151                              & 0.072                              & 0.170 & 0.080                                 & 0.184 \\ \cmidrule(l){2-22} 
                           & 24  & {\color[HTML]{FF0000} \textbf{0.119}} & {\color[HTML]{FF0000} \textbf{0.216}} & 0.125                                  & 0.222                                 & 0.124                              & {\color[HTML]{0000FF} {\ul 0.220}} & 0.128                                 & 0.226                                 & 0.137                               & 0.237                              & {\color[HTML]{0000FF} {\ul 0.122}}    & 0.221                                 & 0.127                                 & 0.224                                 & 0.198  & 0.299                              & 0.155                              & 0.274 & 0.173                                 & 0.294 \\
                           & 36  & {\color[HTML]{FF0000} \textbf{0.166}} & {\color[HTML]{FF0000} \textbf{0.263}} & 0.174                                  & 0.271                                 & {\color[HTML]{0000FF} {\ul 0.167}} & {\color[HTML]{0000FF} {\ul 0.266}} & 0.170                                 & 0.268                                 & 0.184                               & 0.280                              & 0.183                                 & 0.279                                 & 0.174                                 & 0.269                                 & 0.229  & 0.326                              & 0.196                              & 0.306 & 0.200                                 & 0.309 \\
                           & 48  & {\color[HTML]{0000FF} {\ul 0.204}}    & {\color[HTML]{FF0000} \textbf{0.298}} & 0.222                                  & 0.312                                 & 0.218                              & 0.307                              & 0.218                                 & {\color[HTML]{0000FF} {\ul 0.306}}    & 0.229                               & 0.318                              & {\color[HTML]{FF0000} \textbf{0.200}} & {\color[HTML]{FF0000} \textbf{0.298}} & 0.225                                 & 0.314                                 & 0.267  & 0.352                              & 0.244                              & 0.344 & 0.302                                 & 0.393 \\
                           & 60  & {\color[HTML]{0000FF} {\ul 0.250}}    & {\color[HTML]{0000FF} {\ul 0.331}}    & 0.264                                  & 0.341                                 & 0.264                              & 0.341                              & 0.262                                 & 0.339                                 & 0.279                               & 0.352                              & {\color[HTML]{FF0000} \textbf{0.238}} & {\color[HTML]{FF0000} \textbf{0.328}} & 0.265                                 & 0.339                                 & 0.327  & 0.394                              & 0.318                              & 0.401 & 0.380                                 & 0.449 \\
\multirow{-10}{*}{\rotatebox[origin=c]{90}{NASDAQ}}  & Avg & {\color[HTML]{FF0000} \textbf{0.185}} & {\color[HTML]{FF0000} \textbf{0.277}} & 0.196                                  & 0.286                                 & 0.193                              & 0.284                              & 0.194                                 & 0.285                                 & 0.207                               & 0.297                              & {\color[HTML]{0000FF} {\ul 0.186}}    & {\color[HTML]{0000FF} {\ul 0.281}}    & 0.198                                 & 0.286                                 & 0.255  & 0.343                              & 0.228                              & 0.331 & 0.263                                 & 0.361 \\ \midrule
                           & 3   & 6.153                                 & {\color[HTML]{FF0000} \textbf{0.371}} & {\color[HTML]{0000FF} {\ul 6.148}}     & 0.383                                 & 6.183                              & {\color[HTML]{0000FF} {\ul 0.378}} & 6.190                                 & 0.387                                 & 6.237                               & 0.393                              & 6.209                                 & 0.392                                 & {\color[HTML]{FF0000} \textbf{6.112}} & 0.380                                 & 7.597  & 0.510                              & 6.254                              & 0.438 & 6.227                                 & 0.424 \\
                           & 6   & 6.443                                 & {\color[HTML]{FF0000} \textbf{0.386}} & 6.455                                  & 0.397                                 & 6.465                              & {\color[HTML]{0000FF} {\ul 0.393}} & 6.696                                 & 0.404                                 & 6.484                               & 0.400                              & 6.475                                 & 0.402                                 & {\color[HTML]{FF0000} \textbf{6.425}} & 0.395                                 & 7.962  & 0.515                              & 6.579                              & 0.467 & {\color[HTML]{0000FF} {\ul 6.426}}    & 0.441 \\
                           & 9   & {\color[HTML]{FF0000} \textbf{6.659}} & {\color[HTML]{FF0000} \textbf{0.399}} & 6.687                                  & 0.412                                 & 6.714                              & 0.415                              & 6.768                                 & {\color[HTML]{0000FF} {\ul 0.411}}    & 6.689                               & {\color[HTML]{0000FF} {\ul 0.411}} & 6.702                                 & 0.418                                 & 6.743                                 & 0.426                                 & 8.150  & 0.524                              & 6.776                              & 0.508 & {\color[HTML]{0000FF} {\ul 6.660}}    & 0.458 \\
                           & 12  & 6.842                                 & {\color[HTML]{FF0000} \textbf{0.409}} & 6.899                                  & 0.424                                 & 6.852                              & 0.415                              & 7.168                                 & 0.424                                 & 6.868                               & 0.419                              & 6.902                                 & 0.426                                 & {\color[HTML]{0000FF} {\ul 6.814}}    & {\color[HTML]{0000FF} {\ul 0.414}}    & 8.117  & 0.533                              & 6.927                              & 0.513 & {\color[HTML]{FF0000} \textbf{6.772}} & 0.468 \\
                           & Avg & 6.524                                 & {\color[HTML]{FF0000} \textbf{0.391}} & 6.547                                  & 0.404                                 & 6.553                              & {\color[HTML]{0000FF} {\ul 0.400}} & 6.705                                 & 0.406                                 & 6.569                               & 0.405                              & 6.572                                 & 0.409                                 & {\color[HTML]{0000FF} {\ul 6.523}}    & 0.404                                 & 7.956  & 0.520                              & 6.634                              & 0.481 & {\color[HTML]{FF0000} \textbf{6.521}} & 0.448 \\ \cmidrule(l){2-22} 
                           & 24  & 6.887                                 & {\color[HTML]{FF0000} \textbf{0.424}} & 6.919                                  & 0.450                                 & 6.925                              & 0.440                              & {\color[HTML]{FF0000} \textbf{6.531}} & 0.432                                 & 6.886                               & 0.437                              & 6.900                                 & 0.446                                 & 6.858                                 & {\color[HTML]{0000FF} {\ul 0.430}}    & 8.023  & 0.612                              & 6.883                              & 0.520 & {\color[HTML]{0000FF} {\ul 6.809}}    & 0.484 \\
                           & 36  & 6.436                                 & {\color[HTML]{FF0000} \textbf{0.440}} & 6.456                                  & 0.457                                 & 6.463                              & 0.451                              & {\color[HTML]{FF0000} \textbf{5.935}} & 0.453                                 & 6.431                               & 0.452                              & 6.520                                 & 0.467                                 & 6.400                                 & {\color[HTML]{0000FF} {\ul 0.445}}    & 7.229  & 0.595                              & 6.393                              & 0.538 & {\color[HTML]{0000FF} {\ul 6.340}}    & 0.497 \\
                           & 48  & 6.012                                 & {\color[HTML]{FF0000} \textbf{0.449}} & 6.031                                  & 0.468                                 & 6.031                              & {\color[HTML]{0000FF} {\ul 0.460}} & {\color[HTML]{FF0000} \textbf{5.871}} & 0.464                                 & 6.101                               & 0.483                              & 6.108                                 & 0.484                                 & 5.959                                 & {\color[HTML]{FF0000} \textbf{0.449}} & 7.184  & 0.641                              & 5.940                              & 0.547 & {\color[HTML]{0000FF} {\ul 5.897}}    & 0.519 \\
                           & 60  & 5.699                                 & {\color[HTML]{0000FF} {\ul 0.456}}    & 5.740                                  & 0.478                                 & 5.723                              & 0.463                              & {\color[HTML]{FF0000} \textbf{5.389}} & 0.463                                 & 5.681                               & 0.462                              & 5.732                                 & 0.476                                 & 5.633                                 & {\color[HTML]{FF0000} \textbf{0.452}} & 6.805  & 0.645                              & 5.605                              & 0.552 & {\color[HTML]{0000FF} {\ul 5.543}}    & 0.520 \\
\multirow{-10}{*}{\rotatebox[origin=c]{90}{Wiki}}    & Avg & 6.259                                 & {\color[HTML]{FF0000} \textbf{0.442}} & 6.286                                  & 0.463                                 & 6.285                              & 0.453                              & {\color[HTML]{FF0000} \textbf{5.931}} & 0.453                                 & 6.275                               & 0.458                              & 6.315                                 & 0.468                                 & 6.212                                 & {\color[HTML]{0000FF} {\ul 0.444}}    & 7.310  & 0.623                              & 6.205                              & 0.539 & {\color[HTML]{0000FF} {\ul 6.147}}    & 0.505 \\ \midrule

\multicolumn{2}{c|@{\hspace{2pt}}}{1st Count}    & 22 & 46 & 2&1   & 0 & 0 & 10 & 1 & 0 & 0 & 7 & 5 & 2 & 3 & 0 & 0 & 0 & 0 & 11 & 0                                 \\ \bottomrule

\end{tabular}
}
\caption{Full results of short-term multivariate time series forecasting. Two settings are used for comparison: `Input-12-Predict-\{3,6,9,12\}' and `Input-36-Predict-\{24,36,48,60\}'. The best and second-best results are highlighted in {\color[HTML]{FF0000} \textbf{bold}} and {\color[HTML]{0000FF} {\ul underlined}}, respectively.}
\label{tab_short_term_appd}
\end{center}
\end{table*}

Tables \ref{tab_long_term_appd_p1} and \ref{tab_long_term_appd_p2} present the comprehensive results for long-term multivariate time series forecasting, serving as the full version of Table 1 in the main paper. Specifically, FreEformer consistently outperforms state-of-the-art models across all prediction lengths for the PEMS datasets, demonstrating its robust performance advantage.

Table \ref{tab_short_term_appd} provides the comprehensive results for short-term multivariate time series forecasting under two lookback-prediction settings: `Input-12-Predict-\{3,6,9,12\}' and `Input-36-Predict-\{24,36,48,60\}'. Despite the limited sizes of the ILI, COVID-19, NASDAQ and Wiki datasets, FreEformer consistently outperforms state-of-the-art models. Interestingly, the linear and frequency-based model FreTS shows competitive performance, likely due to the strong inductive bias of frequency modeling.

\subsection{Enhanced Attention Generality}

\begin{table*}[t]
\begin{center}
{\fontsize{8}{9}\selectfont
\setlength{\tabcolsep}{5.5pt}
\begin{tabular}{@{}cc|
cccc|
cccc|
cccc|
cccc@{}}
\toprule
\multicolumn{2}{c|}{\multirow{2}{*}{Model}} & \multicolumn{4}{c|}{iTrans.}                                         & \multicolumn{4}{c|}{PatchTST}                                        & \multicolumn{4}{c|}{Leddam}                                          & \multicolumn{4}{c}{Fredformer}                                      \\ \cmidrule(l){3-18} 
\multicolumn{2}{c|}{}                        & \multicolumn{2}{c|}{Van. Attn.}    & \multicolumn{2}{c|}{Enh. Attn.} & \multicolumn{2}{c|}{Van. Attn.}    & \multicolumn{2}{c|}{Enh. Attn.} & \multicolumn{2}{c|}{Van. Attn.}    & \multicolumn{2}{c|}{Enh. Attn.} & \multicolumn{2}{c|}{Van. Attn.}    & \multicolumn{2}{c}{Enh. Attn.} \\ \midrule
\multicolumn{2}{c|}{Metric}                  & MSE   & \multicolumn{1}{c|}{MAE}   & MSE            & MAE            & MSE   & \multicolumn{1}{c|}{MAE}   & MSE            & MAE            & MSE   & \multicolumn{1}{c|}{MAE}   & MSE            & MAE            & MSE   & \multicolumn{1}{c|}{MAE}   & MSE            & MAE           \\ \midrule
\multirow{5}{*}{\rotatebox[origin=c]{90}{ETTm1}}          & 96         & 0.334 & \multicolumn{1}{c|}{0.368} & 0.322          & 0.360          & 0.329 & \multicolumn{1}{c|}{0.367} & 0.322          & 0.360          & 0.319 & \multicolumn{1}{c|}{0.359} & 0.319          & 0.359          & 0.326 & \multicolumn{1}{c|}{0.361} & 0.324          & 0.360         \\
                                & 192        & 0.377 & \multicolumn{1}{c|}{0.391} & 0.365          & 0.383          & 0.367 & \multicolumn{1}{c|}{0.385} & 0.362          & 0.384          & 0.369 & \multicolumn{1}{c|}{0.383} & 0.367          & 0.382          & 0.363 & \multicolumn{1}{c|}{0.380} & 0.364          & 0.381         \\
                                & 336        & 0.426 & \multicolumn{1}{c|}{0.420} & 0.399          & 0.405          & 0.399 & \multicolumn{1}{c|}{0.410} & 0.390          & 0.407          & 0.394 & \multicolumn{1}{c|}{0.402} & 0.392          & 0.400          & 0.395 & \multicolumn{1}{c|}{0.403} & 0.394          & 0.404         \\
                                & 720        & 0.491 & \multicolumn{1}{c|}{0.459} & 0.469          & 0.444          & 0.454 & \multicolumn{1}{c|}{0.439} & 0.449          & 0.437          & 0.460 & \multicolumn{1}{c|}{0.442} & 0.459          & 0.440          & 0.453 & \multicolumn{1}{c|}{0.438} & 0.456          & 0.440         \\
                                & Avg        & 0.407 & \multicolumn{1}{c|}{0.410} & 0.389          & 0.398          & 0.211 & \multicolumn{1}{c|}{0.303} & 0.381          & 0.397          & 0.386 & \multicolumn{1}{c|}{0.397} & 0.384          & 0.395          & 0.384 & \multicolumn{1}{c|}{0.396} & 0.385          & 0.396         \\ \midrule
\multirow{5}{*}{\rotatebox[origin=c]{90}{Weather}}        & 96         & 0.174 & \multicolumn{1}{c|}{0.214} & 0.161          & 0.207          & 0.177 & \multicolumn{1}{c|}{0.218} & 0.162          & 0.206          & 0.156 & \multicolumn{1}{c|}{0.202} & 0.156          & 0.202          & 0.163 & \multicolumn{1}{c|}{0.207} & 0.155          & 0.202         \\
                                & 192        & 0.221 & \multicolumn{1}{c|}{0.254} & 0.213          & 0.253          & 0.225 & \multicolumn{1}{c|}{0.259} & 0.208          & 0.250          & 0.207 & \multicolumn{1}{c|}{0.250} & 0.205          & 0.248          & 0.211 & \multicolumn{1}{c|}{0.251} & 0.206          & 0.248         \\
                                & 336        & 0.278 & \multicolumn{1}{c|}{0.296} & 0.271          & 0.297          & 0.278 & \multicolumn{1}{c|}{0.297} & 0.265          & 0.291          & 0.262 & \multicolumn{1}{c|}{0.291} & 0.262          & 0.291          & 0.267 & \multicolumn{1}{c|}{0.292} & 0.266          & 0.291         \\
                                & 720        & 0.358 & \multicolumn{1}{c|}{0.349} & 0.353          & 0.350          & 0.354 & \multicolumn{1}{c|}{0.348} & 0.347          & 0.345          & 0.343 & \multicolumn{1}{c|}{0.343} & 0.343          & 0.343          & 0.343 & \multicolumn{1}{c|}{0.341} & 0.340          & 0.341         \\
                                & Avg        & 0.258 & \multicolumn{1}{c|}{0.278} & 0.249          & 0.276          & 0.211 & \multicolumn{1}{c|}{0.303} & 0.245          & 0.273          & 0.242 & \multicolumn{1}{c|}{0.272} & 0.242          & 0.271          & 0.246 & \multicolumn{1}{c|}{0.273} & 0.242          & 0.271         \\ \midrule
\multirow{5}{*}{\rotatebox[origin=c]{90}{ECL}}            & 96         & 0.148 & \multicolumn{1}{c|}{0.240} & 0.137          & 0.232          & 0.161 & \multicolumn{1}{c|}{0.250} & 0.154          & 0.245          & 0.141 & \multicolumn{1}{c|}{0.235} & 0.139          & 0.234          & 0.147 & \multicolumn{1}{c|}{0.241} & 0.139          & 0.236         \\
                                & 192        & 0.162 & \multicolumn{1}{c|}{0.253} & 0.154          & 0.248          & 0.199 & \multicolumn{1}{c|}{0.289} & 0.166          & 0.256          & 0.159 & \multicolumn{1}{c|}{0.252} & 0.158          & 0.252          & 0.165 & \multicolumn{1}{c|}{0.258} & 0.158          & 0.254         \\
                                & 336        & 0.178 & \multicolumn{1}{c|}{0.269} & 0.169          & 0.264          & 0.215 & \multicolumn{1}{c|}{0.305} & 0.183          & 0.273          & 0.173 & \multicolumn{1}{c|}{0.268} & 0.172          & 0.269          & 0.177 & \multicolumn{1}{c|}{0.273} & 0.175          & 0.272         \\
                                & 720        & 0.225 & \multicolumn{1}{c|}{0.317} & 0.200          & 0.296          & 0.256 & \multicolumn{1}{c|}{0.337} & 0.221          & 0.307          & 0.201 & \multicolumn{1}{c|}{0.295} & 0.198          & 0.294          & 0.213 & \multicolumn{1}{c|}{0.304} & 0.203          & 0.297         \\
                                & Avg        & 0.178 & \multicolumn{1}{c|}{0.270} & 0.165          & 0.260          & 0.208 & \multicolumn{1}{c|}{0.295} & 0.181          & 0.270          & 0.169 & \multicolumn{1}{c|}{0.263} & 0.167          & 0.262          & 0.176 & \multicolumn{1}{c|}{0.269} & 0.169          & 0.265         \\ \midrule
\multirow{5}{*}{\rotatebox[origin=c]{90}{Solar-Energy}}          & 96         & 0.203 & \multicolumn{1}{c|}{0.237} & 0.194          & 0.233          & 0.234 & \multicolumn{1}{c|}{0.286} & 0.201          & 0.244          & 0.197 & \multicolumn{1}{c|}{0.241} & 0.192          & 0.224          & 0.185 & \multicolumn{1}{c|}{0.233} & 0.148          & 0.246         \\
                                & 192        & 0.233 & \multicolumn{1}{c|}{0.261} & 0.225          & 0.258          & 0.267 & \multicolumn{1}{c|}{0.310} & 0.232          & 0.264          & 0.231 & \multicolumn{1}{c|}{0.264} & 0.231          & 0.255          & 0.227 & \multicolumn{1}{c|}{0.253} & 0.237          & 0.273         \\
                                & 336        & 0.248 & \multicolumn{1}{c|}{0.273} & 0.241          & 0.273          & 0.290 & \multicolumn{1}{c|}{0.315} & 0.249          & 0.275          & 0.241 & \multicolumn{1}{c|}{0.268} & 0.245          & 0.264          & 0.246 & \multicolumn{1}{c|}{0.284} & 0.254          & 0.283         \\
                                & 720        & 0.249 & \multicolumn{1}{c|}{0.275} & 0.245          & 0.277          & 0.289 & \multicolumn{1}{c|}{0.317} & 0.248          & 0.276          & 0.250 & \multicolumn{1}{c|}{0.281} & 0.245          & 0.276          & 0.247 & \multicolumn{1}{c|}{0.276} & 0.250          & 0.287         \\
                                & Avg        & 0.233 & \multicolumn{1}{c|}{0.262} & 0.226          & 0.260          & 0.211 & \multicolumn{1}{c|}{0.303} & 0.232          & 0.265          & 0.230 & \multicolumn{1}{c|}{0.264} & 0.228          & 0.255          & 0.226 & \multicolumn{1}{c|}{0.262} & 0.222          & 0.272         \\ \midrule
\multirow{5}{*}{\rotatebox[origin=c]{90}{PEMS03}}         & 12         & 0.071 & \multicolumn{1}{c|}{0.174} & 0.062          & 0.164          & 0.099 & \multicolumn{1}{c|}{0.216} & 0.071          & 0.175          & 0.063 & \multicolumn{1}{c|}{0.164} & 0.062          & 0.165          & 0.068 & \multicolumn{1}{c|}{0.174} & 0.063          & 0.164         \\
                                & 24         & 0.093 & \multicolumn{1}{c|}{0.201} & 0.079          & 0.185          & 0.142 & \multicolumn{1}{c|}{0.259} & 0.103          & 0.211          & 0.080 & \multicolumn{1}{c|}{0.185} & 0.078          & 0.182          & 0.094 & \multicolumn{1}{c|}{0.205} & 0.082          & 0.188         \\
                                & 48         & 0.125 & \multicolumn{1}{c|}{0.236} & 0.113          & 0.222          & 0.211 & \multicolumn{1}{c|}{0.319} & 0.168          & 0.266          & 0.124 & \multicolumn{1}{c|}{0.226} & 0.112          & 0.219          & 0.152 & \multicolumn{1}{c|}{0.262} & 0.125          & 0.234         \\
                                & 96         & 0.164 & \multicolumn{1}{c|}{0.275} & 0.154          & 0.266          & 0.269 & \multicolumn{1}{c|}{0.370} & 0.251          & 0.328          & 0.160 & \multicolumn{1}{c|}{0.266} & 0.160          & 0.264          & 0.228 & \multicolumn{1}{c|}{0.330} & 0.165          & 0.276         \\
                                & Avg        & 0.113 & \multicolumn{1}{c|}{0.222} & 0.102          & 0.209          & 0.211 & \multicolumn{1}{c|}{0.303} & 0.148          & 0.245          & 0.107 & \multicolumn{1}{c|}{0.210} & 0.103          & 0.208          & 0.135 & \multicolumn{1}{c|}{0.243} & 0.109          & 0.215         \\ \midrule
\multirow{5}{*}{\rotatebox[origin=c]{90}{PEMS07}}         & 12         & 0.067 & \multicolumn{1}{c|}{0.165} & 0.055          & 0.145          & 0.095 & \multicolumn{1}{c|}{0.207} & 0.065          & 0.161          & 0.055 & \multicolumn{1}{c|}{0.145} & 0.054          & 0.144          & 0.063 & \multicolumn{1}{c|}{0.158} & 0.059          & 0.152         \\
                                & 24         & 0.088 & \multicolumn{1}{c|}{0.190} & 0.069          & 0.161          & 0.150 & \multicolumn{1}{c|}{0.262} & 0.097          & 0.193          & 0.070 & \multicolumn{1}{c|}{0.164} & 0.068          & 0.161          & 0.089 & \multicolumn{1}{c|}{0.192} & 0.079          & 0.175         \\
                                & 48         & 0.110 & \multicolumn{1}{c|}{0.215} & 0.096          & 0.190          & 0.253 & \multicolumn{1}{c|}{0.340} & 0.214          & 0.312          & 0.094 & \multicolumn{1}{c|}{0.192} & 0.087          & 0.183          & 0.136 & \multicolumn{1}{c|}{0.241} & 0.113          & 0.211         \\
                                & 96         & 0.139 & \multicolumn{1}{c|}{0.245} & 0.124          & 0.219          & 0.346 & \multicolumn{1}{c|}{0.404} & 0.250          & 0.311          & 0.117 & \multicolumn{1}{c|}{0.217} & 0.113          & 0.214          & 0.197 & \multicolumn{1}{c|}{0.298} & 0.163          & 0.259         \\
                                & Avg        & 0.101 & \multicolumn{1}{c|}{0.204} & 0.086          & 0.179          & 0.211 & \multicolumn{1}{c|}{0.303} & 0.156          & 0.244          & 0.084 & \multicolumn{1}{c|}{0.180} & 0.080          & 0.175          & 0.121 & \multicolumn{1}{c|}{0.222} & 0.103          & 0.199         \\ \midrule
\multirow{5}{*}{\rotatebox[origin=c]{90}{METR-LA}}           & 3          & 0.205 & \multicolumn{1}{c|}{0.188} & 0.204          & 0.190          & 0.204 & \multicolumn{1}{c|}{0.190} & 0.204          & 0.190          & 0.204 & \multicolumn{1}{c|}{0.191} & 0.208          & 0.193          & 0.205 & \multicolumn{1}{c|}{0.188} & 0.204          & 0.188         \\
                                & 6          & 0.300 & \multicolumn{1}{c|}{0.229} & 0.294          & 0.227          & 0.298 & \multicolumn{1}{c|}{0.227} & 0.298          & 0.229          & 0.293 & \multicolumn{1}{c|}{0.227} & 0.294          & 0.229          & 0.298 & \multicolumn{1}{c|}{0.227} & 0.294          & 0.227         \\
                                & 9          & 0.386 & \multicolumn{1}{c|}{0.265} & 0.374          & 0.261          & 0.382 & \multicolumn{1}{c|}{0.263} & 0.382          & 0.263          & 0.369 & \multicolumn{1}{c|}{0.264} & 0.360          & 0.259          & 0.385 & \multicolumn{1}{c|}{0.263} & 0.385          & 0.261         \\
                                & 12         & 0.460 & \multicolumn{1}{c|}{0.295} & 0.444          & 0.292          & 0.456 & \multicolumn{1}{c|}{0.292} & 0.456          & 0.293          & 0.442 & \multicolumn{1}{c|}{0.292} & 0.420          & 0.290          & 0.457 & \multicolumn{1}{c|}{0.292} & 0.453          & 0.290         \\
                                & Avg        & 0.338 & \multicolumn{1}{c|}{0.244} & 0.329          & 0.243          & 0.211 & \multicolumn{1}{c|}{0.303} & 0.335          & 0.244          & 0.327 & \multicolumn{1}{c|}{0.243} & 0.321          & 0.243          & 0.336 & \multicolumn{1}{c|}{0.242} & 0.334          & 0.242         \\ \bottomrule
\end{tabular}
}
\caption{Applying enhanced attention to Transformer-based forecasters. Consistent performance improvements are observed when replacing vanilla attention with the enhanced version, leaving other components unchanged. The lookback length is 12 for METR-LA and 96 for all other datasets.}

\label{tab_attn_genearal_appd}
\end{center}
\end{table*}

To evaluate the generality of the enhanced attention mechanism, we integrate it into several state-of-the-art Transformer-based forecasters, including iTransformer, PatchTST, Leddam, and Fredformer. Only the attention mechanism is updated, while other components and the loss function remain unchanged. Notably, Leddam and Fredformer are more sophisticated than iTransformer and PatchTST, and the updated modules in Leddam and Fredformer constitute only a small portion of their overall architectures. For instance, only the `channel-wise self-attention' module is updated in Leddam.

Table \ref{tab_attn_genearal_appd} shows that replacing vanilla attention with the enhanced version consistently improves performance. On the PEMS07 dataset, the enhanced attention achieves average MSE improvements of 14.9\% for iTransformer, 25.9\% for PatchTST, 4.1\% for Leddam, and 14.6\% for Fredformer. This demonstrates the effectiveness and general applicability of the enhanced attention. Table \ref{tab_attn_genearal_appd} provides the full results corresponding to Table 9 in the main paper.

Figures \ref{fig_itrans_attn}, \ref{fig_patchtst_attn}, \ref{fig_leddam_attn}, and \ref{fig_fredformer_attn} show the attention matrices of iTransformer, PatchTST, Leddam, and Fredformer, respectively. After incorporating the learnable matrix, the final attention matrix demonstrates higher rank and more prominent values.

\begin{figure*}[htbp]
   \centering
   \includegraphics[width=0.8\linewidth]{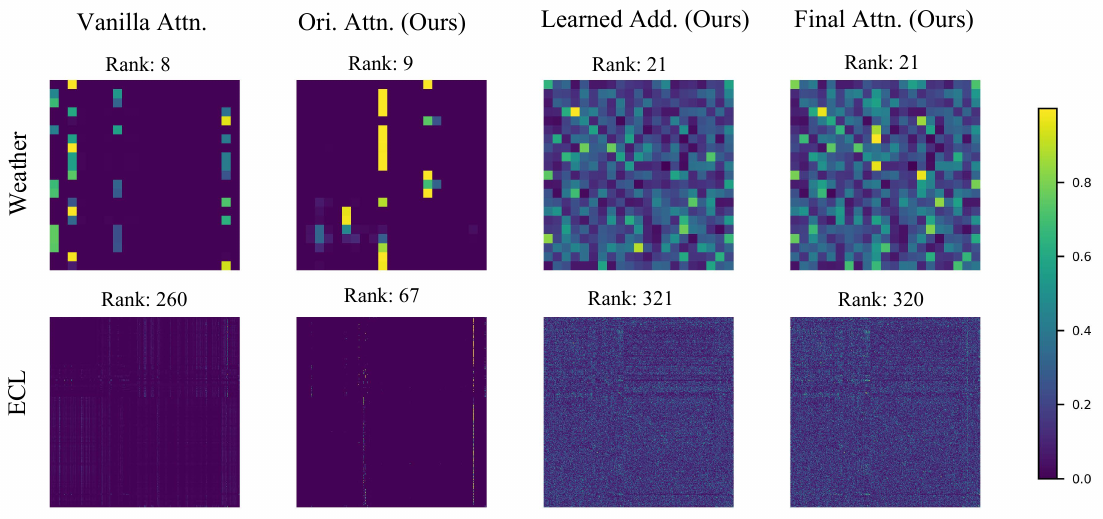}
   \caption{Attention matrices of \textbf{iTransformer} from vanilla and enhanced attention. The left column shows the low-rank attention matrix from the vanilla self-attention mechanism (Weather: 8, ECL: 260), with most entries near zero. The right three columns show the original attention matrix ($\mathbf{A} $), the learned addition matrix ($\mathrm{Softplus} (\mathbf{B}) $), and the final attention matrix ($\mathrm{Norm} \left ( \mathbf{A}+\mathrm{Softplus} (\mathbf{B}) \right ) $). The final attention matrix, after incorporating the learned matrix, exhibits more prominent values and higher ranks (Weather: 21, ECL: 320).}
   \label{fig_itrans_attn}
\end{figure*}

\begin{figure*}[htbp]
   \centering
   \includegraphics[width=0.8\linewidth]{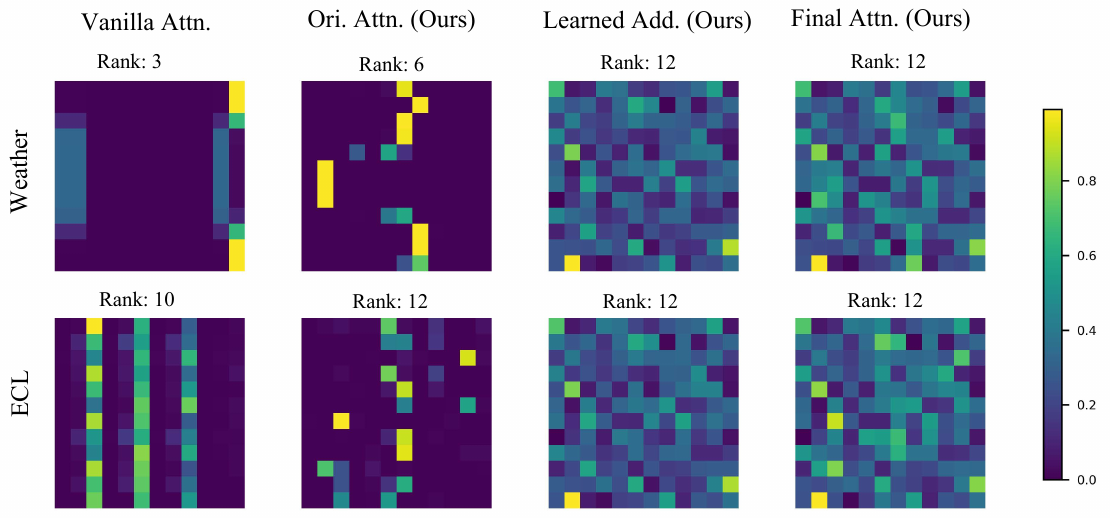}
   \caption{Attention matrices of \textbf{PatchTST} from vanilla and enhanced attention. The left column shows the low-rank attention matrix from the vanilla self-attention mechanism (Weather: 3, ECL: 10), with most entries near zero. The right three columns show the original attention matrix ($\mathbf{A} $), the learned addition matrix ($\mathrm{Softplus} (\mathbf{B}) $), and the final attention matrix ($\mathrm{Norm} \left ( \mathbf{A}+\mathrm{Softplus} (\mathbf{B}) \right ) $). The final attention matrix, after incorporating the learned matrix, exhibits more prominent values and higher ranks (Weather: 12, ECL: 12).}
   \label{fig_patchtst_attn}
\end{figure*}

\begin{figure*}[htbp]
   \centering
   \includegraphics[width=0.8\linewidth]{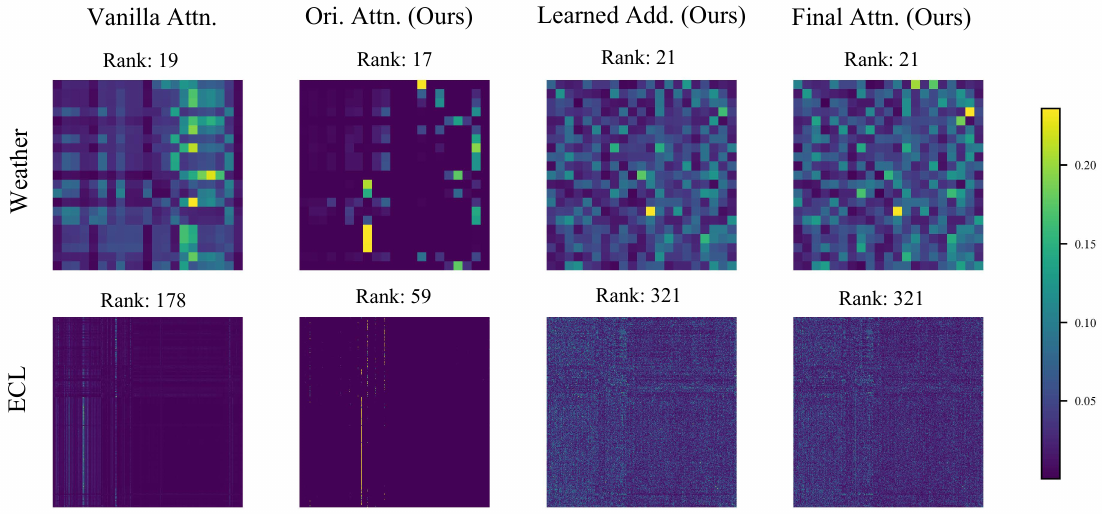}
   \caption{Attention matrices of \textbf{Leddam} from vanilla and enhanced attention. The left column shows the low-rank attention matrix from the vanilla self-attention mechanism (Weather: 19, ECL: 178), with most entries near zero. The right three columns show the original attention matrix ($\mathbf{A} $), the learned addition matrix ($\mathrm{Softplus} (\mathbf{B}) $), and the final attention matrix ($\mathrm{Norm} \left ( \mathbf{A}+\mathrm{Softplus} (\mathbf{B}) \right ) $). The final attention matrix, after incorporating the learned matrix, exhibits more prominent values and higher ranks (Weather: 21, ECL: 321).}
   \label{fig_leddam_attn}
\end{figure*}

\begin{figure*}[htbp]
   \centering
   \includegraphics[width=0.8\linewidth]{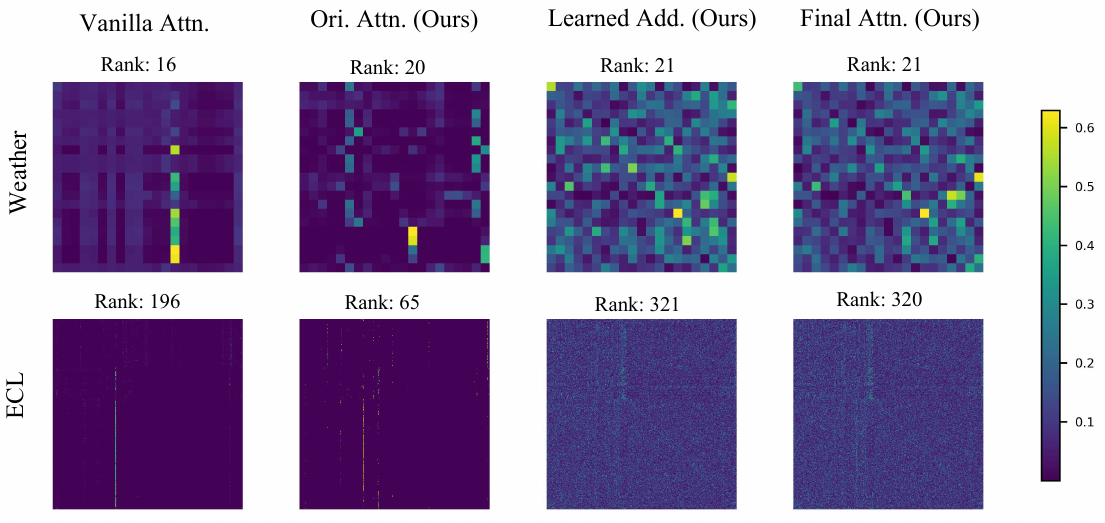}
   \caption{Attention matrices of \textbf{Fredformer} from vanilla and enhanced attention. The left column shows the low-rank attention matrix from the vanilla self-attention mechanism (Weather: 16, ECL: 196), with most entries near zero. The right three columns show the original attention matrix ($\mathbf{A} $), the learned addition matrix ($\mathrm{Softplus} (\mathbf{B}) $), and the final attention matrix ($\mathrm{Norm} \left ( \mathbf{A}+\mathrm{Softplus} (\mathbf{B}) \right ) $). The final attention matrix, after incorporating the learned matrix, exhibits more prominent values and higher ranks (Weather: 21, ECL: 320).}
   \label{fig_fredformer_attn}
\end{figure*}

\subsection{Full Results of Architecture Ablations}

\begin{table*}[t]
\begin{center}
{\fontsize{8}{9}\selectfont
\setlength{\tabcolsep}{8.5pt}
\begin{tabular}{ccc|c|cc|cc|cc|cc|cc}
\toprule
                         &                           &                          &                       & \multicolumn{2}{c|}{ETTm1}                                                    & \multicolumn{2}{c|}{Weather}                                                  & \multicolumn{2}{c|}{ECL}                                                      & \multicolumn{2}{c|}{Traffic}                                                  & \multicolumn{2}{c}{COVID-19}                                                  \\ \cmidrule(l){5-14} 
\multirow{-2}{*}{Layer}  & \multirow{-2}{*}{Dim.}    & \multirow{-2}{*}{Patch.} & \multirow{-2}{*}{$\tau$} & MSE                                   & MAE                                   & MSE                                   & MAE                                   & MSE                                   & MAE                                   & MSE                                   & MAE                                   & MSE                                   & MAE                                   \\ \midrule
                         &                           &                          & 96 / 3                   & 0.312                                 & 0.343                                 & 0.155                                 & 0.194                                 & 0.165                                 & 0.258                                 & 0.463                                 & 0.298                                 & 1.152                                 & 0.510                                 \\
                         &                           &                          & 192 / 6                   & 0.364                                 & 0.370                                 & 0.209                                 & 0.242                                 & 0.173                                 & 0.265                                 & 0.473                                 & 0.304                                 & 1.815                                 & 0.661                                 \\
                         &                           &                          & 336 / 9                   & 0.397                                 & {\color[HTML]{FF0000} \textbf{0.391}} & 0.268                                 & 0.287                                 & 0.190                                 & 0.283                                 & 0.488                                 & 0.310                                 & 2.374                                 & 0.798                                 \\
                         &                           &                          & 720 / 12                   & 0.466                                 & 0.429                                 & 0.347                                 & 0.338                                 & 0.230                                 & 0.315                                 & 0.527                                 & 0.331                                 & 2.819                                 & 0.927                                 \\ \cmidrule(l){4-14} 
\multirow{-5}{*}{Linear} & \multirow{-5}{*}{Variate} & \multirow{-5}{*}{\ding{55}}     & Avg                   & 0.385                                 & 0.383                                 & 0.245                                 & 0.265                                 & 0.189                                 & 0.280                                 & 0.488                                 & 0.311                                 & 2.040                                 & 0.724                                 \\ \midrule
                         &                           &                          & 96 / 3                    & 0.316                                 & 0.347                                 & 0.163                                 & 0.199                                 & 0.157                                 & 0.241                                 & 0.457                                 & 0.272                                 & 1.112                                 & 0.510                                 \\
                         &                           &                          & 192 / 6                   & 0.366                                 & 0.372                                 & 0.208                                 & 0.241                                 & 0.168                                 & 0.251                                 & 0.467                                 & 0.276                                 & 1.830                                 & 0.686                                 \\
                         &                           &                          & 336 / 9                   & 0.398                                 & 0.393                                 & 0.266                                 & 0.284                                 & 0.186                                 & 0.269                                 & 0.484                                 & 0.283                                 & 2.443                                 & 0.841                                 \\
                         &                           &                          & 720 / 12                   & 0.463                                 & 0.430                                 & 0.345                                 & 0.336                                 & 0.225                                 & 0.302                                 & 0.519                                 & 0.301                                 & 2.960                                 & 0.979                                 \\ \cmidrule(l){4-14} 
\multirow{-5}{*}{Linear} & \multirow{-5}{*}{Fre.}    & \multirow{-5}{*}{\ding{55}}     & Avg                   & 0.386                                 & 0.385                                 & 0.246                                 & 0.265                                 & 0.184                                 & 0.266                                 & 0.482                                 & 0.283                                 & 2.086                                 & 0.754                                 \\ \midrule
                         &                           &                          & 96 / 3                    & 0.308                                 & {\color[HTML]{FF0000} \textbf{0.340}} & 0.162                                 & 0.197                                 & 0.153                                 & 0.236                                 & 0.474                                 & 0.254                                 & 1.116                                 & 0.516                                 \\
                         &                           &                          & 192 / 6                   & 0.363                                 & 0.370                                 & 0.209                                 & 0.242                                 & 0.167                                 & 0.251                                 & 0.485                                 & 0.261                                 & 1.843                                 & 0.688                                 \\
                         &                           &                          & 336 / 9                   & 0.394                                 & {\color[HTML]{FF0000} \textbf{0.391}} & 0.263                                 & {\color[HTML]{FF0000} \textbf{0.282}} & 0.186                                 & 0.268                                 & 0.507                                 & 0.268                                 & 2.425                                 & 0.828                                 \\
                         &                           &                          & 720 / 12                   & 0.459                                 & {\color[HTML]{FF0000} \textbf{0.428}} & 0.343                                 & 0.336                                 & 0.225                                 & 0.300                                 & 0.550                                 & 0.286                                 & 3.019                                 & 0.997                                 \\ \cmidrule(l){4-14} 
\multirow{-5}{*}{Trans.} & \multirow{-5}{*}{Fre.}    & \multirow{-5}{*}{\ding{52}}    & Avg                   & 0.381                                 & 0.382                                 & 0.244                                 & 0.264                                 & 0.183                                 & 0.264                                 & 0.504                                 & 0.267                                 & 2.100                                 & 0.757                                 \\ \midrule
                         &                           &                          & 96 / 3                    & 0.309                                 & 0.342                                 & 0.162                                 & 0.199                                 & 0.154                                 & 0.239                                 & 0.465                                 & 0.271                                 & 1.118                                 & 0.506                                 \\
                         &                           &                          & 192 / 6                   & 0.366                                 & 0.371                                 & 0.209                                 & 0.243                                 & 0.166                                 & 0.250                                 & 0.473                                 & 0.277                                 & 1.872                                 & 0.702                                 \\
                         &                           &                          & 336 / 9                   & 0.396                                 & 0.393                                 & 0.266                                 & 0.286                                 & 0.183                                 & 0.267                                 & 0.493                                 & 0.283                                 & 2.445                                 & 0.835                                 \\
                         &                           &                          & 720 / 12                   & 0.461                                 & 0.430                                 & 0.342                                 & 0.336                                 & 0.223                                 & 0.301                                 & 0.527                                 & 0.301                                 & 3.028                                 & 0.985                                 \\ \cmidrule(l){4-14} 
\multirow{-5}{*}{Trans.} & \multirow{-5}{*}{Fre.}    & \multirow{-5}{*}{\ding{55}}     & Avg                   & 0.383                                 & 0.384                                 & 0.245                                 & 0.266                                 & 0.181                                 & 0.264                                 & 0.489                                 & 0.283                                 & 2.116                                 & 0.757                                 \\ \midrule
                         &                           &                          & 96 / 3                    & 0.315                                 & 0.349                                 & 0.157                                 & 0.195                                 & 0.134                                 & 0.225                                 & 0.408                                 & 0.241                                 & 1.101                                 & 0.513                                 \\
                         &                           &                          & 192 / 6                   & 0.361                                 & 0.372                                 & 0.204                                 & 0.241                                 & {\color[HTML]{FF0000} \textbf{0.152}} & 0.242                                 & {\color[HTML]{FF0000} \textbf{0.422}} & 0.252                                 & 1.812                                 & 0.692                                 \\
                         &                           &                          & 336 / 9                   & 0.400                                 & 0.399                                 & 0.265                                 & 0.287                                 & 0.166                                 & 0.257                                 & 0.450                                 & 0.262                                 & 2.356                                 & 0.821                                 \\
                         &                           &                          & 720 / 12                   & 0.463                                 & 0.435                                 & {\color[HTML]{FF0000} \textbf{0.339}} & 0.335                                 & {\color[HTML]{FF0000} \textbf{0.196}} & {\color[HTML]{FF0000} \textbf{0.284}} & 0.491                                 & 0.277                                 & 2.848                                 & 0.960                                 \\ \cmidrule(l){4-14} 
\multirow{-5}{*}{Trans.} & \multirow{-5}{*}{Variate} & \multirow{-5}{*}{\ding{52}}    & Avg                   & 0.385                                 & 0.389                                 & 0.241                                 & 0.264                                 & {\color[HTML]{FF0000} \textbf{0.162}} & 0.252                                 & 0.443                                 & 0.258                                 & 2.029                                 & 0.747                                 \\ \midrule
                         &                           &                          & 96 / 3                    & {\color[HTML]{FF0000} \textbf{0.306}} & {\color[HTML]{FF0000} \textbf{0.340}} & {\color[HTML]{FF0000} \textbf{0.153}} & {\color[HTML]{FF0000} \textbf{0.189}} & {\color[HTML]{FF0000} \textbf{0.133}} & {\color[HTML]{FF0000} \textbf{0.223}} & {\color[HTML]{FF0000} \textbf{0.395}} & {\color[HTML]{FF0000} \textbf{0.233}} & {\color[HTML]{FF0000} \textbf{1.094}} & {\color[HTML]{FF0000} \textbf{0.498}} \\
                         &                           &                          & 192 / 6                   & {\color[HTML]{FF0000} \textbf{0.359}} & {\color[HTML]{FF0000} \textbf{0.367}} & {\color[HTML]{FF0000} \textbf{0.201}} & {\color[HTML]{FF0000} \textbf{0.236}} & {\color[HTML]{FF0000} \textbf{0.152}} & {\color[HTML]{FF0000} \textbf{0.240}} & 0.423                                 & {\color[HTML]{FF0000} \textbf{0.245}} & {\color[HTML]{FF0000} \textbf{1.726}} & {\color[HTML]{FF0000} \textbf{0.624}} \\
                         &                           &                          & 336 / 9                   & {\color[HTML]{FF0000} \textbf{0.392}} & {\color[HTML]{FF0000} \textbf{0.391}} & {\color[HTML]{FF0000} \textbf{0.261}} & {\color[HTML]{FF0000} \textbf{0.282}} & {\color[HTML]{FF0000} \textbf{0.165}} & {\color[HTML]{FF0000} \textbf{0.256}} & {\color[HTML]{FF0000} \textbf{0.443}} & {\color[HTML]{FF0000} \textbf{0.254}} & {\color[HTML]{FF0000} \textbf{2.238}} & {\color[HTML]{FF0000} \textbf{0.741}} \\
                         &                           &                          & 720 / 12                   & {\color[HTML]{FF0000} \textbf{0.458}} & {\color[HTML]{FF0000} \textbf{0.428}} & 0.341                                 & {\color[HTML]{FF0000} \textbf{0.334}} & 0.198                                 & 0.286 & {\color[HTML]{FF0000} \textbf{0.480}} & {\color[HTML]{FF0000} \textbf{0.274}} & {\color[HTML]{FF0000} \textbf{2.509}} & {\color[HTML]{FF0000} \textbf{0.828}} \\ \cmidrule(l){4-14} 
\multirow{-5}{*}{Trans.} & \multirow{-5}{*}{Variate} & \multirow{-5}{*}{\ding{55}}     & Avg                   & {\color[HTML]{FF0000} \textbf{0.379}} & {\color[HTML]{FF0000} \textbf{0.381}} & {\color[HTML]{FF0000} \textbf{0.239}} & {\color[HTML]{FF0000} \textbf{0.260}} & {\color[HTML]{FF0000} \textbf{0.162}} & {\color[HTML]{FF0000} \textbf{0.251}} & {\color[HTML]{FF0000} \textbf{0.435}} & {\color[HTML]{FF0000} \textbf{0.251}} & {\color[HTML]{FF0000} \textbf{1.892}} & {\color[HTML]{FF0000} \textbf{0.673}} \\ \bottomrule
\end{tabular}
}
\caption{Complete results of architecture ablation studies on layers, dimensions, and patching settings. The layers include linear and Transformer layers, and the dimensions refer to frequency and variate dimensions. `Patch.' indicates patching along the frequency dimension. The patch length and stride are 6 and 3 for COVID-19, and 16 and 8 for other datasets. Prediction lengths are $\tau \in {3, 6, 9, 12}$ for COVID-19 and $\tau \in {96, 192, 336, 720}$ for other datasets. The final setting corresponds to the FreEformer configuration. This table is the extended version of Table 5 in the main paper.}
\label{tab_arch_abl_append}
\end{center}
\end{table*}

Table \ref{tab_arch_abl_append} is the full version of Table 5 in the main paper, showing complete results for the architecture ablations of FreEformer. FreEformer consistently outperforms its variants with alternative layer, dimension, and patching settings.

\subsection{Comparison With Other Attention Mechanisms}

\begin{table*}[t]
\begin{center}
{\fontsize{8}{9}\selectfont
\setlength{\tabcolsep}{5.5pt}
\begin{tabular}{@{}cc|cc|cc|cc|cc|cc|cc|cc|cc@{}}
\toprule
\multicolumn{2}{c}{Model}       & \multicolumn{2}{c}{\begin{tabular}[c]{@{}c@{}}Ours\end{tabular}}         & 
\multicolumn{2}{c}{\begin{tabular}[c]{@{}c@{}}Informer\\ \shortcite{informer2021} \end{tabular}}     & 
\multicolumn{2}{c}{\begin{tabular}[c]{@{}c@{}}Flowformer\\ \shortcite{flowformer} \end{tabular}} & 
\multicolumn{2}{c}{\begin{tabular}[c]{@{}c@{}}Flashformer\\ \shortcite{flashattention} \end{tabular}}  & 
\multicolumn{2}{c}{\begin{tabular}[c]{@{}c@{}}FLatten\\ \shortcite{flattentrans} \end{tabular}} & 
\multicolumn{2}{c}{\begin{tabular}[c]{@{}c@{}}Mamba\\ \shortcite{mamba} \end{tabular}}  & 
\multicolumn{2}{c}{\begin{tabular}[c]{@{}c@{}}LASER\\ \shortcite{laser}  \end{tabular}}        & 
\multicolumn{2}{c}{\begin{tabular}[c]{@{}c@{}}Lin.Attn.\\ \shortcite{linear_softmax}  \end{tabular}}    \\ \midrule
\multicolumn{2}{c|}{Metric}      & MSE                                   & MAE                                   & MSE                                   & MAE                                   & MSE                                 & MAE                                   & MSE                                   & MAE                                   & MSE                  & MAE                                               & MSE                                & MAE                                & MSE                                   & MAE                                   & MSE                                   & MAE                                   \\ \midrule
                          & 96  & {\color[HTML]{FF0000} \textbf{0.395}} & {\color[HTML]{FF0000} \textbf{0.233}} & 0.412                                 & 0.235                                 & 0.413                               & {\color[HTML]{0000FF} {\ul 0.234}}    & 0.445                                 & 0.273                                 & 0.409                & 0.235                                             & {\color[HTML]{0000FF} {\ul 0.408}} & 0.252                              & 0.410                                 & 0.235                                 & 0.415                                 & 0.235                                 \\
                          & 192 & {\color[HTML]{0000FF} {\ul 0.423}}    & {\color[HTML]{FF0000} \textbf{0.245}} & 0.441                                 & 0.247                                 & 0.437                               & {\color[HTML]{0000FF} {\ul 0.246}}    & {\color[HTML]{FF0000} \textbf{0.422}} & {\color[HTML]{FF0000} \textbf{0.245}} & 0.440                & 0.247                                             & 0.432                              & 0.264                              & 0.439                                 & 0.247                                 & 0.438                                 & {\color[HTML]{0000FF} {\ul 0.246}}    \\
                          & 336 & {\color[HTML]{FF0000} \textbf{0.443}} & {\color[HTML]{0000FF} {\ul 0.254}}    & 0.455                                 & 0.255                                 & 0.460                               & 0.256                                 & {\color[HTML]{0000FF} {\ul 0.446}}    & {\color[HTML]{FF0000} \textbf{0.252}} & 0.462                & {\color[HTML]{0000FF} {\ul 0.254}}                & 0.450                              & 0.272                              & 0.458                                 & {\color[HTML]{0000FF} {\ul 0.254}}    & 0.458                                 & {\color[HTML]{0000FF} {\ul 0.254}}    \\
                          & 720 & {\color[HTML]{FF0000} \textbf{0.480}} & 0.274                                 & 0.489                                 & 0.274                                 & 0.504                               & 0.275                                 & {\color[HTML]{0000FF} {\ul 0.481}}    & {\color[HTML]{0000FF} {\ul 0.272}}    & 0.504                & {\color[HTML]{0000FF} {\ul 0.272}}                & 0.484                              & 0.291                              & 0.498                                 & {\color[HTML]{FF0000} \textbf{0.271}} & 0.496                                 & {\color[HTML]{0000FF} {\ul 0.272}}    \\
\multirow{-5}{*}{\rotatebox[origin=c]{90}{Traffic}} & Avg & {\color[HTML]{FF0000} \textbf{0.435}} & {\color[HTML]{FF0000} \textbf{0.251}} & 0.449                                 & 0.253                                 & 0.453                               & 0.253                                 & 0.448                                 & 0.260                                 & 0.453                & {\color[HTML]{0000FF} {\ul 0.252}}                & {\color[HTML]{0000FF} {\ul 0.443}} & 0.270                              & 0.451                                 & {\color[HTML]{0000FF} {\ul 0.252}}    & 0.452                                 & {\color[HTML]{0000FF} {\ul 0.252}}    \\ \midrule
                          & 12  & {\color[HTML]{FF0000} \textbf{0.060}} & {\color[HTML]{FF0000} \textbf{0.160}} & {\color[HTML]{0000FF} {\ul 0.062}}    & {\color[HTML]{0000FF} {\ul 0.163}}    & 0.063                               & 0.164                                 & 0.064                                 & 0.165                                 & 0.064                & 0.165                                             & 0.065                              & 0.166                              & 0.064                                 & 0.166                                 & 0.065                                 & 0.166                                 \\
                          & 24  & {\color[HTML]{FF0000} \textbf{0.077}} & {\color[HTML]{FF0000} \textbf{0.181}} & {\color[HTML]{0000FF} {\ul 0.079}}    & {\color[HTML]{0000FF} {\ul 0.184}}    & 0.083                               & 0.189                                 & 0.084                                 & 0.190                                 & 0.085                & 0.190                                             & 0.086                              & 0.192                              & 0.084                                 & 0.190                                 & 0.084                                 & 0.189                                 \\
                          & 48  & {\color[HTML]{FF0000} \textbf{0.112}} & {\color[HTML]{FF0000} \textbf{0.218}} & {\color[HTML]{0000FF} {\ul 0.113}}    & {\color[HTML]{0000FF} {\ul 0.221}}    & 0.132                               & 0.238                                 & 0.130                                 & 0.236                                 & 0.127                & 0.234                                             & 0.127                              & 0.234                              & 0.122                                 & 0.230                                 & 0.130                                 & 0.235                                 \\
                          & 96  & {\color[HTML]{0000FF} {\ul 0.159}}    & {\color[HTML]{0000FF} {\ul 0.265}}    & {\color[HTML]{FF0000} \textbf{0.158}} & {\color[HTML]{FF0000} \textbf{0.263}} & 0.176                               & 0.280                                 & 0.179                                 & 0.283                                 & 0.181                & 0.284                                             & 0.182                              & 0.285                              & 0.174                                 & 0.279                                 & 0.182                                 & 0.284                                 \\
\multirow{-5}{*}{\rotatebox[origin=c]{90}{PEMS03}}  & Avg & {\color[HTML]{FF0000} \textbf{0.102}} & {\color[HTML]{FF0000} \textbf{0.206}} & {\color[HTML]{0000FF} {\ul 0.103}}    & {\color[HTML]{0000FF} {\ul 0.208}}    & 0.113                               & 0.218                                 & 0.114                                 & 0.218                                 & 0.114                & 0.218                                             & 0.115                              & 0.219                              & 0.111                                 & 0.216                                 & 0.115                                 & 0.218                                 \\ \midrule
                          & 96  & {\color[HTML]{0000FF} {\ul 0.153}}    & {\color[HTML]{FF0000} \textbf{0.189}} & {\color[HTML]{FF0000} \textbf{0.152}} & {\color[HTML]{0000FF} {\ul 0.191}}    & 0.156                               & {\color[HTML]{0000FF} {\ul 0.191}}    & 0.163                                 & 0.198                                 & 0.167                & 0.200                                             & 0.160                              & 0.197                              & 0.162                                 & 0.197                                 & 0.162                                 & 0.197                                 \\
                          & 192 & {\color[HTML]{FF0000} \textbf{0.201}} & {\color[HTML]{FF0000} \textbf{0.236}} & {\color[HTML]{0000FF} {\ul 0.204}}    & {\color[HTML]{0000FF} {\ul 0.238}}    & 0.209                               & 0.242                                 & 0.208                                 & 0.241                                 & 0.212                & 0.246                                             & 0.206                              & 0.241                              & 0.210                                 & 0.245                                 & 0.208                                 & 0.242                                 \\
                          & 336 & {\color[HTML]{FF0000} \textbf{0.261}} & {\color[HTML]{FF0000} \textbf{0.282}} & 0.268                                 & 0.285                                 & 0.263                               & 0.284                                 & 0.264                                 & {\color[HTML]{0000FF} {\ul 0.283}}    & 0.267                & 0.286                                             & {\color[HTML]{0000FF} {\ul 0.262}} & {\color[HTML]{0000FF} {\ul 0.283}} & 0.263                                 & {\color[HTML]{FF0000} \textbf{0.282}} & 0.266                                 & 0.285                                 \\
                          & 720 & {\color[HTML]{FF0000} \textbf{0.341}} & {\color[HTML]{FF0000} \textbf{0.334}} & 0.345                                 & 0.338                                 & {\color[HTML]{0000FF} {\ul 0.342}}  & {\color[HTML]{0000FF} {\ul 0.336}}    & 0.345                                 & 0.338                                 & 0.346                & 0.339                                             & 0.343                              & {\color[HTML]{0000FF} {\ul 0.336}} & 0.343                                 & 0.337                                 & 0.344                                 & 0.337                                 \\
\multirow{-5}{*}{\rotatebox[origin=c]{90}{Weather}} & Avg & {\color[HTML]{FF0000} \textbf{0.239}} & {\color[HTML]{FF0000} \textbf{0.260}} & {\color[HTML]{0000FF} {\ul 0.242}}    & {\color[HTML]{0000FF} {\ul 0.263}}    & {\color[HTML]{0000FF} {\ul 0.242}}  & {\color[HTML]{0000FF} {\ul 0.263}}    & 0.245                                 & 0.265                                 & 0.248                & 0.267                                             & 0.243                              & 0.264                              & 0.244                                 & 0.265                                 & 0.245                                 & 0.265                                 \\ \midrule
                          & 96  & {\color[HTML]{FF0000} \textbf{0.180}} & {\color[HTML]{FF0000} \textbf{0.191}} & 0.183                                 & 0.194                                 & 0.186                               & 0.195                                 & {\color[HTML]{0000FF} {\ul 0.182}}    & {\color[HTML]{0000FF} {\ul 0.193}}    & 0.200                & 0.200                                             & 0.189                              & 0.201                              & {\color[HTML]{FF0000} \textbf{0.180}} & 0.194                                 & 0.205                                 & 0.202                                 \\
                          & 192 & {\color[HTML]{FF0000} \textbf{0.213}} & {\color[HTML]{FF0000} \textbf{0.215}} & {\color[HTML]{0000FF} {\ul 0.214}}    & {\color[HTML]{0000FF} {\ul 0.216}}    & 0.221                               & 0.219                                 & 0.221                                 & 0.220                                 & 0.227                & 0.221                                             & 0.223                              & 0.224                              & 0.220                                 & 0.219                                 & 0.228                                 & 0.222                                 \\
                          & 336 & {\color[HTML]{FF0000} \textbf{0.233}} & {\color[HTML]{FF0000} \textbf{0.232}} & 0.237                                 & {\color[HTML]{0000FF} {\ul 0.234}}    & 0.243                               & 0.235                                 & {\color[HTML]{0000FF} {\ul 0.236}}    & 0.235                                 & 0.240                & 0.235                                             & 0.243                              & 0.239                              & 0.236                                 & 0.236                                 & 0.240                                 & 0.236                                 \\
                          & 720 & {\color[HTML]{FF0000} \textbf{0.241}} & {\color[HTML]{FF0000} \textbf{0.237}} & 0.243                                 & {\color[HTML]{0000FF} {\ul 0.239}}    & 0.248                               & 0.241                                 & 0.244                                 & 0.240                                 & 0.248                & 0.243                                             & 0.256                              & 0.245                              & {\color[HTML]{0000FF} {\ul 0.242}}    & 0.241                                 & 0.247                                 & 0.242                                 \\
\multirow{-5}{*}{\rotatebox[origin=c]{90}{Solar-Energy}}   & Avg & {\color[HTML]{FF0000} \textbf{0.217}} & {\color[HTML]{FF0000} \textbf{0.219}} & {\color[HTML]{0000FF} {\ul 0.219}}    & {\color[HTML]{0000FF} {\ul 0.221}}    & 0.224                               & 0.222                                 & 0.221                                 & 0.222                                 & 0.229                & 0.224                                             & 0.228                              & 0.227                              & {\color[HTML]{0000FF} {\ul 0.219}}    & 0.222                                 & 0.230                                 & 0.225                                 \\ \midrule
                          & 3   & {\color[HTML]{0000FF} {\ul 0.508}}    & 0.366                                 & 0.542                                 & {\color[HTML]{0000FF} {\ul 0.357}}    & 0.580                               & 0.394                                 & 0.596                                 & 0.417                                 & 0.614                & 0.385                                             & 0.615                              & 0.389                              & 0.590                                 & 0.402                                 & {\color[HTML]{FF0000} \textbf{0.447}} & {\color[HTML]{FF0000} \textbf{0.354}} \\
                          & 6   & {\color[HTML]{FF0000} \textbf{0.898}} & {\color[HTML]{0000FF} {\ul 0.522}}    & 1.239                                 & 0.540                                 & {\color[HTML]{0000FF} {\ul 0.959}}  & {\color[HTML]{FF0000} \textbf{0.515}} & 1.129                                 & 0.553                                 & 1.181                & 0.565                                             & 1.109                              & 0.567                              & 1.054                                 & 0.540                                 & 1.074                                 & 0.542                                 \\
                          & 9   & {\color[HTML]{FF0000} \textbf{1.277}} & {\color[HTML]{0000FF} {\ul 0.648}}    & {\color[HTML]{0000FF} {\ul 1.360}}    & {\color[HTML]{FF0000} \textbf{0.640}} & 1.519                               & 0.720                                 & 1.777                                 & 0.722                                 & 1.450                & 0.710                                             & 1.858                              & 0.779                              & 1.502                                 & 0.732                                 & 1.778                                 & 0.754                                 \\
                          & 12  & {\color[HTML]{FF0000} \textbf{1.876}} & {\color[HTML]{FF0000} \textbf{0.805}} & 2.850                                 & {\color[HTML]{0000FF} {\ul 0.807}}    & {\color[HTML]{0000FF} {\ul 2.096}}  & 0.816                                 & 2.686                                 & 0.870                                 & 4.125                & 0.949                                             & 2.450                              & 0.887                              & 3.238                                 & 0.860                                 & 2.513                                 & 0.845                                 \\
\multirow{-5}{*}{\rotatebox[origin=c]{90}{ILI}}     & Avg & {\color[HTML]{FF0000} \textbf{1.140}} & {\color[HTML]{FF0000} \textbf{0.585}} & 1.498                                 & {\color[HTML]{0000FF} {\ul 0.586}}    & {\color[HTML]{0000FF} {\ul 1.288}}  & 0.611                                 & 1.547                                 & 0.640                                 & 1.842                & 0.652                                             & 1.508                              & 0.655                              & 1.596                                 & 0.633                                 & 1.453                                 & 0.624                                 \\ \bottomrule
\end{tabular}
}
\caption{Comparison with state-of-the-art attention models using the FreEformer architecture. The best and second-best results are highlighted in {\color[HTML]{FF0000} \textbf{bold}} and {\color[HTML]{0000FF} {\ul underlined}}, respectively. The lookback length is set to 96, except for the ILI dataset, where it is set to 12.}

\label{tab_other_attn_appd}
\end{center}
\end{table*}

We further compare the enhanced attention with state-of-the-art attention models using the FreEformer architecture. As shown in Table \ref{tab_other_attn_appd}, the enhanced attention consistently outperforms other attention paradigms, despite requiring only minimal modifications to the vanilla attention mechanism and being easy to implement. Table \ref{tab_other_attn_appd} is the complete version of Table 8 in the main paper.

\section{More Experiments}

\subsection{Best Performance with Varying Lookback Lengths}

\begin{table*}[t]
\begin{center}
{\fontsize{7}{8}\selectfont
\setlength{\tabcolsep}{5pt}
\begin{tabular}{@{}cc|cc|cc|cc|cc|cc|cc|cc|cc|cc@{}}
\toprule
\multicolumn{2}{c|}{Model}      & \multicolumn{2}{c|}{FreEformer}                                               & \multicolumn{2}{c|}{Leddam}                & \multicolumn{2}{c|}{CARD}                                                     & \multicolumn{2}{c|}{Fredformer}               & \multicolumn{2}{c|}{iTrans.}          & \multicolumn{2}{c|}{TimeMixer}                                                & \multicolumn{2}{c|}{PatchTST}                                           & \multicolumn{2}{c|}{TimesNet} & \multicolumn{2}{c}{DLinear} \\ \midrule
\multicolumn{2}{c|}{Metric}     & MSE                                   & MAE                                   & MSE                                & MAE   & MSE                                   & MAE                                   & MSE                                   & MAE   & MSE                                & MAE   & MSE                                   & MAE                                   & MSE                                & MAE                                & MSE           & MAE           & MSE          & MAE          \\ \midrule
                          & 96  & {\color[HTML]{FF0000} \textbf{0.124}} & {\color[HTML]{FF0000} \textbf{0.214}} & 0.134                              & 0.227 & {\color[HTML]{0000FF} {\ul 0.129}}    & 0.223                                 & {\color[HTML]{0000FF} {\ul 0.129}}    & 0.226 & 0.132                              & 0.227 & {\color[HTML]{0000FF} {\ul 0.129}}    & 0.224                                 & {\color[HTML]{0000FF} {\ul 0.129}} & {\color[HTML]{0000FF} {\ul 0.222}} & 0.168         & 0.272         & 0.140        & 0.237        \\
                          & 192 & {\color[HTML]{0000FF} {\ul 0.142}}    & {\color[HTML]{0000FF} {\ul 0.232}}    & 0.156                              & 0.248 & 0.154                                 & 0.245                                 & 0.148                                 & 0.244 & 0.154                              & 0.251 & {\color[HTML]{FF0000} \textbf{0.140}} & {\color[HTML]{FF0000} \textbf{0.220}} & 0.147                              & 0.240                              & 0.184         & 0.322         & 0.153        & 0.249        \\
                          & 336 & {\color[HTML]{FF0000} \textbf{0.158}} & {\color[HTML]{FF0000} \textbf{0.248}} & 0.166                              & 0.264 & {\color[HTML]{0000FF} {\ul 0.161}}    & 0.257                                 & 0.165                                 & 0.262 & 0.170                              & 0.268 & {\color[HTML]{0000FF} {\ul 0.161}}    & {\color[HTML]{0000FF} {\ul 0.255}}    & 0.163                              & 0.259                              & 0.198         & 0.300         & 0.169        & 0.267        \\
                          & 720 & {\color[HTML]{FF0000} \textbf{0.185}} & {\color[HTML]{FF0000} \textbf{0.274}} & 0.195                              & 0.291 & {\color[HTML]{FF0000} \textbf{0.185}} & {\color[HTML]{0000FF} {\ul 0.278}}    & {\color[HTML]{0000FF} {\ul 0.193}}    & 0.286 & {\color[HTML]{0000FF} {\ul 0.193}} & 0.288 & 0.194                                 & 0.287                                 & 0.197                              & 0.290                              & 0.220         & 0.320         & 0.203        & 0.301        \\
\multirow{-5}{*}{\rotatebox[origin=c]{90}{ECL}}     & Avg & {\color[HTML]{FF0000} \textbf{0.152}} & {\color[HTML]{FF0000} \textbf{0.242}} & 0.163                              & 0.257 & 0.157                                 & 0.251                                 & 0.159                                 & 0.254 & 0.162                              & 0.258 & {\color[HTML]{0000FF} {\ul 0.156}}    & {\color[HTML]{0000FF} {\ul 0.247}}    & 0.159                              & 0.253                              & 0.193         & 0.304         & 0.166        & 0.264        \\ \midrule
                          & 96  & {\color[HTML]{FF0000} \textbf{0.341}} & {\color[HTML]{FF0000} \textbf{0.225}} & 0.366                              & 0.260 & {\color[HTML]{FF0000} \textbf{0.341}} & {\color[HTML]{0000FF} {\ul 0.229}}    & {\color[HTML]{0000FF} {\ul 0.358}}    & 0.257 & 0.359                              & 0.262 & 0.360                                 & 0.249                                 & 0.360                              & 0.249                              & 0.593         & 0.321         & 0.410        & 0.282        \\
                          & 192 & {\color[HTML]{0000FF} {\ul 0.369}}    & {\color[HTML]{FF0000} \textbf{0.234}} & 0.394                              & 0.270 & {\color[HTML]{FF0000} \textbf{0.367}} & {\color[HTML]{0000FF} {\ul 0.243}}    & 0.381                                 & 0.272 & 0.376                              & 0.270 & 0.375                                 & 0.250                                 & 0.379                              & 0.256                              & 0.617         & 0.336         & 0.423        & 0.287        \\
                          & 336 & {\color[HTML]{0000FF} {\ul 0.387}}    & {\color[HTML]{FF0000} \textbf{0.241}} & 0.400                              & 0.283 & 0.388                                 & {\color[HTML]{0000FF} {\ul 0.254}}    & 0.396                                 & 0.277 & 0.393                              & 0.279 & {\color[HTML]{FF0000} \textbf{0.385}} & 0.270                                 & 0.392                              & 0.264                              & 0.629         & 0.336         & 0.436        & 0.296        \\
                          & 720 & {\color[HTML]{FF0000} \textbf{0.424}} & {\color[HTML]{FF0000} \textbf{0.260}} & 0.442                              & 0.297 & {\color[HTML]{0000FF} {\ul 0.427}}    & {\color[HTML]{0000FF} {\ul 0.276}}    & {\color[HTML]{FF0000} \textbf{0.424}} & 0.296 & 0.434                              & 0.293 & 0.430                                 & 0.281                                 & 0.432                              & 0.286                              & 0.640         & 0.350         & 0.466        & 0.315        \\
\multirow{-5}{*}{\rotatebox[origin=c]{90}{Traffic}} & Avg & {\color[HTML]{FF0000} \textbf{0.380}} & {\color[HTML]{FF0000} \textbf{0.240}} & 0.400                              & 0.278 & {\color[HTML]{0000FF} {\ul 0.381}}    & {\color[HTML]{0000FF} {\ul 0.251}}    & 0.390                                 & 0.275 & 0.390                              & 0.276 & 0.388                                 & 0.263                                 & 0.391                              & 0.264                              & 0.620         & 0.336         & 0.434        & 0.295        \\ \midrule
                          & 96  & {\color[HTML]{FF0000} \textbf{0.144}} & {\color[HTML]{FF0000} \textbf{0.185}} & 0.149                              & 0.199 & {\color[HTML]{0000FF} {\ul 0.145}}    & {\color[HTML]{0000FF} {\ul 0.186}}    & 0.150                                 & 0.203 & 0.165                              & 0.214 & 0.147                                 & 0.197                                 & 0.149                              & 0.198                              & 0.172         & 0.220         & 0.176        & 0.237        \\
                          & 192 & 0.190                                 & {\color[HTML]{0000FF} {\ul 0.230}}    & 0.196                              & 0.243 & {\color[HTML]{FF0000} \textbf{0.187}} & {\color[HTML]{FF0000} \textbf{0.227}} & 0.194                                 & 0.246 & 0.208                              & 0.253 & {\color[HTML]{0000FF} {\ul 0.189}}    & 0.239                                 & 0.194                              & 0.241                              & 0.219         & 0.261         & 0.220        & 0.282        \\
                          & 336 & {\color[HTML]{FF0000} \textbf{0.238}} & {\color[HTML]{0000FF} {\ul 0.268}}    & 0.243                              & 0.280 & {\color[HTML]{FF0000} \textbf{0.238}} & {\color[HTML]{FF0000} \textbf{0.258}} & 0.243                                 & 0.284 & 0.257                              & 0.292 & {\color[HTML]{0000FF} {\ul 0.241}}    & 0.280                                 & 0.306                              & 0.282                              & 0.246         & 0.337         & 0.265        & 0.319        \\
                          & 720 & 0.319                                 & {\color[HTML]{0000FF} {\ul 0.325}}    & 0.321                              & 0.334 & {\color[HTML]{FF0000} \textbf{0.308}} & {\color[HTML]{FF0000} \textbf{0.321}} & {\color[HTML]{FF0000} \textbf{0.308}} & 0.333 & 0.331                              & 0.343 & {\color[HTML]{0000FF} {\ul 0.310}}    & 0.330                                 & 0.314                              & 0.334                              & 0.365         & 0.359         & 0.323        & 0.362        \\
\multirow{-5}{*}{\rotatebox[origin=c]{90}{Weather}} & Avg & 0.223                                 & {\color[HTML]{0000FF} {\ul 0.252}}    & 0.227                              & 0.264 & {\color[HTML]{FF0000} \textbf{0.220}} & {\color[HTML]{FF0000} \textbf{0.248}} & 0.224                                 & 0.266 & 0.240                              & 0.275 & {\color[HTML]{0000FF} {\ul 0.222}}    & 0.262                                 & 0.241                              & 0.264                              & 0.251         & 0.294         & 0.246        & 0.300        \\ \midrule
                          & 96  & {\color[HTML]{FF0000} \textbf{0.160}} & {\color[HTML]{FF0000} \textbf{0.196}} & 0.186                              & 0.242 & 0.170                                 & {\color[HTML]{0000FF} {\ul 0.207}}    & 0.187                                 & 0.236 & 0.190                              & 0.241 & {\color[HTML]{0000FF} {\ul 0.167}}    & 0.220                                 & 0.224                              & 0.278                              & 0.219         & 0.314         & 0.289        & 0.377        \\
                          & 192 & {\color[HTML]{0000FF} {\ul 0.190}}    & {\color[HTML]{FF0000} \textbf{0.212}} & 0.208                              & 0.262 & 0.192                                 & {\color[HTML]{0000FF} {\ul 0.219}}    & 0.196                                 & 0.251 & 0.233                              & 0.261 & {\color[HTML]{FF0000} \textbf{0.187}} & 0.249                                 & 0.253                              & 0.298                              & 0.231         & 0.322         & 0.319        & 0.397        \\
                          & 336 & {\color[HTML]{FF0000} \textbf{0.195}} & {\color[HTML]{FF0000} \textbf{0.219}} & 0.218                              & 0.265 & 0.226                                 & {\color[HTML]{0000FF} {\ul 0.233}}    & 0.208                                 & 0.265 & 0.226                              & 0.275 & {\color[HTML]{0000FF} {\ul 0.200}}    & 0.258                                 & 0.273                              & 0.306                              & 0.246         & 0.337         & 0.352        & 0.415        \\
                          & 720 & {\color[HTML]{FF0000} \textbf{0.206}} & {\color[HTML]{FF0000} \textbf{0.224}} & {\color[HTML]{0000FF} {\ul 0.208}} & 0.273 & 0.217                                 & {\color[HTML]{0000FF} {\ul 0.243}}    & 0.209                                 & 0.272 & 0.220                              & 0.282 & 0.215                                 & 0.250                                 & 0.272                              & 0.308                              & 0.280         & 0.363         & 0.356        & 0.412        \\
\multirow{-5}{*}{\rotatebox[origin=c]{90}{Solar-Energy}}   & Avg & {\color[HTML]{FF0000} \textbf{0.188}} & {\color[HTML]{FF0000} \textbf{0.213}} & 0.205                              & 0.261 & 0.201                                 & {\color[HTML]{0000FF} {\ul 0.225}}    & 0.200                                 & 0.256 & 0.217                              & 0.265 & {\color[HTML]{0000FF} {\ul 0.192}}    & 0.244                                 & 0.256                              & 0.298                              & 0.244         & 0.334         & 0.329        & 0.400        \\ \midrule
                          & 96  & {\color[HTML]{FF0000} \textbf{0.282}} & {\color[HTML]{FF0000} \textbf{0.332}} & 0.294                              & 0.347 & 0.288                                 & {\color[HTML]{FF0000} \textbf{0.332}} & {\color[HTML]{0000FF} {\ul 0.284}}    & 0.338 & 0.309                              & 0.357 & 0.291                                 & {\color[HTML]{0000FF} {\ul 0.340}}    & 0.293                              & 0.346                              & 0.338         & 0.375         & 0.299        & 0.343        \\
                          & 192 & {\color[HTML]{FF0000} \textbf{0.323}} & {\color[HTML]{0000FF} {\ul 0.358}}    & 0.336                              & 0.369 & 0.332                                 & {\color[HTML]{FF0000} \textbf{0.357}} & {\color[HTML]{FF0000} \textbf{0.323}} & 0.364 & 0.346                              & 0.383 & {\color[HTML]{0000FF} {\ul 0.327}}    & 0.365                                 & 0.333                              & 0.370                              & 0.374         & 0.387         & 0.335        & 0.365        \\
                          & 336 & {\color[HTML]{FF0000} \textbf{0.358}} & {\color[HTML]{0000FF} {\ul 0.378}}    & 0.364                              & 0.389 & 0.364                                 & {\color[HTML]{FF0000} \textbf{0.376}} & {\color[HTML]{FF0000} \textbf{0.358}} & 0.387 & 0.385                              & 0.410 & {\color[HTML]{0000FF} {\ul 0.360}}    & 0.381                                 & 0.369                              & 0.392                              & 0.410         & 0.411         & 0.369        & 0.386        \\
                          & 720 & {\color[HTML]{FF0000} \textbf{0.405}} & {\color[HTML]{0000FF} {\ul 0.413}}    & 0.421                              & 0.419 & {\color[HTML]{0000FF} {\ul 0.414}}    & {\color[HTML]{FF0000} \textbf{0.407}} & 0.420                                 & 0.417 & 0.440                              & 0.442 & 0.415                                 & 0.417                                 & 0.416                              & 0.420                              & 0.478         & 0.450         & 0.425        & 0.421        \\
\multirow{-5}{*}{\rotatebox[origin=c]{90}{ETTm1}}   & Avg & {\color[HTML]{FF0000} \textbf{0.342}} & {\color[HTML]{0000FF} {\ul 0.370}}    & 0.354                              & 0.381 & 0.350                                 & {\color[HTML]{FF0000} \textbf{0.368}} & {\color[HTML]{0000FF} {\ul 0.346}}    & 0.376 & 0.370                              & 0.398 & 0.348                                 & 0.376                                 & 0.353                              & 0.382                              & 0.400         & 0.406         & 0.357        & 0.379        \\ \bottomrule
\end{tabular}
}
\caption{Best performance using varying lookback lengths. The best and second-best results are highlighted in {\color[HTML]{FF0000} \textbf{bold}} and {\color[HTML]{0000FF} {\ul underlined}}, respectively. For FreEformer, the lookback length is chosen from \{96,192,336,720\}. }

\label{tab_variable_lookback_appd}
\end{center}
\end{table*}

We further evaluate forecasting performance under varying lookback lengths. Unlike the constant lookback length setting in the main paper, we explore optimal performance with lookback lengths selected from \{96, 192, 336, 720\}. As shown in Table \ref{tab_variable_lookback_appd}, FreEformer achieves state-of-the-art performance across various forecasting horizons and datasets under this setting, demonstrating its robust forecasting capabilities.

\subsection{More Metrics}

To provide a more comprehensive evaluation of forecasting models, we introduce the following scale-free metrics: Coefficient of Determination ($R^2$), Pearson Correlation Coefficient ($r$), and Mean Absolute Scaled Error (MASE). These metrics are calculated as follows:

\begin{equation}
\begin{aligned}
  &R^2=\frac{1}{N}\sum_{n=1}^{N}\left ( 1-\frac{\left | \hat{\mathbf{y}} _{n:}-\mathbf{y}_{n:} \right | ^2_2 }{\left | \mathbf{y} _{n:}-\overline{\mathbf{y}_{n:}}  \right | ^2_2}  \right )  , \\
  &r =  \frac{1}{N}\sum_{n=1}^{N}\frac{ {\textstyle \sum_{t=1}^{\tau}}  \left ( \hat{\mathbf{y}} _{n,t}-\overline{\hat{\mathbf{y}} _{n:}}  \right )\left ( \mathbf{y} _{n,t}-\overline{\mathbf{y} _{n:}}  \right )  }
{  \sqrt{ {\textstyle \sum_{t=1}^{\tau}} \left ( \hat{\mathbf{y}} _{n,t}-\overline{\hat{\mathbf{y}} _{n:}}  \right )^2} \sqrt{ {\textstyle \sum_{t=1}^{\tau}} \left ( \mathbf{y} _{n,t}-\overline{\mathbf{y} _{n:}}  \right )^2}     } ,   \\
  &\mathrm{MASE} = \frac{1}{N}\sum_{n=1}^{N}\frac{\frac{1}{\tau} \left \| \hat{\mathbf{y}} _{n:}-\mathbf{y}_{n:} \right \|_1  }{\frac{1}{\tau-1}  {\textstyle \sum_{t=1}^{\tau-1} \left | \mathbf{y} _{n,t}-\mathbf{y}_{t-1,n} \right | } }.
\end{aligned}
\label{more_metrics}
\end{equation} 

\noindent Here, $\mathbf{y} \in \mathbb{R}^{N \times \tau}$ and $\hat{\mathbf{y}} \in \mathbb{R}^{N \times \tau}$ denote the ground truth and predictions, respectively, and $\overline{(\cdot)}$ represents the mean value.

As shown in Table \ref{tab_more_metrics}, FreEformer outperforms state-of-the-art models on the ECL and Traffic datasets under these metrics, demonstrating its robustness.

\begin{table*}[ht]
\begin{center}
{\fontsize{8}{10}\selectfont
\setlength{\tabcolsep}{7pt}
\begin{tabular}{@{}cc|
ccc|
ccc|
ccc|
ccc@{}}
\toprule
\multicolumn{2}{c|}{Model}      & \multicolumn{3}{c|}{FreEformer}                                                                                       & \multicolumn{3}{c|}{Leddam}                                                                                  & \multicolumn{3}{c|}{CARD}                                                       & \multicolumn{3}{c}{iTrans.}                                                     \\ \midrule
\multicolumn{2}{c|}{Metric}      & \multicolumn{1}{c}{$R^2$ ($\uparrow$)}                & \multicolumn{1}{c}{r ($\uparrow$)}                 & \multicolumn{1}{c|}{MASE ($\downarrow$)}              & \multicolumn{1}{c}{$R^2$ ($\uparrow$)}             & \multicolumn{1}{c}{r ($\uparrow$)}              & \multicolumn{1}{c|}{MASE ($\downarrow$)}           & \multicolumn{1}{c}{$R^2$ ($\uparrow$)}             & \multicolumn{1}{c}{r ($\uparrow$)} & \multicolumn{1}{c|}{MASE ($\downarrow$)}           & \multicolumn{1}{c}{$R^2$ ($\uparrow$)}             & \multicolumn{1}{c}{r ($\uparrow$)}              & \multicolumn{1}{c}{MASE ($\downarrow$)} \\ \midrule
                          & 96  & {\color[HTML]{FF0000} \textbf{0.617}} & {\color[HTML]{FF0000} \textbf{0.915}} & {\color[HTML]{FF0000} \textbf{0.913}} & 0.527                              & {\color[HTML]{0000FF} {\ul 0.911}} & {\color[HTML]{0000FF} {\ul 0.964}} & 0.533                              & 0.905                 & 0.975                              & {\color[HTML]{0000FF} {\ul 0.550}} & 0.907                              & 0.990                    \\
                          & 192 & {\color[HTML]{FF0000} \textbf{0.623}} & {\color[HTML]{FF0000} \textbf{0.907}} & {\color[HTML]{FF0000} \textbf{0.990}} & {\color[HTML]{0000FF} {\ul 0.593}} & {\color[HTML]{0000FF} {\ul 0.902}} & 1.041                              & 0.547                              & 0.899                 & {\color[HTML]{0000FF} {\ul 1.028}} & 0.590                              & 0.900                              & 1.057                    \\
                          & 336 & {\color[HTML]{FF0000} \textbf{0.705}} & {\color[HTML]{FF0000} \textbf{0.899}} & {\color[HTML]{FF0000} \textbf{1.074}} & 0.670                              & {\color[HTML]{0000FF} {\ul 0.893}} & 1.134                              & {\color[HTML]{0000FF} {\ul 0.678}} & 0.891                 & {\color[HTML]{0000FF} {\ul 1.113}} & {\color[HTML]{0000FF} {\ul 0.678}} & {\color[HTML]{0000FF} {\ul 0.893}} & 1.133                    \\
                          & 720 & {\color[HTML]{FF0000} \textbf{0.686}} & {\color[HTML]{FF0000} \textbf{0.885}} & {\color[HTML]{FF0000} \textbf{1.202}} & {\color[HTML]{0000FF} {\ul 0.661}} & {\color[HTML]{0000FF} {\ul 0.882}} & {\color[HTML]{0000FF} {\ul 1.251}} & 0.653                              & 0.877                 & 1.273                              & 0.638                              & 0.872                              & 1.315                    \\
\multirow{-5}{*}{\rotatebox[origin=c]{90}{ECL}}     & Avg & {\color[HTML]{FF0000} \textbf{0.658}} & {\color[HTML]{FF0000} \textbf{0.901}} & {\color[HTML]{FF0000} \textbf{1.045}} & {\color[HTML]{0000FF} {\ul 0.613}} & {\color[HTML]{0000FF} {\ul 0.897}} & 1.098                              & 0.603                              & 0.893                 & {\color[HTML]{0000FF} {\ul 1.097}} & 0.614                              & 0.893                              & 1.124                    \\ \midrule
                          & 96  & {\color[HTML]{FF0000} \textbf{0.731}} & {\color[HTML]{FF0000} \textbf{0.895}} & {\color[HTML]{FF0000} \textbf{0.748}} & 0.681                              & 0.885                              & 0.909                              & {\color[HTML]{0000FF} {\ul 0.690}} & 0.882                 & {\color[HTML]{0000FF} {\ul 0.855}} & {\color[HTML]{0000FF} {\ul 0.690}} & {\color[HTML]{0000FF} {\ul 0.889}} & 0.860                    \\
                          & 192 & {\color[HTML]{FF0000} \textbf{0.713}} & {\color[HTML]{FF0000} \textbf{0.881}} & {\color[HTML]{FF0000} \textbf{0.772}} & 0.666                              & 0.865                              & 0.949                              & 0.683                              & 0.869                 & {\color[HTML]{0000FF} {\ul 0.856}} & {\color[HTML]{0000FF} {\ul 0.689}} & {\color[HTML]{0000FF} {\ul 0.877}} & 0.879                    \\
                          & 336 & {\color[HTML]{FF0000} \textbf{0.724}} & {\color[HTML]{FF0000} \textbf{0.872}} & {\color[HTML]{FF0000} \textbf{0.771}} & 0.707                              & {\color[HTML]{0000FF} {\ul 0.867}} & 0.897                              & 0.702                              & 0.862                 & {\color[HTML]{0000FF} {\ul 0.843}} & {\color[HTML]{0000FF} {\ul 0.713}} & {\color[HTML]{0000FF} {\ul 0.867}} & 0.883                    \\
                          & 720 & {\color[HTML]{FF0000} \textbf{0.698}} & {\color[HTML]{FF0000} \textbf{0.857}} & {\color[HTML]{FF0000} \textbf{0.824}} & 0.687                              & 0.853                              & 0.947                              & 0.679                              & 0.848                 & {\color[HTML]{0000FF} {\ul 0.894}} & {\color[HTML]{0000FF} {\ul 0.693}} & {\color[HTML]{0000FF} {\ul 0.857}} & 0.913                    \\
\multirow{-5}{*}{\rotatebox[origin=c]{90}{Traffic}} & Avg & {\color[HTML]{FF0000} \textbf{0.716}} & {\color[HTML]{FF0000} \textbf{0.876}} & {\color[HTML]{FF0000} \textbf{0.779}} & 0.685                              & 0.868                              & 0.926                              & 0.689                              & 0.865                 & {\color[HTML]{0000FF} {\ul 0.862}} & {\color[HTML]{0000FF} {\ul 0.696}} & {\color[HTML]{0000FF} {\ul 0.873}} & 0.884                    \\ \midrule
\multicolumn{2}{c|}{Model}     & \multicolumn{3}{c|}{TimeMixer}                                                                                        & \multicolumn{3}{c|}{DLinear}                                                                                 & \multicolumn{3}{c|}{PatchTST}                                                   & \multicolumn{3}{c}{TimesNet}                                                    \\ \midrule
\multicolumn{2}{c|}{Metric}      & \multicolumn{1}{c}{$R^2$ ($\uparrow$)}                & \multicolumn{1}{c}{r ($\uparrow$)}                 & \multicolumn{1}{c|}{MASE ($\downarrow$)}              & \multicolumn{1}{c}{$R^2$ ($\uparrow$)}             & \multicolumn{1}{c}{r ($\uparrow$)}              & \multicolumn{1}{c|}{MASE ($\downarrow$)}           & \multicolumn{1}{c}{$R^2$ ($\uparrow$)}             & \multicolumn{1}{c}{r ($\uparrow$)} & \multicolumn{1}{c|}{MASE ($\downarrow$)}           & \multicolumn{1}{c}{$R^2$ ($\uparrow$)}             & \multicolumn{1}{c}{r ($\uparrow$)}              & \multicolumn{1}{c}{MASE ($\downarrow$)} \\ \midrule
                          & 96  & 0.456                                 & 0.897                                 & 1.022                                 & 0.120                              & 0.870                              & 1.300                              & 0.426                              & 0.890                 & 1.119                              & 0.405                              & 0.892                              & 1.118                    \\
                          & 192 & 0.575                                 & 0.893                                 & 1.076                                 & 0.215                              & 0.867                              & 1.312                              & 0.509                              & 0.887                 & 1.140                              & 0.503                              & 0.885                              & 1.195                    \\
                          & 336 & 0.640                                 & 0.886                                 & 1.155                                 & 0.247                              & 0.858                              & 1.387                              & 0.604                              & 0.880                 & 1.216                              & 0.592                              & 0.875                              & 1.262                    \\
                          & 720 & 0.634                                 & 0.871                                 & 1.314                                 & 0.148                              & 0.844                              & 1.536                              & 0.605                              & 0.864                 & 1.371                              & 0.603                              & 0.864                              & 1.388                    \\
\multirow{-5}{*}{\rotatebox[origin=c]{90}{ECL}}     & Avg & 0.576                                 & 0.887                                 & 1.142                                 & 0.182                              & 0.860                              & 1.384                              & 0.536                              & 0.880                 & 1.211                              & 0.526                              & 0.879                              & 1.241                    \\ \midrule
                          & 96  & 0.640                                 & 0.888                                 & 0.979                                 & 0.441                              & 0.811                              & 1.346                              & 0.668                              & 0.880                 & 0.903                              & 0.625                              & 0.877                              & 1.023                    \\
                          & 192 & 0.660                                 & 0.862                                 & 0.959                                 & 0.551                              & 0.815                              & 1.189                              & 0.677                              & 0.871                 & 0.887                              & 0.635                              & 0.867                              & 1.009                    \\
                          & 336 & 0.685                                 & 0.852                                 & 0.960                                 & 0.582                              & 0.810                              & 1.167                              & 0.700                              & 0.863                 & 0.874                              & 0.654                              & 0.852                              & 1.079                    \\
                          & 720 & 0.672                                 & 0.842                                 & 0.986                                 & 0.564                              & 0.794                              & 1.224                              & 0.681                              & 0.850                 & 0.925                              & 0.661                              & 0.850                              & 1.047                    \\
\multirow{-5}{*}{\rotatebox[origin=c]{90}{Traffic}} & Avg & 0.664                                 & 0.861                                 & 0.971                                 & 0.535                              & 0.808                              & 1.232                              & 0.682                              & 0.866                 & 0.897                              & 0.644                              & 0.862                              & 1.040                    \\ \bottomrule
\end{tabular}

}
 \caption{Forecasting performance on the $R^2$, Correlation Coefficient ($r$), and MASE metrics. The best and second-best results are highlighted in {\color[HTML]{FF0000} \textbf{bold}} and {\color[HTML]{0000FF} {\ul underlined}}, respectively. The symbols $\downarrow$ and $\uparrow$ indicate that lower and higher values are preferable, respectively. The lookback length is uniformly set to 96.}
\label{tab_more_metrics}
\end{center}
\end{table*}

\subsection{Lookback Length}

\begin{figure}[H]
   \centering
   \includegraphics[width=1.0\linewidth]{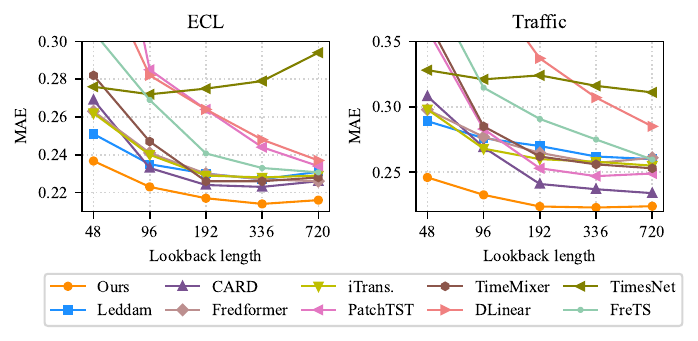}
   \caption{Performance comparison under various lookback lengths, with a fixed prediction horizon of 96.}
   \label{fig_lookback}
\end{figure}

Increasing the lookback length $T$ improves frequency resolution, which is inversely proportional to $T$ (specifically, $1/T$) \cite{signal_and_system}, leading to more refined frequency representations. Figure \ref{fig_lookback} illustrates the MAE changes under different lookback lengths. FreEformer exhibits improved performance with increasing lookback length, consistently outperforming the state-of-the-art forecasters across all lookback lengths on the ECL and Traffic datasets.

\subsection{Limited Training Data} 

\begin{figure}[H]
   \centering
   \includegraphics[width=1.0\linewidth]{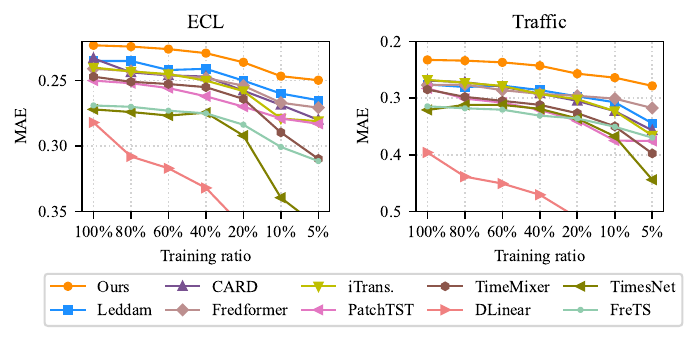}
   \caption{Performance comparison under various training ratios using the Input-96-Predict-96 setting. The y-axis is inverted for improved readability.}
   \label{fig7}
\end{figure}

Evaluating model performance with limited training data highlights a model's learning efficiency and robustness. As illustrated in Figure \ref{fig7}, FreEformer consistently outperforms state-of-the-art forecasters, including frequency-based models (FreTS and Fredformer), across all training ratios, from 100\% to as low as 5\%. This demonstrates the effectiveness of our frequency modeling approach, which explores multivariate correlations in the frequency spectrum, and the strength of the overall model design.

\subsection{Comparison with Large Pre-Trained Models}

\begin{table*}[ht]
\begin{center}
{\fontsize{8}{10}\selectfont
\setlength{\tabcolsep}{5pt}
\begin{tabular}{@{}c|cc|cc|cc|ccccc@{}}
\toprule
Train. Ratio & \multicolumn{2}{c|}{100\%}                                          & \multicolumn{2}{c|}{20\%}                                           & \multicolumn{2}{c|}{5\%}                                   & \multicolumn{5}{c}{0\% (Zero-shot)}                                                                                                                                                                                                                                                    \\ \midrule
Model        & Ours      
& \begin{tabular}[c]{@{}c@{}}Timer\\ \shortcite{timer} \end{tabular} 
& Ours           
& \begin{tabular}[c]{@{}c@{}}Timer\\ \shortcite{timer} \end{tabular} 
& Ours  
& \begin{tabular}[c]{@{}c@{}}Timer\\ \shortcite{timer}\end{tabular} 
& 
\begin{tabular}[c]{@{}c@{}}Timer-28B\\ \shortcite{timer} \end{tabular} 
& 
\begin{tabular}[c]{@{}c@{}}Moirai-L\\ \shortcite{Moirai} \end{tabular} 
& \begin{tabular}[c]{@{}c@{}}Moment\\ \shortcite{moment}\end{tabular} 
& \begin{tabular}[c]{@{}c@{}}TimesFM\\ \shortcite{google_timesfm}\end{tabular} & \begin{tabular}[c]{@{}c@{}}Chronos-S20\\ \shortcite{Chronos} \end{tabular} \\ \midrule
Pre-trained  & No             & UTSD-12G                                           & No             & UTSD-12G                                           & No    & UTSD-12G                                           & UTSD+LOTSA                                             & LOTSA                                                 & TS Pile                                             & N/A                                            & Public Data                                              \\ \midrule
Look-back    & Best           & 672                                                & Best           & 672                                                & Best  & 672                                                & 672                                                    & 512                                                   & 512                                                 & 512                                                  & 512                                                      \\ \midrule
ETTh1        & 0.370          & \textbf{0.358}                                     & 0.402          & \textbf{0.359}                                     & 0.427 & \textbf{0.362}                                     & 0.393                                                  & 0.394                                                 & 0.674                                               & 0.414                                                & 0.454                                                    \\
ETTh2        & 0.284          & --                                                 & 0.294          & \textbf{0.284}                                     & 0.315 & \textbf{0.280}                                     & 0.308                                                  & 0.293                                                 & 0.330                                               & 0.318                                                & 0.326                                                    \\
ETTm1        & 0.282          & --                                                 & \textbf{0.307} & 0.321                                              & 0.320 & \textbf{0.321}                                     & 0.420                                                  & 0.452                                                 & 0.670                                               & 0.354                                                & 0.451                                                    \\
ETTm2        & 0.165          & --                                                 & \textbf{0.168} & 0.187                                              & 0.178 & \textbf{0.176}                                     & 0.247                                                  & 0.214                                                 & 0.257                                               & 0.201                                                & 0.190                                                    \\
ECL          & 0.124          & \textbf{0.136}                                     & \textbf{0.133} & 0.134                                              & 0.145 & \textbf{0.132}                                     & 0.147                                                  & 0.155                                                 & 0.744                                               & -                                                    & -                                                        \\
Traffic      & \textbf{0.341} & 0.351                                              & 0.375          & \textbf{0.352}                                     & 0.425 & \textbf{0.361}                                     & 0.414                                                  & 0.399                                                 & 1.293                                               & -                                                    & -                                                        \\
Weather      & \textbf{0.144} & 0.154                                              & 0.161          & \textbf{0.151}                                     & 0.175 & \textbf{0.151}                                     & 0.243                                                  & 0.221                                                 & 0.255                                               & -                                                    & -                                                        \\
PEMS03       & \textbf{0.060} & 0.118                                              & 0.121          & \textbf{0.116}                                     & 0.232 & \textbf{0.125}                                     & -                                                      & -                                                     & -                                                   & -                                                    & -                                                        \\
PEMS04       & \textbf{0.068} & 0.107                                              & 0.201          & \textbf{0.120}                                     & 0.815 & \textbf{0.135}                                     & -                                                      & -                                                     & -                                                   & -                                                    & -                                                        \\ \bottomrule
\end{tabular}
}
 \caption{Comparison with pre-trained time series models. The better results are highlighted in bold. The MSE is reported and the prediction length is set to 96.}
\label{tab_pretrain}
\end{center}
\end{table*}

We further compare FreEformer's performance with existing pre-trained time series models under fine-tuning scenarios. Surprisingly, as shown in Table \ref{tab_pretrain}, FreEformer outperforms the pre-trained model Timer \cite{timer} on the Traffic, Weather, PEMS03, and PEMS04 datasets, even when Timer is fine-tuned using the full training set. Remarkably, with only 20\% of the training data, FreEformer achieves better performance than Timer on the ETTm1, ETTm2, and ECL datasets. These results provide several key insights: 1) The proposed model demonstrates robust forecasting performance even with limited training data; 2) Current pre-trained time series models still require extensive fine-tuning to achieve optimal downstream performance; 3) The scale and diversity of existing time series corpora for pre-training remain limited, and the development of pre-trained time series models is still in its early stages.

\subsection{Efficiency} 

As shown in Figure \ref{fig8}, FreEformer achieves a favorable trade-off between performance and training efficiency, thanks to its relatively simple frequency-modeling architecture. Detailed comparisons of resource consumption during the training and inference stages between FreEformer and state-of-the-art forecasters are provided in Table \ref{tab_gpu}.

\begin{figure*}[ht]
   \centering
   \includegraphics[width=1.0\linewidth]{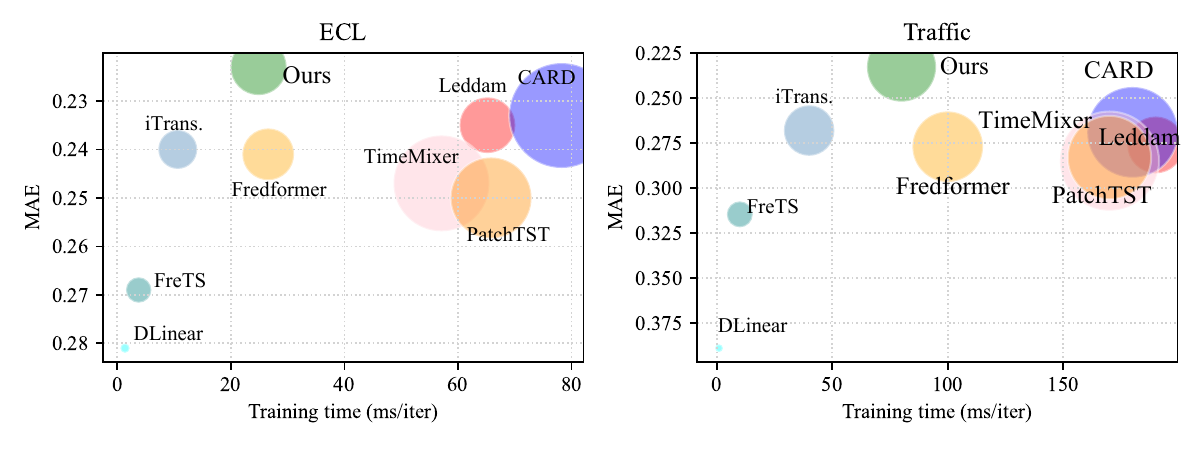}
   \caption{Comparison of model efficiency. Lookback and prediction lengths are both 96, with a fixed batch size of 16. The Y-axis is inverted for clarity. Bubble area is proportional to training GPU memory usage. For improved visual presentation, bubbles within each subfigure are scaled by a uniform factor.}
   \label{fig8}
\end{figure*}

\begin{table*}[ht]
\begin{center}
{\fontsize{8}{10}\selectfont
\setlength{\tabcolsep}{6pt}
\begin{tabular}{@{}cc|cccccccccc@{}}
\toprule
\multicolumn{2}{c|}{Models}                    & FreEformer & Leddam & CARD   & Fredformer & iTrans. & TimeMixer & DLinear & TimesNet & PatchTST & FreTS \\ \midrule
\multirow{6}{*}{\rotatebox[origin=c]{90}{ETTh1}}        & Params(M)      & 10.63      & 8.55   & 0.03   & 8.59       & 4.83    & 0.08      & 0.02    & 299.94   & 3.76     & 0.42  \\
                              & FLOPS(M)       & 61.51      & 111.54 & 1.41   & 135.01     & 33.99   & 10.72     & 0.13    & 289708   & 272.20   & 2.94  \\
                              & T.T. (ms/iter) & 23.33      & 23.54  & 29.10  & 26.48      & 10.11   & 27.99     & 1.42    & 501.98   & 8.37     & 3.63  \\
                              & T.M.(GB)       & 0.26       & 0.18   & 0.03   & 0.24       & 0.21    & 0.05      & 0.02    & 5.80     & 0.14     & 0.03  \\
                              & I.T.(ms/iter)  & 2.92       & 2.62   & 4.05   & 3.73       & 1.72    & 3.86      & 0.21    & 155.97   & 1.38     & 0.49  \\
                              & I.M. (MB)      & 179.80     & 51.04  & 11.02  & 78.04      & 156.00  & 18.06     & 8.48    & 1396.24  & 51.83    & 14.47 \\ \midrule
\multirow{6}{*}{\rotatebox[origin=c]{90}{ECL}}          & Params(M)      & 11.05      & 8.56   & 1.39   & 12.117     & 4.83    & 0.12      & 0.02    & 300.58   & 3.76     & 0.42  \\
                              & FLOPS(G)       & 3.23       & 5.32   & 5.07   & 5.55       & 1.87    & 1.74      & 0.01    & 293.98   & 12.48    & 0.13  \\
                              & T.T. (ms/iter) & 24.89      & 65.19  & 78.25  & 26.587     & 10.66   & 57.09     & 1.37    & 509.17   & 65.83    & 3.78  \\
                              & T.M.(GB)       & 1.46       & 1.44   & 5.08   & 1.24       & 0.70    & 4.24      & 0.04    & 5.82     & 3.02     & 0.30  \\
                              & I.T.(ms/iter)  & 6.42       & 14.78  & 28.21  & 7.08       & 3.46    & 22.79     & 0.20    & 157.13   & 20.52    & 0.48  \\
                              & I.M. (GB)      & 0.38       & 0.49   & 0.63   & 0.26       & 0.23    & 0.82      & 0.02    & 1.41     & 0.93     & 0.23  \\ \midrule
\multirow{6}{*}{\rotatebox[origin=c]{90}{Exchange}}     & Params(M)      & 10.63      & 8.56   & 1.39   & 8.59       & 4.83    & 0.08      & 0.02    & 299.94   & 3.76     & 0.42  \\
                              & FLOPS(M)       & 70.33      & 127.49 & 114.83 & 135.01     & 38.88   & 12.25     & 0.15    & 287909   & 311.08   & 3.36  \\
                              & T.T. (ms/iter) & 23.25      & 24.81  & 35.57  & 28.56      & 10.14   & 27.46     & 1.47    & 505.39   & 9.75     & 3.51  \\
                              & T.M.(GB)       & 0.26       & 0.18   & 0.17   & 0.24       & 0.21    & 0.06      & 0.02    & 5.80     & 0.15     & 0.03  \\
                              & I.T.(ms/iter)  & 2.91       & 2.62   & 6.02   & 6.86       & 1.70    & 3.81      & 0.21    & 156.33   & 2.47     & 0.48  \\
                              & I.M. (GB)      & 0.18       & 0.05   & 0.03   & 0.08       & 0.16    & 0.02      & 0.01    & 1.40     & 0.05     & 0.02  \\ \midrule
\multirow{6}{*}{\rotatebox[origin=c]{90}{Traffic}}      & Params(M)      & 13.61      & 8.56   & 0.98   & 11.09      & 4.83    & 0.12      & 0.02    & 301.69   & 3.76     & 0.42  \\
                              & FLOPS(G)       & 10.59      & 15.24  & 10.47  & 14.78      & 6.45    & 4.66      & 0.02    & 288.72   & 33.52    & 0.36  \\
                              & T.T. (s/iter)  & 0.08       & 0.19   & 0.18   & 0.10       & 0.04    & 0.17      & 0.001   & 0.50     & 0.17     & 0.01  \\
                              & T.M.(GB)       & 5.52       & 3.74   & 9.50   & 5.76       & 2.95    & 11.33     & 0.07    & 5.86     & 8.01     & 0.78  \\
                              & I.T.(ms/iter)  & 25.11      & 45.68  & 61.21  & 31.60      & 13.45   & 67.21     & 0.20    & 15.64    & 55.10    & 3.48  \\
                              & I.M. (GB)      & 1.51       & 1.25   & 1.69   & 1.84       & 1.27    & 2.20      & 0.04    & 1.42     & 2.44     & 0.59  \\ \midrule
\multirow{6}{*}{\rotatebox[origin=c]{90}{Weather}}      & Params(M)      & 10.64      & 8.56   & 0.03   & 0.50       & 4.83    & 0.10      & 0.02    & 299.97   & 3.76     & 0.42  \\
                              & FLOPS(G)       & 0.19       & 0.34   & 0.004  & 0.01       & 0.10    & 0.05      & 0.000   & 291.21   & 0.82     & 0.01  \\
                              & T.T. (ms/iter) & 23.25      & 27.55  & 28.69  & 37.85      & 9.45    & 33.30     & 1.38    & 505.88   & 9.89     & 3.53  \\
                              & T.M.(GB)       & 0.30       & 0.21   & 0.05   & 0.05       & 0.22    & 0.16      & 0.02    & 5.80     & 0.27     & 0.04  \\
                              & I.T.(ms/iter)  & 2.90       & 3.15   & 4.11   & 7.19       & 1.66    & 5.10      & 0.20    & 156.27   & 2.96     & 0.48  \\
                              & I.M. (MB)      & 63.96      & 70.06  & 14.45  & 18.28      & 34.88   & 36.03     & 9.06    & 1397.18  & 90.78    & 24.37 \\ \midrule
\multirow{6}{*}{\rotatebox[origin=c]{90}{Solar-Energy}} & Params(M)      & 10.71      & 8.56   & 1.39   & 4.61       & 4.83    & 0.12      & 0.02    & 300.21   & 3.76     & 0.42  \\
                              & FLOPS(G)       & 1.28       & 2.22   & 2.05   & 1.20       & 0.72    & 0.74      & 0.003   & 290.64   & 5.33     & 0.06  \\
                              & T.T. (ms/iter) & 23.35      & 48.64  & 36.28  & 31.77      & 9.97    & 35.90     & 1.44    & 506.58   & 28.89    & 3.60  \\
                              & T.M.(GB)       & 0.62       & 0.68   & 2.22   & 0.61       & 0.37    & 1.83      & 0.03    & 5.81     & 1.34     & 0.15  \\
                              & I.T.(ms/iter)  & 2.94       & 7.65   & 10.44  & 5.16       & 1.96    & 8.54      & 0.20    & 156.66   & 8.49     & 0.48  \\
                              & I.M. (GB)      & 0.24       & 0.23   & 0.28   & 0.16       & 0.18    & 0.36      & 0.01    & 1.40     & 0.42     & 0.11  \\ \midrule
\multirow{6}{*}{\rotatebox[origin=c]{90}{PEMS07}}       & Params(M)      & 13.76      & 8.56   & 1.39   & 13.44      & 4.83    & 0.12      & 0.02    & 301.73   & 3.76     & 0.42  \\
                              & FLOPS(G)       & 10.93      & 15.65  & 16.10  & 20.53      & 6.66    & 4.77      & 0.02    & 290.97   & 34.34    & 0.37  \\
                              & T.T. (s/iter)  & 0.11       & 0.19   & 0.28   & 0.11       & 0.04    & 0.17      & 0.002   & 0.51     & 0.18     & 0.01  \\
                              & T.M.(GB)       & 7.65       & 3.85   & 13.99  & 6.42       & 3.07    & 11.62     & 0.07    & 5.86     & 8.20     & 0.79  \\
                              & I.T.(ms/iter)  & 20.80      & 47.06  & 94.01  & 36.54      & 13.92   & 68.96     & 0.20    & 156.78   & 8.19     & 0.50  \\
                              & I.M. (GB)      & 1.95       & 1.28   & 1.74   & 2.03       & 1.32    & 2.25      & 0.04    & 1.43     & 2.50     & 0.61  \\ \bottomrule
\end{tabular}
}
\caption{Resource utilization of the FreEformer and state-of-the-art forecasters on various real-world datasets under the `\texttt{input-96-predict-96}' setting. The table reports the number of parameters, FLOPS, training time (T.T.), training memory usage (T.M.), inference time (I.T.), and inference memory usage (I.M.). FLOPS calculations are performed using the {\tt fvcore} library and the batch size is fixed as 16. }
\label{tab_gpu}
\end{center}
\end{table*}

\section{More Ablation Studies}

\subsection{On More Bases Beyond Fourier}

Fourier and wavelet bases are among the most commonly used orthogonal bases in signal processing. For the wavelet bases, we choose two types: the classic Haar and discrete Meyer wavelets. The approximation and detail coefficients are processed independently, similar to how the real and imaginary parts are handled in FreEformer. In addition to Fourier and wavelet bases, we also evaluate three polynomial bases: Legendre, Chebyshev, and Laguerre, which are defined as follows:

\begin{equation}
\begin{aligned}
\mathrm{Legendre} _{k}(t)&=\frac{1}{2^{k} k!} \frac{d^{k}}{d t^{k}}\left[\left(t^{2}-1\right)^{k}\right],\\
\mathrm{Chebyshev} _{k}(t)&=\mathrm{cos} \left (k \cdot  \mathrm{arccos} \left ( t \right )   \right ) ,\\
\mathrm{Laguerre} _{k}(t)&=e^t\frac{d^{k}}{d t^{k}}\left ( e^{-t}t^k \right ). 
\end{aligned}
\label{eq_bases}
\end{equation}

\noindent Here, $\mathrm{Legendre} _{k}(t)$, $\mathrm{Chebyshev} _{k}(t)$, and $\mathrm{Laguerre} _{k}(t)$ represent the $k$-th polynomial basis for Legendre, Chebyshev, and Laguerre bases, respectively. These bases are real-valued polynomials that are orthogonal over specific intervals, with the Chebyshev and Laguerre bases being orthogonal under certain weight functions. 

We replace Fourier bases with the aforementioned bases and present the results in Table \ref{tab_more_bases}. For comparison, we also include results without any transformation (denoted as `Identity'), which correspond to the temporal representation in Table 6 of the main paper. As shown in Table \ref{tab_more_bases}, Fourier bases generally outperform the other bases, achieving average MSE improvements of 8.4\%, 6.9\%, and 2.2\% compared to the Identity, polynomial, and wavelet bases, respectively. The wavelet bases also perform well, even slightly outperforming Fourier bases on the PEMS03 dataset. Exploring more robust basis functions is a promising direction for future research.

\begin{table*}[t]
\begin{center}
{
\setlength{\tabcolsep}{8pt}
\begin{tabular}{@{}c|c|cc|cc|cc|cc|cc@{}}
\toprule
                            &                           & \multicolumn{2}{c|}{Weather}                                                  & \multicolumn{2}{c|}{Traffic}                                                  & \multicolumn{2}{c|}{Solar}                                                    & \multicolumn{2}{c|}{PEMS03}                                                               & \multicolumn{2}{c}{ILI}                                                       \\ \cmidrule(l){3-12} 
\multirow{-2}{*}{Base}      & \multirow{-2}{*}{Horizon} & MSE                                   & MAE                                   & MSE                                   & MAE                                   & MSE                                   & MAE                                   & MSE                                         & MAE                                         & MSE                                   & MAE                                   \\ \midrule
                            & H1                        & 0.153                                 & 0.189                                 & 0.395                                 & 0.233                                 & 0.180                                 & 0.191                                 & 0.060                                       & 0.160                                       & 0.508                                 & 0.366                                 \\
                            & H2                        & 0.201                                 & 0.236                                 & 0.423                                 & 0.245                                 & 0.213                                 & 0.215                                 & 0.077                                       & 0.181                                       & 0.898                                 & 0.522                                 \\
                            & H3                        & 0.261                                 & 0.282                                 & 0.443                                 & 0.254                                 & 0.233                                 & 0.232                                 & 0.112                                       & 0.218                                       & 1.277                                 & 0.648                                 \\
                            & H4                        & 0.341                                 & 0.334                                 & 0.480                                 & 0.274                                 & 0.241                                 & 0.237                                 & 0.159                                       & 0.265                                       & 1.876                                 & 0.805                                 \\ \cmidrule(l){2-12} 
\multirow{-5}{*}{\begin{tabular}[c]{@{}c@{}}Fourier\\ (Ours)\end{tabular}}   & Avg                       & {\color[HTML]{FF0000} \textbf{0.239}} & {\color[HTML]{FF0000} \textbf{0.260}} & {\color[HTML]{FF0000} \textbf{0.435}} & {\color[HTML]{0000FF} {\ul 0.251}} & {\color[HTML]{FF0000} \textbf{0.217}} & {\color[HTML]{0000FF} {\ul 0.219}}    & 0.102          &  0.206          & {\color[HTML]{FF0000} \textbf{1.140}} & {\color[HTML]{FF0000} \textbf{0.585}} \\ \midrule
                            & H1                        & 0.158                                 & 0.193                                 & 0.391                                 & 0.233                                 & 0.193                                 & 0.205                                 & 0.061                                       & 0.161                                       & 0.601                                 & 0.390                                 \\
                            & H2                        & 0.204                                 & 0.239                                 & 0.424                                 & 0.246                                 & 0.228                                 & 0.227                                 & 0.080                                       & 0.184                                       & 0.960                                 & 0.533                                 \\
                            & H3                        & 0.266                                 & 0.283                                 & 0.450                                 & 0.255                                 & 0.254                                 & 0.244                                 & 0.118                                       & 0.221                                       & 1.390                                 & 0.670                                 \\
                            & H4                        & 0.345                                 & 0.335                                 & 0.489                                 & 0.274                                 & 0.263                                 & 0.250                                 & 0.171                                       & 0.271                                       & 1.752                                 & 0.790                                 \\ \cmidrule(l){2-12} 
\multirow{-5}{*}{Legendre}  & Avg                       & 0.243                                 & 0.263                                 & {\color[HTML]{0000FF} {\ul 0.438}}    & 0.252                                 & 0.234                                 & 0.231                                 & 0.108                                       & 0.209                                       & {\color[HTML]{0000FF} {\ul 1.176}}    & 0.596                                 \\ \midrule
                            & H1                        & 0.157                                 & 0.193                                 & 0.405                                 & 0.249                                 & 0.190                                 & 0.204                                 & 0.064                                       & 0.164                                       & 0.537                                 & 0.370                                 \\
                            & H2                        & 0.202                                 & 0.237                                 & 0.429                                 & 0.260                                 & 0.227                                 & 0.226                                 & 0.084                                       & 0.187                                       & 0.859                                 & 0.523                                 \\
                            & H3                        & 0.268                                 & 0.287                                 & 0.453                                 & 0.267                                 & 0.251                                 & 0.242                                 & 0.122                                       & 0.223                                       & 1.746                                 & 0.738                                 \\
                            & H4                        & 0.344                                 & 0.336                                 & 0.484                                 & 0.287                                 & 0.261                                 & 0.246                                 & 0.177                                       & 0.275                                       & 2.723                                 & 0.890                                 \\ \cmidrule(l){2-12} 
\multirow{-5}{*}{Chebyshev} & Avg                       & 0.243                                 & 0.263                                 & 0.443                                 & 0.265                                 & 0.232                                 & 0.229                                 & 0.112                                       & 0.212                                       & 1.466                                 & 0.630                                 \\ \midrule
                            & H1                        & 0.161                                 & 0.197                                 & 0.409                                 & 0.243                                 & 0.193                                 & 0.204                                 & 0.066                                       & 0.168                                       & 0.631                                 & 0.415                                 \\
                            & H2                        & 0.208                                 & 0.243                                 & 0.421                                 & 0.255                                 & 0.228                                 & 0.225                                 & 0.086                                       & 0.191                                       & 1.182                                 & 0.637                                 \\
                            & H3                        & 0.270                                 & 0.287                                 & 0.451                                 & 0.263                                 & 0.253                                 & 0.244                                 & 0.123                                       & 0.228                                       & 1.688                                 & 0.771                                 \\
                            & H4                        & 0.350                                 & 0.340                                 & 0.488                                 & 0.283                                 & 0.262                                 & 0.249                                 & 0.177                                       & 0.273                                       & 2.153                                 & 0.893                                 \\ \cmidrule(l){2-12} 
\multirow{-5}{*}{Laguerre}  & Avg                       & 0.247                                 & 0.267                                 & 0.442                                 & 0.261                                 & 0.234                                 & 0.230                                 & 0.113                                       & 0.215                                       & 1.413                                 & 0.679                                 \\ \midrule
                            & H1                        & 0.156                                 & 0.190                                 & 0.408                                 & 0.231                                 & 0.182                                 & 0.192                                 & 0.062                                       & 0.165                                       & 0.568                                 & 0.394                                 \\
                            & H2                        & 0.203                                 & 0.239                                 & 0.440                                 & 0.245                                 & 0.213                                 & 0.213                                 & 0.077                                       & 0.180                                       & 0.986                                 & 0.532                                 \\
                            & H3                        & 0.260                                 & 0.280                                 & 0.461                                 & 0.253                                 & 0.234                                 & 0.231                                 & 0.108                                       & 0.215                                       & 1.461                                 & 0.686                                 \\
                            & H4                        & 0.350                                 & 0.340                                 & 0.498                                 & 0.272                                 & 0.243                                 & 0.237                                 & 0.152                                       & 0.257                                       & 1.987                                 & 0.805                                 \\ \cmidrule(l){2-12} 
\multirow{-5}{*}{Wavelet1}  & Avg                       & {\color[HTML]{0000FF} {\ul 0.242}}    & {\color[HTML]{0000FF} {\ul 0.262}}    & 0.452                                 &  {\color[HTML]{FF0000} \textbf{0.250}}    & {\color[HTML]{0000FF} {\ul 0.218}}    & {\color[HTML]{FF0000} \textbf{0.218}} & {\color[HTML]{0000FF} {\ul 0.100}} & {\color[HTML]{0000FF} {\ul 0.204}} & 1.250                                 & 0.604                                 \\ \midrule
                            & H1                        & 0.156                                 & 0.191                                 & 0.409                                 & 0.232                                 & 0.183                                 & 0.196                                 & 0.061                                       & 0.159                                       & 0.486                                 & 0.358                                 \\
                            & H2                        & 0.206                                 & 0.239                                 & 0.439                                 & 0.245                                 & 0.216                                 & 0.219                                 & 0.077                                       & 0.180                                       & 0.942                                 & 0.521                                 \\
                            & H3                        & 0.272                                 & 0.286                                 & 0.463                                 & 0.254                                 & 0.238                                 & 0.235                                 & 0.105                                       & 0.213                                       & 1.390                                 & 0.668                                 \\
                            & H4                        & 0.345                                 & 0.336                                 & 0.503                                 & 0.274                                 & 0.248                                 & 0.240                                 & 0.150                                       & 0.257                                       & 2.024                                 & 0.827                                 \\ \cmidrule(l){2-12} 
\multirow{-5}{*}{Wavelet2}  & Avg                       & 0.244                                 & 0.263                                 & 0.453                                 & {\color[HTML]{0000FF} {\ul 0.251}}                                 & 0.221                                 & 0.222                                 & {\color[HTML]{FF0000} \textbf{0.098}}       & {\color[HTML]{FF0000} \textbf{0.202}}       & 1.210                                 & {\color[HTML]{0000FF} {\ul 0.593}}    \\ \midrule
                            & H1                        & 0.155                                 & 0.191                                 & 0.399                                 & 0.234                                 & 0.187                                 & 0.198                                 & 0.062                                       & 0.162                                       & 0.494                                 & 0.352                                 \\
                            & H2                        & 0.205                                 & 0.239                                 & 0.431                                 & 0.246                                 & 0.221                                 & 0.219                                 & 0.079                                       & 0.183                                       & 0.889                                 & 0.503                                 \\
                            & H3                        & 0.268                                 & 0.287                                 & 0.454                                 & 0.255                                 & 0.246                                 & 0.238                                 & 0.116                                       & 0.220                                       & 1.328                                 & 0.657                                 \\
                            & H4                        & 0.343                                 & 0.335                                 & 0.488                                 & 0.273                                 & 0.259                                 & 0.245                                 & 0.232                                       & 0.320                                       & 2.791                                 & 0.861                                 \\ \cmidrule(l){2-12} 
\multirow{-5}{*}{Identity}  & Avg                       & 0.243                                 & 0.263                                 & 0.443                                 & 0.252                                 & 0.228                                 & 0.225                                 & 0.122                                       & 0.221                                       & 1.375                                 & {\color[HTML]{0000FF} {\ul 0.593}}    \\ \bottomrule
\end{tabular}
}
\caption{Ablation studies on additional polynomial and wavelet bases. For `Wavelet1' and `Wavelet2', we employ the classic Haar and discrete Meyer wavelets, respectively. `Identity' indicates no transformation, corresponding to the temporal representation shown in Table 6 of the main paper. The lookback length $T$ and prediction horizons $\left(\mathrm{H}1, \mathrm{H}2, \mathrm{H}3, \mathrm{H}4\right)$ are configured as follows. Weather, Traffic, and Solar-Energy datasets: $T = 96$; $\left(\mathrm{H}1, \mathrm{H}2, \mathrm{H}3, \mathrm{H}4\right) = \left(96, 192, 336, 720\right)$. PEMS03 dataset: $T = 96$; $\left(\mathrm{H}1, \mathrm{H}2, \mathrm{H}3, \mathrm{H}4\right) = \left(12, 24, 48, 96\right)$. ILI dataset: $T = 12$; $\left(\mathrm{H}1, \mathrm{H}2, \mathrm{H}3, \mathrm{H}4\right) = \left(3, 6, 9, 12\right)$. For Legendre, Chebyshev, and Laguerre bases, we select the number of basis functions from the set $\{5, \mathrm{T}\}$ and report the better results.}
\label{tab_more_bases}
\end{center}
\end{table*}

\subsection{Module Ablations}

To verify the effectiveness of the frequency-domain representation and enhanced attention modules, we conduct ablation studies on these two components. As shown in Table \ref{tab_module_ablation}, both modules consistently improve performance. On these datasets, the frequency-domain representation yields an average improvement of \textbf{6.2\%}, while the enhanced attention mechanism provides an additional improvement of \textbf{4.9\%}.

\begin{table*}[t]
\begin{center}
{\fontsize{9}{10}\selectfont
\setlength{\tabcolsep}{5.35pt}
\begin{tabular}{@{}ccc|c|cc|cc|cc|cc|cc|cc@{}}
\toprule
                       &                        &                             &                           & \multicolumn{2}{c|}{Weather}                                                  & \multicolumn{2}{c|}{ECL}                                                      & \multicolumn{2}{c|}{Traffic}                                                  & \multicolumn{2}{c|}{Solar-Energy}                                             & \multicolumn{2}{c|}{PEMS03}                                                   & \multicolumn{2}{c}{ILI}                                                       \\ \cmidrule(l){5-16} 
\multirow{-2}{*}{Base} & \multirow{-2}{*}{Fre.} & \multirow{-2}{*}{Enh.Attn.} & \multirow{-2}{*}{Horizon} & MSE                                   & MAE                                   & MSE                                   & MAE                                   & MSE                                   & MAE                                   & MSE                                   & MAE                                   & MSE                                   & MAE                                   & MSE                                   & MAE                                   \\ \midrule
                       &                        &                             & H1                        & 0.171                                 & 0.202                                 & 0.139                                 & 0.228                                 & 0.394                                 & 0.238                                 & 0.189                                 & 0.199                                 & 0.065                                 & 0.168                                 & 0.667                                 & 0.380                                 \\
                       &                        &                             & H2                        & 0.208                                 & 0.240                                 & 0.157                                 & 0.245                                 & 0.429                                 & 0.251                                 & 0.221                                 & 0.220                                 & 0.090                                 & 0.197                                 & 1.155                                 & 0.551                                 \\
                       &                        &                             & H3                        & 0.263                                 & 0.282                                 & 0.174                                 & 0.261                                 & 0.455                                 & 0.256                                 & 0.241                                 & 0.237                                 & 0.139                                 & 0.245                                 & 2.190                                 & 0.769                                 \\
                       &                        &                             & H4                        & 0.349                                 & 0.338                                 & 0.197                                 & 0.283                                 & 0.488                                 & 0.273                                 & 0.252                                 & 0.243                                 & 0.291                                 & 0.363                                 & 4.547                                 & 0.975                                 \\ \cmidrule(l){4-16} 
\multirow{-5}{*}{\ding{52}}    & \multirow{-5}{*}{}     & \multirow{-5}{*}{}          & Avg                       & 0.248                                 & 0.266                                 & 0.167                                 & 0.254                                 & {\color[HTML]{0000FF} {\ul 0.441}}    & 0.254                                 & 0.226                                 & 0.225                                 & 0.146                                 & 0.243                                 & 2.140                                 & 0.669                                 \\ \midrule
                       &                        &                             & H1                        & 0.163                                 & 0.198                                 & 0.137                                 & 0.226                                 & 0.411                                 & 0.235                                 & 0.182                                 & 0.195                                 & 0.065                                 & 0.167                                 & 0.493                                 & 0.368                                 \\
                       &                        &                             & H2                        & 0.207                                 & 0.240                                 & 0.155                                 & 0.242                                 & 0.436                                 & 0.246                                 & 0.221                                 & 0.220                                 & 0.084                                 & 0.190                                 & 1.124                                 & 0.548                                 \\
                       &                        &                             & H3                        & 0.265                                 & 0.283                                 & 0.170                                 & 0.258                                 & 0.457                                 & 0.253                                 & 0.234                                 & 0.233                                 & 0.124                                 & 0.231                                 & 1.880                                 & 0.737                                 \\
                       &                        &                             & H4                        & 0.344                                 & 0.337                                 & 0.197                                 & 0.284                                 & 0.502                                 & 0.273                                 & 0.243                                 & 0.241                                 & 0.181                                 & 0.284                                 & 2.545                                 & 0.847                                 \\ \cmidrule(l){4-16} 
\multirow{-5}{*}{\ding{52}}    & \multirow{-5}{*}{\ding{52}}    & \multirow{-5}{*}{}          & Avg                       & {\color[HTML]{0000FF} {\ul 0.245}}    & {\color[HTML]{0000FF} {\ul 0.264}}    & {\color[HTML]{0000FF} {\ul 0.165}}    & {\color[HTML]{0000FF} {\ul 0.252}}    & 0.451                                 & {\color[HTML]{0000FF} {\ul 0.252}}    & {\color[HTML]{0000FF} {\ul 0.220}}    & {\color[HTML]{0000FF} {\ul 0.222}}    & {\color[HTML]{0000FF} {\ul 0.113}}    & {\color[HTML]{0000FF} {\ul 0.218}}    & {\color[HTML]{0000FF} {\ul 1.510}}    & {\color[HTML]{0000FF} {\ul 0.625}}    \\ \midrule
                       &                        &                             & H1                        & 0.153                                 & 0.189                                 & 0.133                                 & 0.223                                 & 0.395                                 & 0.233                                 & 0.180                                 & 0.191                                 & 0.060                                 & 0.160                                 & 0.508                                 & 0.366                                 \\
                       &                        &                             & H2                        & 0.201                                 & 0.236                                 & 0.152                                 & 0.240                                 & 0.423                                 & 0.245                                 & 0.213                                 & 0.215                                 & 0.077                                 & 0.181                                 & 0.898                                 & 0.522                                 \\
                       &                        &                             & H3                        & 0.261                                 & 0.282                                 & 0.165                                 & 0.256                                 & 0.443                                 & 0.254                                 & 0.233                                 & 0.232                                 & 0.112                                 & 0.218                                 & 1.277                                 & 0.648                                 \\
                       &                        &                             & H4                        & 0.341                                 & 0.334                                 & 0.198                                 & 0.286                                 & 0.480                                 & 0.274                                 & 0.241                                 & 0.237                                 & 0.159                                 & 0.265                                 & 1.876                                 & 0.805                                 \\ \cmidrule(l){4-16} 
\multirow{-5}{*}{\ding{52}}    & \multirow{-5}{*}{\ding{52}}    & \multirow{-5}{*}{\ding{52}}         & Avg                       & {\color[HTML]{FF0000} \textbf{0.239}} & {\color[HTML]{FF0000} \textbf{0.260}} & {\color[HTML]{FF0000} \textbf{0.162}} & {\color[HTML]{FF0000} \textbf{0.251}} & {\color[HTML]{FF0000} \textbf{0.435}} & {\color[HTML]{FF0000} \textbf{0.251}} & {\color[HTML]{FF0000} \textbf{0.217}} & {\color[HTML]{FF0000} \textbf{0.219}} & {\color[HTML]{FF0000} \textbf{0.102}} & {\color[HTML]{FF0000} \textbf{0.206}} & {\color[HTML]{FF0000} \textbf{1.140}} & {\color[HTML]{FF0000} \textbf{0.585}} \\ \bottomrule
\end{tabular}
}
\caption{Ablation studies on the frequency-domain representation and enhanced attention modules. The base model refers to FreEformer without the frequency-domain representation and using vanilla attention. The lookback length $T$ and prediction horizons $\left(\mathrm{H}1, \mathrm{H}2, \mathrm{H}3, \mathrm{H}4\right)$ are configured as follows: Weather, ECL, Traffic, and Solar-Energy datasets: $T = 96$, $\left(\mathrm{H}1, \mathrm{H}2, \mathrm{H}3, \mathrm{H}4\right) = \left(96, 192, 336, 720\right)$; PEMS03 dataset: $T = 96$, $\left(\mathrm{H}1, \mathrm{H}2, \mathrm{H}3, \mathrm{H}4\right) = \left(12, 24, 48, 96\right)$; ILI dataset: $T = 12$, $\left(\mathrm{H}1, \mathrm{H}2, \mathrm{H}3, \mathrm{H}4\right) = \left(3, 6, 9, 12\right)$. This table provides a comprehensive version of Table 6 in the main paper.}
\label{tab_module_ablation}
\end{center}
\end{table*}

\subsection{Variants of Enhanced Attention}

\begin{table*}[t]
\begin{center}
{\fontsize{8}{9}\selectfont
\begin{tabular}{@{}c@{\hspace{2pt}}c@{\hspace{5pt}}|@{\hspace{5pt}}
c@{\hspace{5pt}}c@{\hspace{5pt}}|@{\hspace{5pt}}
c@{\hspace{5pt}}c@{\hspace{5pt}}|@{\hspace{5pt}}
c@{\hspace{5pt}}c@{\hspace{5pt}}|@{\hspace{5pt}}
c@{\hspace{5pt}}c@{\hspace{5pt}}|@{\hspace{5pt}}
c@{\hspace{5pt}}c@{\hspace{5pt}}|@{\hspace{5pt}}
c@{\hspace{5pt}}c@{\hspace{5pt}}|@{\hspace{5pt}}
c@{\hspace{5pt}}c@{\hspace{5pt}}|@{\hspace{5pt}}
c@{\hspace{5pt}}c@{\hspace{5pt}}|@{\hspace{5pt}}
c@{\hspace{5pt}}c@{}}

\toprule
\multicolumn{2}{c|@{\hspace{5pt}}}{Attention}          & \multicolumn{2}{c|@{\hspace{5pt}}}{Vanilla}                                                  & \multicolumn{2}{c|@{\hspace{5pt}}}{Ours}                                                     & \multicolumn{2}{c|@{\hspace{5pt}}}{Var1}                                                     & \multicolumn{2}{c|@{\hspace{5pt}}}{Var2}                                                     & \multicolumn{2}{c|@{\hspace{5pt}}}{Var3}                                                     & \multicolumn{2}{c|@{\hspace{5pt}}}{Var4}                                                     & \multicolumn{2}{c|@{\hspace{5pt}}}{Var5}                                                     & \multicolumn{2}{c|@{\hspace{5pt}}}{Var6}                                                     & \multicolumn{2}{c}{Var7}                                                      \\
\bottomrule
\end{tabular}
}
\caption{Full results of vanilla self-attention and enhanced attention variants. The formulas of these variants are presented in Table \ref{tab_var_formula}. The lookback length $T$ is fixed at 96.}
\label{tab_attn_var_appd}
\end{center}
\end{table*}

\begin{table*}[t]
\begin{center}
{\fontsize{8}{9}\selectfont
\begin{tabular}{@{}c@{\hspace{2pt}}c@{\hspace{5pt}}|@{\hspace{5pt}}
c@{\hspace{5pt}}c@{\hspace{5pt}}|@{\hspace{5pt}}
c@{\hspace{5pt}}c@{\hspace{5pt}}|@{\hspace{5pt}}
c@{\hspace{5pt}}c@{\hspace{5pt}}|@{\hspace{5pt}}
c@{\hspace{5pt}}c@{\hspace{5pt}}|@{\hspace{5pt}}
c@{\hspace{5pt}}c@{\hspace{5pt}}|@{\hspace{5pt}}
c@{\hspace{5pt}}c@{\hspace{5pt}}|@{\hspace{5pt}}
c@{\hspace{5pt}}c@{\hspace{5pt}}|@{\hspace{5pt}}
c@{\hspace{5pt}}c@{\hspace{5pt}}|@{\hspace{5pt}}
c@{\hspace{5pt}}c@{}}

\toprule
\multicolumn{2}{c|@{\hspace{5pt}}}{Attention}      & \multicolumn{2}{c|@{\hspace{5pt}}}{Vanilla} & \multicolumn{2}{c|@{\hspace{5pt}}}{Ours}                                                     & \multicolumn{2}{c|@{\hspace{5pt}}}{Var1}                     & \multicolumn{2}{c|@{\hspace{5pt}}}{Var2}                                               & \multicolumn{2}{c|@{\hspace{5pt}}}{Var3}                                                  & \multicolumn{2}{c|@{\hspace{5pt}}}{Var4}                                               & \multicolumn{2}{c|@{\hspace{5pt}}}{Var5}                                                  & \multicolumn{2}{c|@{\hspace{5pt}}}{Var6} & \multicolumn{2}{c}{Var7}                                                      \\ \midrule
\multicolumn{2}{c|@{\hspace{5pt}}}{Metric}     & MSE           & MAE          & MSE                                   & MAE                                   & MSE                                   & MAE   & MSE                                & MAE                                & MSE                                   & MAE                                & MSE                                & MAE                                & MSE                                & MAE                                   & MSE         & MAE         & MSE                                   & MAE                                   \\
\bottomrule
\end{tabular}

}
\caption{Full results of vanilla self-attention and enhanced attention variants for the PEMS datasets. The formulas of these variants are presented in Table \ref{tab_var_formula}. The lookback length $T$ is fixed at 96.}
\label{tab_attn_var_pems_appd}
\end{center}
\end{table*}

\begin{table*}[t]
\begin{center}
{\fontsize{8}{9}\selectfont
\begin{tabular}{@{}c@{\hspace{2pt}}c@{\hspace{3pt}}|@{\hspace{3pt}}
c@{\hspace{5pt}}c@{\hspace{3pt}}|@{\hspace{3pt}}
c@{\hspace{5pt}}c@{\hspace{3pt}}|@{\hspace{3pt}}
c@{\hspace{5pt}}c@{\hspace{3pt}}|@{\hspace{3pt}}
c@{\hspace{5pt}}c@{\hspace{3pt}}|@{\hspace{3pt}}
c@{\hspace{5pt}}c@{\hspace{3pt}}|@{\hspace{3pt}}
c@{\hspace{5pt}}c@{\hspace{3pt}}|@{\hspace{3pt}}
c@{\hspace{5pt}}c@{\hspace{3pt}}|@{\hspace{3pt}}
c@{\hspace{5pt}}c@{\hspace{3pt}}|@{\hspace{3pt}}
c@{\hspace{5pt}}c@{}}

\toprule
\multicolumn{2}{c|@{\hspace{3pt}}}{Attention}          & \multicolumn{2}{c|@{\hspace{3pt}}}{Vanilla}                                                  & 
\multicolumn{2}{c|@{\hspace{3pt}}}{Ours}                                                     & \multicolumn{2}{c|@{\hspace{3pt}}}{Var1}                                                     & \multicolumn{2}{c|@{\hspace{3pt}}}{Var2}                                                     & \multicolumn{2}{c|@{\hspace{3pt}}}{Var3}                                                     & \multicolumn{2}{c|@{\hspace{3pt}}}{Var4}                                                     & \multicolumn{2}{c|@{\hspace{3pt}}}{Var5}                                                     & \multicolumn{2}{c|@{\hspace{3pt}}}{Var6}                                                     & \multicolumn{2}{c}{Var7}                                                      \\
\bottomrule
\end{tabular}

}
\caption{Full results of vanilla self-attention and enhanced attention variants for short-term forecasting. The formulas of these variants are presented in Table \ref{tab_var_formula}. For the prediction length $\tau  \in \left \{ 3,6,9,12 \right \}$, the lookback length is 12; for the prediction length $\tau  \in \left \{ 24,36,48,60 \right \}$, the lookback length is 36.}

\label{tab_attn_var_short_appd}
\end{center}
\end{table*}

Tables \ref{tab_attn_var_appd}, \ref{tab_attn_var_pems_appd}, and \ref{tab_attn_var_short_appd} compare vanilla attention with enhanced attention variants for both long- and short-term time series forecasting. These enhanced variants consistently outperform vanilla attention across various datasets and prediction settings. Among them, the enhanced attention adopted in the main paper achieves the best performance. Notably, for datasets with limited size, the enhanced attention, despite having more learnable parameters, still achieves better or comparable performance compared to vanilla attention, demonstrating its adaptability to limited data. The performance superiority can be attributed to the following factors:

\begin{itemize}
    \item As analyzed in Section \ref{C.2}, the addition operation increases the rank of the final attention matrix, enriching representation diversity;
    \item The Softplus function applied to $\mathbf{B}$, combined with L1 normalization, offers greater flexibility compared to the Softmax function used in other variants;
    \item Its Jacobian matrix $\partial \mathbf{c}/\partial \mathbf{a}$ (Equation \ref{eq28}) retains the same structure as vanilla attention (Equation \ref{eq22}), while the additional learnable term introduces greater flexibility to the gradient flow of the original branch.
\end{itemize}

\subsection{Applying Weighted L1 Loss Function to Other Forecasters} 

In the training phase of FreEformer, we adopt the weighted L1 loss function from CARD \cite{card} $\mathcal{L}=\frac{1}{\tau} \sum_{t=1}^{\tau} t^{-\alpha} \left \| \mathbf{\hat{y}}_{:t}-\mathbf{y}_{:t} \right \|_1$, where $\mathbf{\hat{y}}_{:t},\mathbf{y}_{:t}\in \mathbb{R}^N $ represent the values of the prediction $\mathbf{\hat{y}}$ and ground truth $\mathbf{y}$ at timestamp $t$. We set $\alpha$ to 0.5 in our experiments. This weighting scheme assigns lower weights to predictions further into the future to reflect the increasing uncertainty.

For a fair comparison, this loss function is also applied to other forecasters. FreEformer continues to outperform these state-of-the-art forecasters, demonstrating that its performance superiority arises primarily from its architectural design rather than the loss function.


\begin{table*}[ht]
\begin{center}
{\fontsize{7}{9}\selectfont
\setlength{\tabcolsep}{4pt}
\begin{tabular}{@{}cc|cc|cc|cc|cc|cc|cc|cc|cc|cc|cc@{}}
\toprule
\multicolumn{2}{c}{Model}            & \multicolumn{2}{c}{Ours}                                                      & \multicolumn{2}{c}{Leddam}                                              & \multicolumn{2}{c}{CARD}                                                      & \multicolumn{2}{c}{Fredformer}                                             & \multicolumn{2}{c}{iTrans.}           & \multicolumn{2}{c}{TimeMixer} & \multicolumn{2}{c}{PatchTST} & \multicolumn{2}{c}{DLinear} & \multicolumn{2}{c}{TimesNet} & \multicolumn{2}{c}{FreTS}                                                     \\ \midrule
\multicolumn{2}{c|}{Metric}           & MSE                                   & MAE                                   & MSE                                & MAE                                & MSE                                   & MAE                                   & MSE                                   & MAE                                & MSE   & MAE                                & MSE            & MAE          & MSE           & MAE          & MSE          & MAE          & MSE           & MAE          & MSE                                   & MAE                                   \\ \midrule
                               & 96  & {\color[HTML]{FF0000} \textbf{0.133}} & {\color[HTML]{FF0000} \textbf{0.223}} & {\color[HTML]{0000FF} {\ul 0.137}} & {\color[HTML]{0000FF} {\ul 0.228}} & 0.141                                 & 0.233                                 & 0.145                                 & 0.232                              & 0.147 & 0.233                              & 0.161          & 0.251        & 0.174         & 0.251        & 0.198        & 0.271        & 0.167         & 0.265        & 0.179                                 & 0.255                                 \\
                               & 192 & {\color[HTML]{FF0000} \textbf{0.152}} & {\color[HTML]{FF0000} \textbf{0.240}} & {\color[HTML]{0000FF} {\ul 0.158}} & 0.248                              & 0.160                                 & 0.250                                 & 0.162                                 & 0.248                              & 0.162 & {\color[HTML]{0000FF} {\ul 0.247}} & 0..174         & 0.262        & 0.182         & 0.260        & 0.197        & 0.274        & 0.182         & 0.279        & 0.182                                 & 0.261                                 \\
                               & 336 & {\color[HTML]{FF0000} \textbf{0.165}} & {\color[HTML]{FF0000} \textbf{0.256}} & {\color[HTML]{0000FF} {\ul 0.172}} & {\color[HTML]{0000FF} {\ul 0.263}} & 0.173                                 & {\color[HTML]{0000FF} {\ul 0.263}}    & 0.175                                 & 0.264                              & 0.179 & 0.266                              & 0.189          & 0.276        & 0.197         & 0.275        & 0.209        & 0.289        & 0.207         & 0.305        & 0.197                                 & 0.277                                 \\
                               & 720 & {\color[HTML]{0000FF} {\ul 0.198}}    & {\color[HTML]{0000FF} {\ul 0.286}}    & 0.203                              & 0.289                              & {\color[HTML]{FF0000} \textbf{0.197}} & {\color[HTML]{FF0000} \textbf{0.284}} & 0.205                                 & 0.289                              & 0.207 & 0.289                              & 0.227          & 0.304        & 0.236         & 0.308        & 0.244        & 0.321        & 0.240         & 0.327        & 0.234                                 & 0.311                                 \\
\multirow{-5}{*}{\rotatebox[origin=c]{90}{ECL}}          & Avg & {\color[HTML]{FF0000} \textbf{0.162}} & {\color[HTML]{FF0000} \textbf{0.251}} & {\color[HTML]{0000FF} {\ul 0.167}} & {\color[HTML]{0000FF} {\ul 0.257}} & 0.168                                 & 0.258                                 & 0.172                                 & 0.258                              & 0.174 & 0.259                              & 0.192          & 0.273        & 0.197         & 0.273        & 0.212        & 0.289        & 0.199         & 0.294        & 0.198                                 & 0.276                                 \\ \midrule
                               & 96  & {\color[HTML]{FF0000} \textbf{0.395}} & {\color[HTML]{FF0000} \textbf{0.233}} & 0.438                              & 0.258                              & 0.419                                 & 0.269                                 & {\color[HTML]{0000FF} {\ul 0.405}}    & {\color[HTML]{0000FF} {\ul 0.252}} & 0.407 & 0.255                              & 0.470          & 0.273        & 0.489         & 0.286        & 0.662        & 0.370        & 0.607         & 0.301        & 0.502                                 & 0.299                                 \\
                               & 192 & {\color[HTML]{FF0000} \textbf{0.423}} & {\color[HTML]{FF0000} \textbf{0.245}} & 0.446                              & {\color[HTML]{0000FF} {\ul 0.262}} & 0.443                                 & 0.276                                 & {\color[HTML]{0000FF} {\ul 0.428}}    & 0.264                              & 0.431 & 0.266                              & 0.481          & 0.282        & 0.492         & 0.287        & 0.614        & 0.347        & 0.632         & 0.316        & 0.498                                 & 0.297                                 \\
                               & 336 & {\color[HTML]{0000FF} {\ul 0.443}}    & {\color[HTML]{FF0000} \textbf{0.254}} & 0.465                              & 0.268                              & 0.460                                 & 0.283                                 & {\color[HTML]{FF0000} \textbf{0.437}} & {\color[HTML]{0000FF} {\ul 0.262}} & 0.449 & 0.274                              & 0.500          & 0.294        & 0.506         & 0.292        & 0.618        & 0.349        & 0.652         & 0.328        & 0.516                                 & 0.302                                 \\
                               & 720 & {\color[HTML]{0000FF} {\ul 0.480}}    & {\color[HTML]{FF0000} \textbf{0.274}} & 0.507                              & {\color[HTML]{0000FF} {\ul 0.281}} & 0.490                                 & 0.299                                 & {\color[HTML]{FF0000} \textbf{0.476}} & 0.287                              & 0.481 & 0.293                              & 0.544          & 0.310        & 0.542         & 0.310        & 0.651        & 0.369        & 0.686         & 0.343        & 0.555                                 & 0.320                                 \\
\multirow{-5}{*}{\rotatebox[origin=c]{90}{Traffic}}      & Avg & {\color[HTML]{FF0000} \textbf{0.435}} & {\color[HTML]{FF0000} \textbf{0.251}} & 0.464                              & 0.267                              & 0.453                                 & 0.282                                 & {\color[HTML]{0000FF} {\ul 0.436}}    & {\color[HTML]{0000FF} {\ul 0.266}} & 0.442 & 0.272                              & 0.499          & 0.290        & 0.507         & 0.294        & 0.636        & 0.359        & 0.644         & 0.322        & 0.518                                 & 0.304                                 \\ \midrule
                               & 96  & {\color[HTML]{FF0000} \textbf{0.180}} & {\color[HTML]{FF0000} \textbf{0.191}} & {\color[HTML]{0000FF} {\ul 0.191}} & 0.201                              & 0.197                                 & 0.211                                 & 0.192                                 & 0.205                              & 0.193 & 0.202                              & 0.196          & 0.213        & 0.207         & 0.214        & 0.287        & 0.297        & 0.224         & 0.225        & {\color[HTML]{FF0000} \textbf{0.180}} & {\color[HTML]{0000FF} {\ul 0.192}}    \\
                               & 192 & {\color[HTML]{FF0000} \textbf{0.213}} & {\color[HTML]{FF0000} \textbf{0.215}} & 0.223                              & 0.221                              & 0.234                                 & 0.234                                 & 0.240                                 & 0.235                              & 0.232 & 0.228                              & 0.230          & 0.236        & 0.243         & 0.234        & 0.318        & 0.313        & 0.270         & 0.254        & {\color[HTML]{0000FF} {\ul 0.216}}    & {\color[HTML]{0000FF} {\ul 0.216}}    \\
                               & 336 & {\color[HTML]{0000FF} {\ul 0.233}}    & {\color[HTML]{0000FF} {\ul 0.232}}    & 0.262                              & 0.253                              & 0.256                                 & 0.250                                 & 0.244                                 & 0.244                              & 0.244 & 0.242                              & 0.256          & 0.253        & 0.272         & 0.251        & 0.364        & 0.328        & 0.279         & 0.266        & {\color[HTML]{FF0000} \textbf{0.232}} & {\color[HTML]{FF0000} \textbf{0.231}} \\
                               & 720 & {\color[HTML]{FF0000} \textbf{0.241}} & {\color[HTML]{FF0000} \textbf{0.237}} & 0.249                              & 0.246                              & 0.260                                 & 0.254                                 & 0.257                                 & 0.253                              & 0.253 & 0.247                              & 0.262          & 0.255        & 0.273         & 0.250        & 0.374        & 0.322        & 0.293         & 0.273        & {\color[HTML]{0000FF} {\ul 0.243}}    & {\color[HTML]{0000FF} {\ul 0.238}}    \\
\multirow{-5}{*}{\rotatebox[origin=c]{90}{Solar-Energy}} & Avg & {\color[HTML]{FF0000} \textbf{0.217}} & {\color[HTML]{FF0000} \textbf{0.219}} & 0.231                              & {\color[HTML]{0000FF} {\ul 0.230}} & 0.237                                 & 0.237                                 & 0.233                                 & 0.234                              & 0.231 & {\color[HTML]{0000FF} {\ul 0.230}} & 0.236          & 0.239        & 0.249         & 0.237        & 0.336        & 0.315        & 0.266         & 0.254        & {\color[HTML]{0000FF} {\ul 0.218}}    & {\color[HTML]{FF0000} \textbf{0.219}} \\ \bottomrule
\end{tabular}
}
\caption{Performance comparison under the weighted L1 loss function. The lookback length is fixed at 96. FreEformer consistently outperforms other forecasters under the new loss function.}
\label{tab_apply_L1}
\end{center}
\end{table*}

\subsection{Hyperparameter Sensitivity}

\paragraph{L1/L2 and $\alpha$}

\begin{table*}[ht]
\begin{center}
{\fontsize{8}{10}\selectfont
\begin{tabular}{@{}cc|cccccc|cccccc@{}}
\toprule
\multicolumn{2}{c|}{L1/L2}           & \multicolumn{6}{c|}{L1}                                                                                                                                                                                                                                                                       & \multicolumn{6}{c}{L2}                                                                                                                                                                                               \\ \midrule
\multicolumn{2}{c|}{Alpha}           & \multicolumn{2}{c|}{0.0}                                                                           & \multicolumn{2}{c|}{0.5}                                                                           & \multicolumn{2}{c|}{1.0}                                                            & \multicolumn{2}{c|}{0.0}                                                                        & \multicolumn{2}{c|}{0.5}                                           & \multicolumn{2}{c}{1.0}                       \\ \midrule
\multicolumn{2}{c|}{Metric}          & MSE                                   & \multicolumn{1}{c|}{MAE}                                   & MSE                                   & \multicolumn{1}{c|}{MAE}                                   & MSE                                         & MAE                                   & MSE                                   & \multicolumn{1}{c|}{MAE}                                & MSE                                   & \multicolumn{1}{c|}{MAE}   & MSE                                   & MAE   \\ \midrule
                               & 96  & {\color[HTML]{0000FF} {\ul 0.135}}    & \multicolumn{1}{c|}{0.226}                                 & {\color[HTML]{FF0000} \textbf{0.133}} & \multicolumn{1}{c|}{{\color[HTML]{FF0000} \textbf{0.223}}} & {\color[HTML]{FF0000} \textbf{0.133}}       & {\color[HTML]{FF0000} \textbf{0.223}} & 0.138                                 & \multicolumn{1}{c|}{{\color[HTML]{0000FF} {\ul 0.234}}} & 0.136                                 & \multicolumn{1}{c|}{0.232} & 0.136                                 & 0.232 \\
                               & 192 & {\color[HTML]{FF0000} \textbf{0.152}} & \multicolumn{1}{c|}{{\color[HTML]{0000FF} {\ul 0.241}}}    & {\color[HTML]{FF0000} \textbf{0.152}} & \multicolumn{1}{c|}{{\color[HTML]{FF0000} \textbf{0.240}}} & {\color[HTML]{FF0000} \textbf{0.152}}       & {\color[HTML]{FF0000} \textbf{0.240}} & 0.158                                 & \multicolumn{1}{c|}{0.252}                              & {\color[HTML]{0000FF} {\ul 0.154}}    & \multicolumn{1}{c|}{0.249} & 0.155                                 & 0.249 \\
                               & 336 & 0.170                                 & \multicolumn{1}{c|}{{\color[HTML]{0000FF} {\ul 0.259}}}    & {\color[HTML]{FF0000} \textbf{0.165}} & \multicolumn{1}{c|}{{\color[HTML]{FF0000} \textbf{0.256}}} & {\color[HTML]{0000FF} {\ul 0.167}}          & {\color[HTML]{FF0000} \textbf{0.256}} & 0.170                                 & \multicolumn{1}{c|}{0.266}                              & 0.171                                 & \multicolumn{1}{c|}{0.266} & 0.171                                 & 0.267 \\
                               & 720 & 0.201                                 & \multicolumn{1}{c|}{0.288}                                 & {\color[HTML]{0000FF} {\ul 0.198}}    & \multicolumn{1}{c|}{{\color[HTML]{0000FF} {\ul 0.286}}}    & {\color[HTML]{FF0000} \textbf{0.195}}       & {\color[HTML]{FF0000} \textbf{0.283}} & 0.199                                 & \multicolumn{1}{c|}{0.292}                              & {\color[HTML]{FF0000} \textbf{0.195}} & \multicolumn{1}{c|}{0.291} & 0.199                                 & 0.294 \\
\multirow{-5}{*}{\rotatebox[origin=c]{90}{ECL}}          & Avg & {\color[HTML]{0000FF} {\ul 0.164}}    & \multicolumn{1}{c|}{{\color[HTML]{0000FF} {\ul 0.253}}}    & {\color[HTML]{FF0000} \textbf{0.162}} & \multicolumn{1}{c|}{{\color[HTML]{FF0000} \textbf{0.251}}} & {\color[HTML]{FF0000} \textbf{0.162}}       & {\color[HTML]{FF0000} \textbf{0.251}} & 0.166                                 & \multicolumn{1}{c|}{0.261}                              & {\color[HTML]{0000FF} {\ul 0.164}}    & \multicolumn{1}{c|}{0.260} & 0.165                                 & 0.261 \\ \midrule
                               & 96  & 0.411                                 & \multicolumn{1}{c|}{{\color[HTML]{0000FF} {\ul 0.234}}}    & {\color[HTML]{FF0000} \textbf{0.395}} & \multicolumn{1}{c|}{{\color[HTML]{FF0000} \textbf{0.233}}} & {\color[HTML]{0000FF} {\ul 0.398}}          & 0.237                                 & 0.408                                 & \multicolumn{1}{c|}{0.263}                              & 0.400                                 & \multicolumn{1}{c|}{0.259} & 0.405                                 & 0.264 \\
                               & 192 & 0.428                                 & \multicolumn{1}{c|}{{\color[HTML]{0000FF} {\ul 0.246}}}    & {\color[HTML]{0000FF} {\ul 0.423}}    & \multicolumn{1}{c|}{{\color[HTML]{FF0000} \textbf{0.245}}} & {\color[HTML]{0000FF} {\ul \textbf{0.423}}} & 0.250                                 & 0.439                                 & \multicolumn{1}{c|}{0.276}                              & {\color[HTML]{FF0000} \textbf{0.421}} & \multicolumn{1}{c|}{0.273} & 0.439                                 & 0.291 \\
                               & 336 & 0.444                                 & \multicolumn{1}{c|}{{\color[HTML]{FF0000} \textbf{0.254}}} & {\color[HTML]{0000FF} {\ul 0.443}}    & \multicolumn{1}{c|}{{\color[HTML]{FF0000} \textbf{0.254}}} & 0.447                                       & {\color[HTML]{0000FF} {\ul 0.259}}    & 0.461                                 & \multicolumn{1}{c|}{0.285}                              & 0.446                                 & \multicolumn{1}{c|}{0.281} & {\color[HTML]{FF0000} \textbf{0.441}} & 0.282 \\
                               & 720 & 0.514                                 & \multicolumn{1}{c|}{0.280}                                 & {\color[HTML]{0000FF} {\ul 0.480}}    & \multicolumn{1}{c|}{{\color[HTML]{FF0000} \textbf{0.274}}} & 0.482                                       & {\color[HTML]{0000FF} {\ul 0.278}}    & 0.506                                 & \multicolumn{1}{c|}{0.305}                              & 0.489                                 & \multicolumn{1}{c|}{0.300} & {\color[HTML]{FF0000} \textbf{0.478}} & 0.300 \\
\multirow{-5}{*}{\rotatebox[origin=c]{90}{Traffic}}      & Avg & 0.449                                 & \multicolumn{1}{c|}{{\color[HTML]{0000FF} {\ul 0.254}}}    & {\color[HTML]{FF0000} \textbf{0.435}} & \multicolumn{1}{c|}{{\color[HTML]{FF0000} \textbf{0.251}}} & {\color[HTML]{0000FF} {\ul 0.438}}          & 0.256                                 & 0.453                                 & \multicolumn{1}{c|}{0.282}                              & 0.439                                 & \multicolumn{1}{c|}{0.278} & 0.441                                 & 0.284 \\ \midrule
                               & 96  & 0.186                                 & \multicolumn{1}{c|}{0.195}                                 & {\color[HTML]{0000FF} {\ul 0.180}}    & \multicolumn{1}{c|}{{\color[HTML]{FF0000} \textbf{0.191}}} & {\color[HTML]{FF0000} \textbf{0.179}}       & {\color[HTML]{0000FF} {\ul 0.192}}    & 0.198                                 & \multicolumn{1}{c|}{0.228}                              & 0.196                                 & \multicolumn{1}{c|}{0.227} & 0.191                                 & 0.226 \\
                               & 192 & {\color[HTML]{0000FF} {\ul 0.214}}    & \multicolumn{1}{c|}{{\color[HTML]{0000FF} {\ul 0.216}}}    & {\color[HTML]{FF0000} \textbf{0.213}} & \multicolumn{1}{c|}{{\color[HTML]{FF0000} \textbf{0.215}}} & {\color[HTML]{0000FF} {\ul 0.214}}          & {\color[HTML]{0000FF} {\ul 0.216}}    & 0.227                                 & \multicolumn{1}{c|}{0.253}                              & 0.225                                 & \multicolumn{1}{c|}{0.252} & 0.221                                 & 0.251 \\
                               & 336 & {\color[HTML]{0000FF} {\ul 0.234}}    & \multicolumn{1}{c|}{{\color[HTML]{0000FF} {\ul 0.233}}}    & {\color[HTML]{FF0000} \textbf{0.233}} & \multicolumn{1}{c|}{{\color[HTML]{FF0000} \textbf{0.232}}} & 0.236                                       & {\color[HTML]{0000FF} {\ul 0.233}}    & {\color[HTML]{FF0000} \textbf{0.233}} & \multicolumn{1}{c|}{0.265}                              & 0.237                                 & \multicolumn{1}{c|}{0.265} & 0.235                                 & 0.265 \\
                               & 720 & {\color[HTML]{FF0000} \textbf{0.239}} & \multicolumn{1}{c|}{{\color[HTML]{FF0000} \textbf{0.237}}} & {\color[HTML]{0000FF} {\ul 0.241}}    & \multicolumn{1}{c|}{{\color[HTML]{FF0000} \textbf{0.237}}} & 0.248                                       & {\color[HTML]{0000FF} {\ul 0.241}}    & {\color[HTML]{FF0000} \textbf{0.239}} & \multicolumn{1}{c|}{0.271}                              & {\color[HTML]{0000FF} {\ul 0.241}}    & \multicolumn{1}{c|}{0.273} & {\color[HTML]{0000FF} {\ul 0.241}}    & 0.274 \\
\multirow{-5}{*}{\rotatebox[origin=c]{90}{Solar-Energy}} & Avg & {\color[HTML]{0000FF} {\ul 0.218}}    & \multicolumn{1}{c|}{{\color[HTML]{0000FF} {\ul 0.220}}}    & {\color[HTML]{FF0000} \textbf{0.217}} & \multicolumn{1}{c|}{{\color[HTML]{FF0000} \textbf{0.219}}} & 0.219                                       & 0.221                                 & 0.224                                 & \multicolumn{1}{c|}{0.254}                              & 0.225                                 & \multicolumn{1}{c|}{0.254} & 0.222                                 & 0.254 \\ \bottomrule
\end{tabular}

}
\caption{Performance comparison under different loss function settings. The lookback length is fixed at 96. Generally, L1 loss function performs better than L2 counterpart.}
\label{tab_L1_apha}
\end{center}
\end{table*}

Table \ref{tab_L1_apha} presents the results under various combinations of loss functions (L1/L2) and $\alpha$ values. Overall, the L1 loss function demonstrates superior performance compared to L2. Among the three $\alpha$ settings, $\alpha = 0.5$ achieves the best results.

\paragraph{Embedding Dimension $d$}

\begin{table*}[ht]
\begin{center}
{\fontsize{8}{10}\selectfont
\begin{tabular}{@{}c|cc|cc|cc|cc|cc|cc@{}}
\toprule
Emb. Dim. & \multicolumn{2}{c|}{1}                                                        & \multicolumn{2}{c|}{4}                                                     & \multicolumn{2}{c|}{8}                                                        & \multicolumn{2}{c|}{16}                                                       & \multicolumn{2}{c|}{32}                                                       & \multicolumn{2}{c}{64}                                                        \\ \midrule
Metric    & MSE                                   & MAE                                   & MSE                                   & MAE                                & MSE                                   & MAE                                   & MSE                                   & MAE                                   & MSE                                   & MAE                                   & MSE                                   & MAE                                   \\ \midrule
ECL       & {\color[HTML]{0000FF} {\ul 0.134}}    & 0.225                                 & {\color[HTML]{0000FF} {\ul 0.134}}    & 0.224                              & {\color[HTML]{FF0000} \textbf{0.133}} & 0.224                                 & {\color[HTML]{FF0000} \textbf{0.133}} & {\color[HTML]{0000FF} {\ul 0.223}}    & {\color[HTML]{0000FF} {\ul 0.134}}    & {\color[HTML]{0000FF} {\ul 0.223}}    & {\color[HTML]{FF0000} \textbf{0.133}} & {\color[HTML]{FF0000} \textbf{0.222}} \\
Traffic   & 0.409                                 & 0.236                                 & {\color[HTML]{0000FF} {\ul 0.396}}    & {\color[HTML]{0000FF} {\ul 0.233}} & 0.413                                 & {\color[HTML]{0000FF} {\ul 0.233}}    & {\color[HTML]{FF0000} \textbf{0.395}} & {\color[HTML]{0000FF} {\ul 0.233}}    & 0.401                                 & {\color[HTML]{FF0000} \textbf{0.232}} & 0.407                                 & 0.234                                 \\
Solar-Energy     & {\color[HTML]{FF0000} \textbf{0.179}} & 0.194                                 & {\color[HTML]{FF0000} \textbf{0.179}} & {\color[HTML]{0000FF} {\ul 0.192}} & {\color[HTML]{0000FF} {\ul 0.180}}    & {\color[HTML]{0000FF} {\ul 0.192}}    & {\color[HTML]{0000FF} {\ul 0.180}}    & {\color[HTML]{FF0000} \textbf{0.191}} & {\color[HTML]{0000FF} {\ul 0.180}}    & {\color[HTML]{FF0000} \textbf{0.191}} & {\color[HTML]{0000FF} {\ul 0.180}}    & {\color[HTML]{0000FF} {\ul 0.192}}    \\
Weather   & 0.155                                 & 0.192                                 & 0.155                                 & {\color[HTML]{0000FF} {\ul 0.191}} & 0.155                                 & {\color[HTML]{0000FF} {\ul 0.191}}    & {\color[HTML]{FF0000} \textbf{0.153}} & {\color[HTML]{FF0000} \textbf{0.189}} & {\color[HTML]{FF0000} \textbf{0.153}} & {\color[HTML]{FF0000} \textbf{0.189}} & {\color[HTML]{0000FF} {\ul 0.154}}    & {\color[HTML]{FF0000} \textbf{0.189}} \\
ILI       & {\color[HTML]{FF0000} \textbf{1.741}} & {\color[HTML]{FF0000} \textbf{0.818}} & 2.084                                 & 0.883                              & 1.881                                 & 0.844                                 & {\color[HTML]{0000FF} {\ul 1.879}}    & {\color[HTML]{0000FF} {\ul 0.823}}    & 2.033                                 & 0.866                                 & 2.008                                 & 0.851                                 \\
NASDAQ    & 0.170                                 & 0.265                                 & 0.167                                 & 0.264                              & {\color[HTML]{FF0000} \textbf{0.160}} & {\color[HTML]{FF0000} \textbf{0.259}} & {\color[HTML]{0000FF} {\ul 0.166}}    & {\color[HTML]{0000FF} {\ul 0.263}}    & 0.170                                 & 0.267                                 & 0.169                                 & 0.266                                 \\ \bottomrule
\end{tabular}
}
\caption{Performance comparison with varying embedding dimensions. Dimension expansion is performed without learnable weights when $d=16$. For the ECL, Traffic, Solar-Energy, and Weather datasets, the lookback and prediction horizons are both 96. For the ILI and NASDAQ datasets, they are both 36.}
\label{tab_embed}
\end{center}
\end{table*}

Table \ref{tab_embed} reports the performance under different embedding dimensions. The default setting of $d=16$ demonstrates robust performance among these choices.

\end{document}